%% file: arxiv.tex
\documentclass{article}

\PassOptionsToPackage{numbers,sort, compress}{natbib}

\usepackage{afterpage}

\usepackage[main, final]{neurips_2026}

\usepackage[utf8]{inputenc} %
\usepackage[T1]{fontenc}    %
\usepackage[para]{footmisc}
\usepackage{hyperref}       %
\usepackage{url}            %
\usepackage{booktabs}       %
\usepackage{amsfonts}       %
\usepackage{nicefrac}       %
\usepackage{microtype}      %
\usepackage{xcolor}         %

\usepackage{pgfplots}
\usepackage{subcaption}
\usepackage{etoc}
\usepackage{titletoc}
\usepackage{pgfplots}
\usepackage{tikz}

\usepackage{placeins}

\usepackage[ruled,vlined,linesnumbered]{algorithm2e} %

\usepackage{multirow}
\usepackage{booktabs}
\usepackage{tabularx}
\usepackage{array}
\usepackage{longtable}
\usepackage{makecell}

\usepackage{wrapfig}

\newcolumntype{Y}{>{\raggedright\arraybackslash}X}

\newcolumntype{L}[1]{>{\raggedright\arraybackslash}p{#1}}

\input{figures/latex_plot_colors}

\input{figures/latex_plot_legends}

\usepackage{amsmath}
\usepackage{amssymb}
\usepackage{mathtools}
\usepackage{amsthm}
\usepackage{bbm}
\usepackage{bm}

\theoremstyle{plain}
\newtheorem{theorem}{Theorem}[section]
\newtheorem{proposition}[theorem]{Proposition}

\theoremstyle{definition}
\newtheorem{definition}[theorem]{Definition}

\newtheorem{remark}[theorem]{Remark}

\newcommand{\R}{\mathbb{R}}

\newcommand{\Z}{\mathcal{Z}}
\newcommand{\AS}{\mathcal{A}}
\newcommand{\SP}{\mathcal{S}}

\newcommand{\E}{\mathbbm{E}}

\newcommand{\V}{\operatorname{Var}}

\newcommand{\KL}{D_{\mathrm{KL}}}

\usepackage{cleveref} 
\crefname{definition}{definition}{definitions}
\Crefname{definition}{Definition}{Definitions}

\DeclareMathOperator*{\argmin}{arg\,min}

\makeatletter
\DeclareRobustCommand{\cev}[1]{%
  {\mathpalette\do@cev{#1}}%
}
\newcommand{\do@cev}[2]{%
  \vbox{\offinterlineskip
    \sbox\z@{$\m@th#1 x$}%
    \ialign{##\cr
      \hidewidth\reflectbox{$\m@th#1\vec{}\mkern4mu$}\hidewidth\cr
      \noalign{\kern-\ht\z@}
      $\m@th#1#2$\cr
    }%
  }%
}

\makeatother

\newcommand{\dd}{\mathrm{d}}

\title{Scalable Maximum Entropy Reinforcement Learning for Diffusion Policies via Adjoint Matching}

\author{%
  Serge Thilges\thanks{Correspondence to \url{serge.thilges@kit.edu}}\quad
  Onur Celik\quad
  Denis Blessing\quad
  Emiliyan Gospodinov\quad
  Gerhard Neumann
  \\[0.5em]
  Autonomous Learning Robots, Karlsruhe Institute of Technology
}

\begin{document}

\maketitle
\setcounter{footnote}{0}

\begin{abstract}
Diffusion policies have recently emerged as a powerful paradigm for representing complex action distributions in reinforcement learning (RL).
However, their application to online RL remains limited by the challenge of scalable training in the absence of ground-truth data, where standard optimization techniques such as score matching are not directly applicable.
In this work, we introduce a highly efficient algorithm for optimizing diffusion policies by leveraging recent advances in stochastic optimal control.
Our approach is based on adjoint matching, which enables simulation-free training and circumvents the need for explicit likelihood estimation or costly backpropagation through the diffusion process.
Furthermore, we propose several extensions that improve the robustness and stability of the method in practical settings.
Empirical results demonstrate that our approach achieves competitive performance while significantly reducing computational overhead, making diffusion policies more viable for online RL scenarios.
\end{abstract}

\section{Introduction}
Diffusion-based policy representations \cite{reuss2023goal, chi2025diffusion} have recently gained significant attention in reinforcement learning (RL) \cite{PsenkaEA024,celik2025dime, wang2021deep,ding2024diffusionbased,ma2025efficient,dong2025maximum} due to their ability to represent highly complex, multi-modal action distributions.
Despite their expressive power, optimizing these models in online RL remains a non-trivial challenge.
Online RL lacks a static dataset of ground-truth samples, which prevents the direct application of celebrated training objectives such as score matching \cite{song2020score} or bridge matching \cite{shi2023diffusion,liu2022let}.
Existing approaches for training diffusion policies generally fall into three categories.
First, several methods leverage the reverse Kullback-Leibler (KL) divergence \cite{wang2024diffusion,celik2025dime,lv2026flacmaximumentropyrl}, which requires backpropagating through the entire diffusion process.
These approaches can be effective but require full simulation of the diffusion process for every gradient step and incur a massive memory overhead that prevents scaling to larger problems.
Second, importance-weighted matching approaches attempt to minimize the error between the predicted actions and the target actions by incorporating learning signals from the $Q$-function via importance sampling \cite{ding2024diffusionbased,ma2025efficient,chen2025onestepflowpolicymirror,dong2025maximum}.
However, it is well-established that importance sampling is notoriously difficult to scale to high-dimensional action spaces \cite{snyder2008obstacles}.
Third, methods that utilize the ``$Q$-score'' \cite{PsenkaEA024}, that is, $\nabla_a Q(a,s)$, either \textit{i)}~are centered around the overdamped Langevin equation \cite{ki2026directsoftpolicysamplinglangevin}, $\dd a_\tau = \sigma^2(\tau)\nabla_a Q(a_\tau, s) \dd \tau + \sqrt{2}\sigma(\tau) \dd B_\tau$, or \textit{ii)}~rely on heuristic techniques designed to sample from regions with high $Q$-values.
The Langevin process converges to $\pi(a|s) \propto \exp Q(a,s)$ \textit{asymptotically} as $\tau \to \infty$ but must be truncated in practice, thereby offering few theoretical guarantees in the finite-time regime, whereas the heuristic methods fundamentally lack guarantees of generating samples from the true target action distribution.

In this work, we propose a novel alternative that counteracts the limitations of existing methods.\footnote{Project website: \url{https://alrhub.github.io/amdp-website/}}
By framing diffusion policy optimization for maximum entropy RL (MaxEnt-RL) \cite{ziebart2008maximum,toussaint2009robot, haarnoja2017reinforcement, haarnoja2018soft} as a stochastic optimal control (SOC) problem, we adapt Adjoint Matching (AM) \cite{domingoenrich2025adjoint, havens2025adjoint}, a recent framework for solving SOC problems in a highly scalable manner.
AM utilizes a fixed-point loss whose unique fixed-point ensures that samples from the target density are generated in finite time.
We call the resulting method \textit{Adjoint Matching Diffusion Policy} (\texttt{AMDP}), offering three advantages:
\begin{itemize}
    \item It utilizes a highly scalable regression objective involving the $Q$-score, similar to score-matching methods, but theoretically grounded for RL.
    \item It circumvents the need for costly backpropagation through the diffusion process, significantly reducing memory requirements.
    \item By adopting reciprocal adjoint matching, the optimization is simulation-free.
    Only the actions are needed, which can be stored and reused for multiple gradient updates.
\end{itemize}
Throughout this work, scalability refers to computational scalability, i.e., training time and memory requirements that remain low as the number of diffusion steps and the actor size increase.
Beyond the core algorithm, we introduce error function action squashing, a numerically stable alternative to $\tanh$ squashing that simplifies the \texttt{AMDP} training objective, and a trust-region loss that stabilizes policy updates with minimal overhead and without altering the theoretical guarantees of adjoint matching.
We show \texttt{AMDP}'s benefits on 63 highly parallelized environments from MuJoCo Playground \cite{mujoco_playground_2025}, ManiSkill \cite{gu2023maniskill2}, and the HumanoidBench \cite{sferrazza2024humanoidbench} in the on-policy setting and demonstrate its flexibility by comparing its performance in the off-policy setting on 7 high-dimensional dog and humanoid environments of the DeepMind Control (DMC) suite \cite{tunyasuvunakool2020dm_control}.
The comparisons indicate that \texttt{AMDP} matches or surpasses the strong baselines' performance while matching the training time of highly efficient Gaussian on-policy methods.
\texttt{AMDP} combines the expressivity of diffusion policies, beneficial in imitation learning \cite{chi2025diffusion,reuss2023goal,pearce2023imitating}, with near-Gaussian training time, updates an order of magnitude faster than reverse-KL methods, and memory independent of the number of diffusion steps.

 \begin{figure}[t!]
            \centering
            \includegraphics[width=\textwidth, trim=0mm 10mm 0mm 3mm, clip]{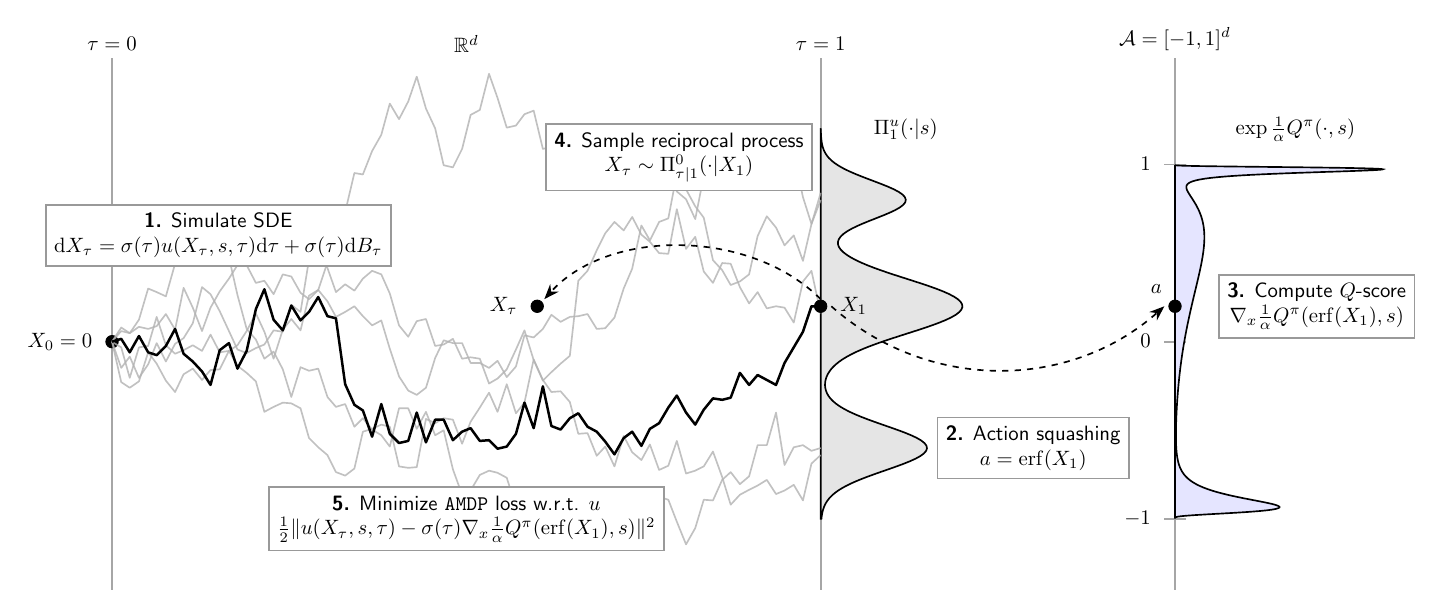}
        \caption[ ]
        {\textbf{Pipeline of Adjoint Matching Diffusion Policy (\texttt{AMDP}) training.}
        An action is sampled by \textbf{(1.)} simulating the forward stochastic differential equation (SDE) from $X_0 = 0$ to a terminal sample $X_1 \sim \Pi^u_1(\cdot|s) \in \mathbb{R}^d$, and \textbf{(2.)} error function squashing it to the bounded action space $\mathcal{A} = [-1,1]^d$.
        The target signal is \textbf{(3.)} the scaled gradient of the $Q$-function, which applies \textbf{(4.)} to intermediate states $X_\tau$, sampled directly from the known reciprocal process $\Pi^0_{\tau|1}$.
        Finally, \textbf{(5.)} the policy's control field $u$ minimizes a regression loss, entirely avoiding backpropagation through the diffusion process.}
        \label{fig:amdp illustration}
\end{figure}

\section{Preliminaries}
\label{sec: preliminaries}

\paragraph{Maximum entropy reinforcement learning.}
\label{sec: max ent rl}

We consider a Markov Decision Process (MDP) defined by the tuple $\mathcal{M} = (\mathcal{S}, \mathcal{A}, p, r, \gamma)$, with continuous state space $\mathcal{S}$ and action space $\mathcal{A}$.
The transition dynamics are $p(s' \mid s, a)$, the reward function is $r(s, a)$, and $\gamma \in [0, 1)$ is the discount factor.
The goal is to learn a policy $\pi(a \mid s)$ that maximizes the expected return.
In continuous control, to prevent premature convergence and encourage exploration, the objective is often augmented with an entropy term (Maximum Entropy RL) \cite{haarnoja2017reinforcement}:
\begin{align}
J_{\mathrm{MaxEnt}}(\pi) = \mathbb{E}_{\pi} \left[ \sum_{t=0}^\infty \gamma^t \left( r(s_t, a_t) + \alpha \mathcal{H}(\pi(\cdot \mid s_t)) \right) \right],
\label{eq:rl_objective}
\end{align}
where $\mathcal{H}(\pi) = -\mathbb{E}_{a \sim \pi}[\log \pi(a \mid s)]$ is the differential entropy.
This objective encourages the agent to maximize rewards while maintaining some randomness in its actions, controlled by the temperature parameter $\alpha > 0$.
For a given policy $\pi$, we define the soft state-action value function (or soft $Q$-function) as the expected return starting from state $s$ and action $a$, following policy $\pi$ thereafter:
\begin{equation}
    Q^\pi(a, s) = r(s, a) + \mathbb{E}_{\pi, p} \left[ \sum_{t=1}^\infty \gamma^t (r(s_t, a_t) + \alpha \mathcal{H}(\pi(\cdot | s_t))) \right].
\end{equation}
The optimal policy $\pi^*$ maximizes $J_{\mathrm{MaxEnt}}(\pi)$.
For a fixed $Q$-function, the maximizer of the entropy-regularized objective $\mathbb{E}_{a \sim \pi}[Q^\pi(a,s)] + \alpha \mathcal{H}(\pi(\cdot \mid s))$ takes the form of an energy-based distribution \cite{ziebart2008maximum,haarnoja2017reinforcement,haarnoja2018soft}, which coincides with $\pi^*$ for the optimal soft $Q$-function:
\begin{equation}
\label{eq: energy policy}
    \pi^*(a|s) = \frac{\exp(Q(a,s)/\alpha)}{\Z(s)}, \quad \Z(s) = \int_{\AS} \exp(Q(a,s)/\alpha) \dd a.
\end{equation}
In this framework, policy improvement can be viewed as an approximate inference problem.
This is commonly tackled using a variational approach by minimizing the reverse Kullback-Leibler (KL) divergence between the parametric policy $\pi$ and the optimal policy $\pi^*$:
\begin{equation}
\label{eq: marginal kl}
    \mathcal{L}(\pi) = \KL\left(\pi(a|s) \mid \pi^*(a|s)\right) = \mathbb{E}_{a \sim \pi} \left[ \log \pi(a|s) - Q^\pi(a, s)/\alpha \right] + \log \Z(s).
\end{equation}
In practice, one optimizes over a family of policy classes.
Traditional methods typically employ Gaussian distributions due to their tractable likelihood and ease of reparameterization.
However, Gaussian policies are often insufficient for modeling complex, multi-modal action distributions.
Diffusion models offer a promising alternative by representing the policy through a dynamical transport process, enabling the representation of highly complex and non-Gaussian action distributions.

\paragraph{Diffusion policies and path space measures.}
Unlike explicit policies (e.g., Gaussians) that directly output action samples, diffusion policies define the distribution $\pi(a \mid s)$ implicitly through a transport process.
Let $\tau \in [0, 1]$ denote the continuous generation time.
The action generation is modeled as the solution to a state-conditioned stochastic differential equation (SDE) $(X_{\tau})_{\tau\in[0,1]}$ \cite{oksendal2003stochastic} on $\R^d$ with
\begin{align}
\dd X_\tau = \sigma(\tau)u(X_\tau, s,\tau) \dd \tau + \sigma(\tau) \dd B_\tau, \quad X_0 \sim \mu_0,
\label{eq: controlled sde}
\end{align}
where $\mu_0$ denotes a tractable prior such as a Gaussian, $u: \mathbb{R}^{d} \times \mathcal{S} \times [0,1] \to \mathbb{R}^{d}$ is a drift or control function and $B_\tau$ is a standard Brownian motion.
After simulating \eqref{eq: controlled sde}, we set $a = X_1$ to be the realized action, i.e., the diffusion process is independent of the MDP dynamics and \eqref{eq: controlled sde} is solved anew from $X_0 \sim \mu_0$ for every environment state $s_t$ to obtain $a_t$.

In contrast to explicit generative models which allow for computing the likelihood, diffusion models require marginalizing over all possible paths to obtain the marginal likelihood at time $\tau=1$, which is intractable, so directly optimizing \eqref{eq: marginal kl} is challenging.
Instead, one can consider optimizing over \textit{path space measures}, which can be viewed as a distribution over the space of continuous paths $X$.
Formally, we denote by $\Pi^u(X|s)$ the state-conditioned path space measure induced by the SDE \eqref{eq: controlled sde} and by $\Pi^u_\tau(X|s)$ the marginal distribution at time $\tau$.
Our goal is to find a drift $u$ such that  $\Pi^{u}_1=\pi^*$, i.e., the marginal distribution of the diffusion process \eqref{eq: controlled sde} at the terminal time $\tau=1$ corresponds to the optimal policy.
To obtain a variational characterization for optimizing $u$ similar to \eqref{eq: marginal kl}, one needs to construct a target path space measure $\Pi^*$ which satisfies $\Pi^*_0=\mu_0$ and $\Pi^*_{1}=\pi^*$.
While there are infinitely many ways of reaching $\pi^*$ from $\mu_0$ and therefore infinitely many $\Pi^*$, we are particularly interested in the solution that additionally minimizes the kinetic-cost $\E_{X\sim\Pi^u}\left[\int_0^1 \frac{1}{2}\|u(X_\tau, s,\tau)\|^2\dd \tau\right]$, corresponding to the Schr\"odinger bridge problem \cite{schrodinger1931umkehrung, follmer2006random,leonard2013survey,chen2016relation}.

\paragraph{Schr\"odinger bridges.}
For a given state $s\in\SP$, the Schr\"odinger bridge (SB) problem is defined as
\begin{equation}
\label{eq: sb problem}
    \min_{u} \ \E_{X\sim\Pi^u}\left[\int_0^1 \frac{1}{2}\|u(X_\tau, s,\tau)\|^2\dd \tau\right]\quad \mathrm{s.t.} \quad X_0 \sim \mu_0, \ X_1 \sim \pi^*(\cdot|s).
\end{equation}
Solving the SB problem is notoriously difficult as it requires solving a coupled system of equations, known as the Schr\"odinger system, often approached by alternating optimization routines such as iterative proportional fitting \cite{fortet1940resolution,kullback1968probability,ruschendorf1993note}.
However, in certain special cases, the Schr\"odinger system can be decoupled, which simplifies the problem significantly \cite{dai1991stochastic,zhang2021path}.
To this end, let $\Pi^0$ denote the path space measure of the \textit{uncontrolled} or \textit{reference} process, induced by \eqref{eq: controlled sde} with $u=0$.
Then, if the reference process is \textit{memoryless} \cite{domingoenrich2025adjoint}, meaning that $X_0$ and $X_1$ are independent, the Schr\"odinger system decouples and solving the SB problem becomes equivalent to minimizing
\begin{equation}
\label{eq: path kl}
    \KL(\Pi^u(\cdot|s)|\Pi^*(\cdot|s)) \triangleq \E_{X\sim \Pi^u}\!\left[\int_0^1 \tfrac{1}{2}\|u(X_\tau, s,\tau)\|^2\dd \tau +\log \Pi^0_1(X_1|s)-\tfrac{1}{\alpha}Q^{\pi}(X_1,s)\right]\!,
\end{equation}
where $\triangleq$ indicates equality up to a constant that is independent of $u$.
Eq. \eqref{eq: path kl} can also be interpreted in terms of the reverse Kullback-Leibler divergence between path space measures and can be seen as a variational characterization for optimizing diffusion models analogous to \eqref{eq: marginal kl}.
We provide further details in Appendix~\ref{appendix: sb connection}.
To enforce memorylessness, we follow \cite{zhang2021path,vargas2023bayesian,havens2025adjoint} and use a deterministic initial condition $X_0=0$.
In that case, the reference dynamic is given by a scaled Brownian motion $\dd X_\tau=\sigma(\tau)\dd B_\tau$ which admits a tractable (state-independent) marginal distribution at the terminal time given by $\Pi^0_1(x|s)=\mathcal{N}\bigl(x|0,\int_0^1\sigma^2(\tau)\dd \tau\bigr)$.
Minimizing \eqref{eq: path kl} can now be performed by using a parameterized control $u_{\theta}$, simulating paths with $u_{\theta}$ and differentiating backwards through these paths to update $\theta$ \cite{zhang2021path,vargas2023denoising,berner2022optimal,vargas2024transport,blessing2025end,blessing2025underdamped}.
While conceptually simple and straightforward to implement, optimizing through the diffusion process can lead to a substantial memory overhead and requires full path simulations for every gradient step, making it challenging to scale to complex problems \cite{domingoenrich2025adjoint}.

\section{Related work}
\label{section: related work}

\paragraph{Generative models in RL.}
Diffusion \citep{sohl2015deep,song2020denoising, ho2020denoising,karras2022elucidating} and Flow \citep{lipman2023flow,albergo2023stochastic} models have seen adoption in Reinforcement Learning as trajectory planners \citep{JannerDTL22,AjayDGTJA23}, distributional critic models \citep{agrawalla2026floq,dong2026value}, synthetic data generation in the maximum entropy RL setting for Gaussian policies \citep{ishfaq2025langevin} or as expressive policy representations \citep{WangHZ23,Kang0DPY23,LuCEP2023,Mao0Z0Z24,FangLZWJ25,PsenkaEA024,ding2024diffusionbased,FangLZWJ25,celik2025dime,RenLAS0MBDS25,ding2026genpo,ma2025reinforcement,mcallister2026flow,lv2026flowbased,zhang2026reinflow,ZhangZG25,ma2025efficient,li2026qlearning,dong2026meanflowpolicyoptimization,qiu2026scoreflowcompletedistributionalcontrol,gong2026proximalpolicyoptimizationpath,lv2026flacmaximumentropyrl,goo2022knowboundariesnecessityexplicit,ding2024consistency,dong2025maximum,black2024training,liu2026flowgrpo,wang2025enhanceddaceralgorithmhigh,chen2025onestepflowpolicymirror,gao2025behaviorregularized,jain2025sampling,zhang2026sac,gao2026flowmatchingpolicyentropy,zhong2026reparameterizationflowpolicyoptimization,zhang2026gorl,le2026trudi}.
A review of the different approaches is provided in Appendix \ref{appendix: extended related work} while we summarize the most relevant work in online RL here.
Lacking a fixed dataset, policies cannot rely on standard score-matching objectives and also need to consider entropy to explore the environment.

\paragraph{Action gradients and weighted regression.}
A branch of prior work adapts the standard score and velocity matching objectives with small adjustments.
DIPO \citep{DIPOYang2023} and DDiffPG \citep{DDiffPG_Li2024} refine previous actions by gradient ascent on the Q-function and applying behavioral cloning.
QSM \citep{PsenkaEA024} directly matches the score function to the Q-score, leading to bias \citep{gao2026flowrltaxonomymodularframework,ki2026directsoftpolicysamplinglangevin}, and adds Gaussian noise to the action for exploration.
Subsequent methods optimize the policy via weighted regression objectives, with weighting defined through the Q-values of the generated actions \citep{ding2024diffusionbased,ma2025efficient,dong2025maximum,chen2025onestepflowpolicymirror,dong2026meanflowpolicyoptimization,gao2026flowmatchingpolicyentropy}.
In contrast to other methods which weight (indirectly) by $\exp Q(a,s)$, QVPO \citep{ding2024diffusionbased} weights the action sample by $Q$ and thus needs to handle negative Q-values.
Additionally, QVPO regularizes with score matching of uniform distribution samples and applies Best-of-N (BoN) sampling both in rollout and objective computation.
\citet{ma2025efficient} introduce both DPMD and SDAC for mirror descent and maximum entropy, respectively, which utilize BoN sampling and additive Gaussian noise for exploration.
Both SDAC \citep{ma2025efficient} and MaxEntDP \citep{dong2025maximum} approximate the log-likelihoods for entropy computation, SDAC treats the BoN result as deterministic and uses the additive noise likelihood and MaxEntDP applies numerical integration but needs to circumvent singularities and generates 500 actions with Q-function evaluation per state.
MFPO \citep{dong2026meanflowpolicyoptimization} utilizes a mean flow model and learns the divergence for likelihood, and thus entropy, estimation of the max-entropy policy while BoN action selection is applied during test time.
Similarly, FMER \citep{gao2026flowmatchingpolicyentropy} computes the entropy through the flow divergence and applies BoN but uses Hutchinson's trace estimator directly without a learned model.

\paragraph{Chain backpropagation.}
To utilize first-order information from the critic without the bias common in guidance methods \citep{gao2026flowrltaxonomymodularframework,ki2026directsoftpolicysamplinglangevin,ZhangZG25}, another line of work directly optimizes the expected return by backpropagating the critic gradient through the action diffusion chain.
DACER \citep{wang2024diffusion} handles exploration by fitting a Gaussian Mixture Model to estimate the entropy which then tunes the scale of additive Gaussian noise during training rollout.
FlowRL \citep{lv2026flowbased} lacks an explicit entropy mechanism but constrains to the behavioral policy with a fixed Lagrangian multiplier.
GenPO \citep{ding2026genpo} applies PPO's \citep{schulman2017proximal} policy gradients and backpropagates the likelihood through the chain of change-of-variables, while integrating an entropy-maximization term, thus considering the entropy-regularized setting.
In contrast, DIME \citep{celik2025dime} operates in the MaxEnt-RL setting and minimizes a reverse KL objective on the denoising process with a principled entropy lower bound.
TruDi \citep{le2026trudi} extends DIME to on-policy RL with a trust region on the joint diffusion trajectory.
FLAC \citep{lv2026flacmaximumentropyrl} also solves the MaxEnt-RL setting but formalizes the problem in continuous time using Schr\"odinger bridges.
In contrast to our work, FLAC backpropagates the loss gradient through the full diffusion and approximates the reference marginal as uniform.
While exact and error-correcting, chain backpropagation requires storing the intermediate steps of the diffusion or flow process (explicitly or implicitly), which heavily bottlenecks scalability for models with many integration steps.

\paragraph{Adjoint matching.}
Recently, QAM \citep{li2026reinforcement} was introduced in the offline-to-online setting as another method utilizing Adjoint Matching \citep{domingoenrich2025adjoint}, fine-tuning a policy pretrained on the behavioral dataset towards a $Q$-tilted target.
In contrast, \texttt{AMDP} builds on Reciprocal Adjoint Matching \citep{havens2025adjoint}, which learns a sampler from scratch given an energy function, here the soft $Q$-function, without a behavioral policy or dataset.
QAM has to solve the adjoint ordinary differential equation (ODE) backwards with many vector-Jacobian products, which is computationally equivalent to backpropagation, whereas the reciprocal adjoint state is available in closed form.
QAM further adds one-step and residual components, clips actions, and controls stochasticity implicitly through the behavioral policy and a per-task tuned temperature, while \texttt{AMDP} introduces the entropy lower bound, enabling temperature auto-tuning, and trust region required for on-policy RL.
Due to the different settings, the two methods are not directly comparable empirically.

\paragraph{Policy gradient.}
To bypass the difficulty of action likelihood computation, DPPO \citep{RenLAS0MBDS25} formulates the diffusion process itself as a multi-step MDP, drawing parallels to RL fine-tuning methods in image generation like DDPO \citep{black2024training} and Flow-GRPO \citep{liu2026flowgrpo}.
To that end, DPPO \citep{RenLAS0MBDS25} directly applies PPO \citep{schulman2017proximal} to the hierarchy-aware returns which however still requires storage of all the noisy intermediate states.
Alternatively, FPO \citep{mcallister2026flow} estimates the marginal likelihood ratio through a Monte-Carlo estimate over multiple intermediate states of the continuous normalizing flow, which also need to be stored.
GoRL \citep{zhang2026gorl} instead applies PPO in the latent space of a diffusion decoder that is periodically refit on its own rollouts, which avoids storing the chain but restricts entropy control to the latent space without a defined target action distribution.

\section{Adjoint matching for maximum entropy reinforcement learning} \label{AM for maxent RL}
\subsection{Adjoint matching for diffusion policies}
\paragraph{Stochastic optimal control and adjoint matching.}
To remedy the problems associated with optimizing the reverse KL, \citet{domingoenrich2025adjoint} introduced \textit{adjoint matching} (AM) for solving stochastic optimal control (SOC) problems.
Indeed, \eqref{eq: path kl} corresponds to a SOC problem with quadratic running costs and terminal cost $g(x,s) =\log \Pi^0_1(x|s)-\tfrac{1}{\alpha}Q^{\pi}(x,s)$ giving the adjoint matching loss,
\begin{equation}
\label{eq: am loss}
    \mathcal{L}^s_{\mathrm{AM}}(u) = \E_{X\sim\Pi^{\bar u}}\left[ \int_0^1\frac{1}{2}\|u(X_\tau,s,\tau)-\sigma(\tau)\nabla_{x}\left(\tfrac{1}{\alpha}Q^{\pi}(X_1,s)-\log \Pi^0_1(X_1|s)\right)\|^2\dd \tau\right],
\end{equation}
for a given state $s$ which does not require differentiation through the diffusion process as indicated by the stop-gradient operator $\mathrm{sg}(\cdot)$ with $\bar u = \mathrm{sg}(u)$.
The adjoint matching loss can be seen as a fixed-point iteration whose unique fixed-point is the optimal control $u^*$ that minimizes \eqref{eq: path kl}.
Formally, define the operator $\Phi(u) \coloneqq \argmin_u \mathcal{L}_{\mathrm{AM}}(u)$ then $u=\Phi(u)$ if and only if $u=u^*$ \cite{domingoenrich2025adjoint}.

Recently, \citet{havens2025adjoint} further improved adjoint matching by taking the expectation in \eqref{eq: am loss} with respect to the \textit{reciprocal projection} of $\Pi^{\bar u}$ onto the uncontrolled process $\Pi^0$ which is defined as
\begin{equation}
    \mathcal{R}(\Pi^{\bar u})(X|s) \coloneqq \Pi^{\bar u}_1(X_1|s) \Pi^0_{|1}(X|X_1, s),
\end{equation}
where $\Pi^0_{|1}$ is the path space measure of the uncontrolled process conditioned on $X_1$.
In the setting considered in this work, i.e., $X_0=0$, the conditional distribution $\Pi^0_{\tau|1}$ becomes a state-independent Gaussian that allows for tractable sampling of intermediate diffusion states $X_\tau\sim\Pi^0_{\tau|1}$ (see Appendix~\ref{appendix: reference process}).
The factorization then leads to the \textit{reciprocal adjoint matching} (RAM) loss
\begin{equation}
\label{eq: ram loss}
    \mathcal{L}^s_{\mathrm{RAM}}(u) = \int_0^1\E_{X_1 \sim \Pi^{\bar u}_1, X_\tau \sim \Pi^0_{\tau|1}}\left[ \frac{1}{2}\|u(X_\tau,s,\tau)-\sigma(\tau)\nabla_{x}\left(-g(X_1,s)\right)\|^2\right]\dd \tau.
\end{equation}
As such, we can store terminal samples from the current control, i.e., $X_1\sim\Pi^{\bar u}_1$, and then optimize the loss \eqref{eq: ram loss} in a simulation-free manner akin to score, bridge, or flow matching methods, making the approach highly efficient and scalable.
Furthermore, the RAM loss has the same theoretical guarantees as adjoint matching \cite{havens2025adjoint}.

\paragraph{Action space adaptation.}
The action space is typically bounded, most commonly $\AS = [-1,1]^d$ \cite{haarnoja2018soft}, so the terminal sample $X_1 \in \R^d$ cannot be executed directly.
Following established practice, we use a differentiable, bijective transformation $f: \R^d \to [-1,1]^d$ to ``squash'' samples into the action space, i.e., $a = f(X_1)$.
Since our diffusion model operates in the unsquashed latent space, the target $\pi^*(a|s) \propto \exp(Q^{\pi}(f(X_1),s)/\alpha)$ induces, by the change of variables formula \cite{bogachev2007measure,papamakarios2021normalizing}, the density
\begin{equation}
    \tilde \pi^*(X_1|s) = \pi^*(a|s) \big|\det J_f(X_1)\big|, \quad \text{where} \quad J_f(X_1) = \frac{\dd f(X_1)}{\dd X_1}.
\end{equation}
Incorporating $J_f$ into the reciprocal adjoint matching objective leads to the modified loss:
\begin{equation}
\label{eq: ram loss action space}
    \mathcal{L}^{s,f}_{\mathrm{RAM}}(u) = \int_0^1\E\left[ \tfrac{1}{2}\|u(X_\tau,s,\tau)-\sigma(\tau)\nabla_{x}\left(\tfrac{1}{\alpha}Q^{\pi}(f(X_1),s)-\log \tfrac{\Pi^0_1(X_1|s)}{|\det J_f(X_1)|}\right)\|^2\right]\dd \tau,
\end{equation}
with the expectation as in \eqref{eq: ram loss}.
While prior work commonly utilizes an element-wise $\tanh$ transformation \cite{haarnoja2018soft}, i.e., $f_{\tanh}(x) \coloneqq \left(\tanh(x_i)\right)_{i=1}^d$, we find that the scaled \textit{error function}, i.e.,
\begin{equation}
\label{eq: erf}
    f_{\mathrm{erf}}(x) \coloneqq \bigl(\mathrm{erf}(kx_i)\bigr)_{i=1}^d, \quad \mathrm{with} \quad k = \Bigl(2 \int_0^1 \sigma^2(\tau) \dd \tau\Bigr)^{-1/2},
\end{equation}
exhibits superior numerical stability in our setting.
With this $k$, the Jacobian determinant is proportional to the Gaussian reference marginal $\Pi^0_1(x|s) = \mathcal{N}(x \,|\, 0, \int_0^1 \sigma^2(\tau) \dd \tau)$, so the log-ratio in \eqref{eq: ram loss action space} is a constant and its gradient vanishes.
We provide further details on $\mathrm{erf}$-squashing in Appendix~\ref{appendix: action squashing}.
The final loss function for \texttt{AMDP} is conceptually simple, highly efficient and scalable:
\begin{equation}
\label{eq: amdp loss}
    \mathcal{L}^{s}_{\mathrm{AMDP}}(u) = \int_0^1\E_{X_1\sim\Pi^{\bar u}_1, X_\tau\sim\Pi^0_{\tau|1}}\left[ \frac{1}{2}\|u(X_\tau,s,\tau)-\sigma(\tau)\nabla_{x}\tfrac{1}{\alpha}Q^{\pi}(f_{\mathrm{erf}}(X_1),s)\|^2\right]\dd \tau.
\end{equation}

\subsection{Policy iteration for adjoint matching diffusion policies}\label{sec policy iteration}
\setlength{\columnsep}{1.5em}  %
\begin{wrapfigure}[10]{r}{0.48\textwidth}
    \centering
    \vspace{-1.4em}  %
    \begin{algorithm}[H]
        \caption{\texttt{AMDP} algorithm sketch}
        \label{alg:onpolicy_reduced}
        \SetAlgoLined
        \DontPrintSemicolon
        \KwIn{Initial $u$ and $Q$}
        \For{each iteration}{
            Collect and store $(s, a, s', r, X_1)$\;
            Update $Q$ via \eqref{eq: soft bellman}\;
            Compute and store $\nabla_x Q(f_{\mathrm{erf}}(X_1),s)$\;
            Update $u$ via \eqref{eq: amdp loss} or \eqref{eq: Lagrangian loss}\;
        }
    \end{algorithm}
\end{wrapfigure}
We now introduce the policy evaluation and improvement steps that constitute our algorithmic framework, sketched in Algorithm~\ref{alg:onpolicy_reduced} and detailed for the on-policy setting in Algorithm~\ref{alg: onpolicy amdp full}.
By leveraging the specific structure of adjoint matching, we can ensure that our iterative process converges toward the optimal maximum entropy policy within the class of admissible controls.
Following \cite{celik2025dime,haarnoja2017reinforcement,haarnoja2018soft}, the guarantees assume bounded rewards, a bounded action space, $\gamma < 1$, a finite log-partition function $\log \Z(s)$ with $\alpha > 0$, and exact representation of $Q$ and $u$ (see Appendix~\ref{appendix: technical assumptions}).
A rigorous proof of convergence is provided in Appendix~\ref{appendix: policy iteration}.

\paragraph{Policy evaluation.}
To obtain policy iteration guarantees, we define a modified ``soft'' Bellman backup.
In standard maximum entropy RL, the Bellman backup is augmented with the differential entropy $\mathcal{H}(\pi(a|s))$ \cite{haarnoja2018soft}.
However, the marginal entropy for diffusion models is generally intractable, as it requires computing $\pi^u(a|s) = \Pi^u_1(X_1|s)\big|\det J_f(X_1)\big|^{-1}$, where the terminal marginal density $\Pi^u_1$ is not available in closed form.
To circumvent this, we propose a tractable entropy lower bound that preserves the convergence properties of the soft Bellman operator.

\begin{proposition}[Entropy lower bound] \label{prop: entropy lower bound}
Consider the diffusion model defined in \eqref{eq: controlled sde} with deterministic initial condition $X_0=0$.
The marginal entropy of the resulting policy satisfies:
\begin{equation}
\label{eq: ent lb}
 \mathcal{H}(\pi^u(a|s)) \geq \mathcal{L}^s_{\mathrm{ENT}}(u) \coloneqq -\E_{X\sim \Pi^u}\left[\int_0^1 \frac{1}{2}\|u(X_\tau, s,\tau)\|^2 \dd \tau + \log \frac{\Pi^0_1(X_1|s)}{\big|\det J_f(X_1)\big|}\right].
\end{equation}
\end{proposition}

This lower bound allows us to define a practical soft Bellman evaluation operator $\mathcal{T}^\pi$ acting on the $Q$-function $Q:\SP\times\AS\to\R$.
For a fixed control $u$ inducing the policy $\pi^u$, the operator is defined as:
\begin{align}
(\mathcal{T}^{\pi^u} Q)(a,s) \coloneqq r(s,a) + \gamma \mathbb{E}_{s' \sim p, a' \sim \pi^u} \big[Q(a',s') + \alpha  \mathcal{L}^{s'}_{\mathrm{ENT}}(u)\big],
\label{eq: soft bellman}
\end{align}
where the entropy surrogate $\alpha \mathcal{L}^{s'}_{\mathrm{ENT}}(u)$ is used in the diffusion setting.
The bound in \eqref{eq: ent lb} is tight if and only if the controlled process lies in the reciprocal class of the reference process (Remark~\ref{remark: tightness}).
Since the fixed point of the policy improvement step below belongs to this class, \eqref{eq: soft bellman} then coincides with the standard soft Bellman operator and its fixed point is the soft $Q$-function of the current policy.

\paragraph{Policy improvement.}
Given the current $Q$-function at temperature $\alpha$, we update the diffusion policy by finding a new drift $u$ that induces a marginal action distribution $\pi^u(a|s)$ that matches the improved target distribution $\pi^*(a|s) \propto \exp(Q(a, s)/\alpha)$.
In our framework, this improvement step is exactly the adjoint matching problem described in Section~\ref{sec: preliminaries}.
Specifically, we update the control $u$ by performing a fixed-point iteration on the loss $\mathcal{L}^s_{\mathrm{AMDP}}(u)$.
As shown in \cite{domingoenrich2025adjoint,havens2025adjoint}, the only fixed-point is the optimal control $u^*$ with path measure $\Pi^{u^*}$.
Since $\Pi^{u^*}=\Pi^*$ it particularly holds $\Pi^{u^*}_1=\Pi^*_1$.
After a change of variables from $X_1$ to $a$, we have that $\pi^u(a|s) = \pi^*(a|s)$ as desired.

\subsection{Trust-region updates}
It is well known that policy updates can be unstable in practice, particularly in on-policy RL where the target distribution shifts rapidly \cite{schulman2015trust,schulman2017proximal}.
Inspired by prior works on trust region optimization \cite{schulman2015trust,otto2021differentiable}, we propose the following constrained optimization scheme:
\begin{equation}
\label{eq: constrained ram loss}
    \min_u \  \mathcal{L}^s_{\mathrm{AMDP}}(u) \quad \mathrm{s.t.}\ \underbrace{\int_0^1\E_{X_1\sim\Pi^{\bar u}_1, X_\tau\sim\Pi^0_{\tau|1}}\left[ \frac{1}{2}\|u(X_\tau,s,\tau)-u_{\mathrm{old}}(X_\tau,s,\tau)\|^2\right]\dd \tau}_{\mathcal{L}^s_{\mathrm{TR}}(u)} \leq \varepsilon,
\end{equation}
where $\varepsilon\geq 0$ denotes the trust-region (TR) bound and $u_{\mathrm{old}}$ is the control function from a previous iteration.
The motivation behind \eqref{eq: constrained ram loss} is to prevent overly aggressive updates to $u$, which could lead to a collapse of the action distribution or divergence in the $Q$-function estimation.
By explicitly penalizing the deviation of the new drift $u$ from the previous iterate $u_{\mathrm{old}}$, we ensure that the diffusion model remains within a local neighborhood where the current $Q$-function estimate is reliable.

We approach this constrained problem using a relaxed Lagrangian formalism, considering the loss functional:
\begin{equation}
\label{eq: Lagrangian loss}
    \mathcal{L}^s_\mathrm{LAG}(u, \lambda) = \mathcal{L}^s_{\mathrm{AMDP}}(u) + \lambda\left( \mathcal{L}^s_{\mathrm{TR}}(u) - \varepsilon \right),
\end{equation}
where $\lambda \geq 0$ is a Lagrange multiplier.
We then solve the saddle point problem $\max_{\lambda \ge 0} \, \min_{u} \mathcal{L}_\mathrm{LAG}(u, \lambda)$.
Since the objective $\mathcal{L}^s_{\mathrm{AMDP}}$ and the constraint $\mathcal{L}^s_{\mathrm{TR}}$ are both quadratic and thus convex in $u$, and assuming the feasible region is non-empty (which holds for any $\varepsilon \geq 0$ as $u = u_{\mathrm{old}}$ is a feasible point), strong duality holds \cite{boyd2004convex}.
Dual descent on $\lambda$ therefore converges to the unique optimal solution, and the inclusion of the trust-region constraint does not bias the final solution of the policy improvement step.
\begin{proposition}[Fixed-point preservation] \label{prop: fixed point preservation}
For any $\lambda \geq 0$, the optimization of $\mathcal{L}^s_\mathrm{LAG}(u, \lambda)$ preserves the unique fixed point of the unconstrained reciprocal adjoint matching loss $\mathcal{L}^s_{\mathrm{AMDP}}(u)$ in \eqref{eq: amdp loss}.
\end{proposition}
In practice, we follow REPPO \cite{voelcker2026relative} and update $\lambda$ and $u$ by joint optimizer steps, with $\lambda$ settling over many steps.
We consider \eqref{eq: amdp loss} the core objective of \texttt{AMDP} and the trust-region loss \eqref{eq: Lagrangian loss} an extension required for the on-policy setting.
The on-policy instantiation in Algorithm~\ref{alg: onpolicy amdp full} optimizes \eqref{eq: Lagrangian loss}, whereas the off-policy experiments (Fig.~\ref{fig:off_policy_dmc_aggregated_main_results}) use \eqref{eq: amdp loss} without a trust region.

\section{Experiments}
We evaluate \texttt{AMDP}'s performance on 63 different environments from the MuJoCo Playground \cite{mujoco_playground_2025}, built on MuJoCo XLA (MJX), ManiSkill \cite{gu2023maniskill2}, and the high-dimensional HumanoidBench \citep{sferrazza2024humanoidbench} benchmark suites in the on-policy setting.
We compare \texttt{AMDP} against strong Gaussian policy baselines REPPO \cite{voelcker2026relative}, PPO \cite{schulman2017proximal}, SPO \cite{xie2025simple}, and against the diffusion policy approaches DPPO \citep{RenLAS0MBDS25}, FPO \citep{mcallister2026flow}, DIME \citep{celik2025dime}.
Here, DPPO and FPO approximate a trust region using a PPO-like clipped objective, framing them particularly suitable for the on-policy RL setting, whereas DIME is a state-of-the-art maximum entropy RL algorithm and optimizes the reverse KL.

Furthermore, we show \texttt{AMDP}'s flexibility by additionally analyzing its performance on DMC's \citep{tunyasuvunakool2020dm_control} seven complex dog and humanoid environments in the off-policy setting, where sample efficiency is crucial.
Here, we compare against the strong diffusion-based baselines DIME \citep{celik2025dime}, QSM \citep{PsenkaEA024}, Diff-QL\citep{WangHZ23} and Consistency-AC \citep{ding2024consistency}.

Every experiment is repeated for ten seeds, and we report aggregated returns using the interquartile mean (IQM) with 95\% confidence intervals based on stratified bootstrapping, as recommended by \citet{agarwal2021deep}.
Further experimental details can be found in Appendix~\ref{appendix: experimental setup} and per environment results can be found in Appendix~\ref{appendix: further experimental results}.

\begin{figure*}[t]
    \centering
    \small
    \legendMainResultsAMDPBoN

    \vspace{0.35em}

    \makebox[\textwidth][c]{%
    \begin{subfigure}{0.245\textwidth}
        \centering
        \includegraphics[width=\linewidth]{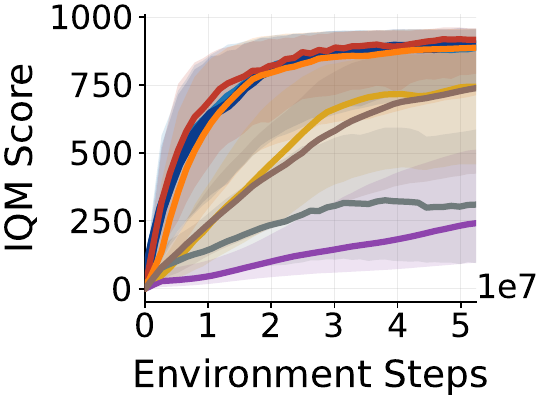}
        \caption{MJX DMC}
        \label{fig:main_results_mjx_dmc}
    \end{subfigure}
    \begin{subfigure}{0.245\textwidth}
        \centering
        \includegraphics[width=\linewidth]{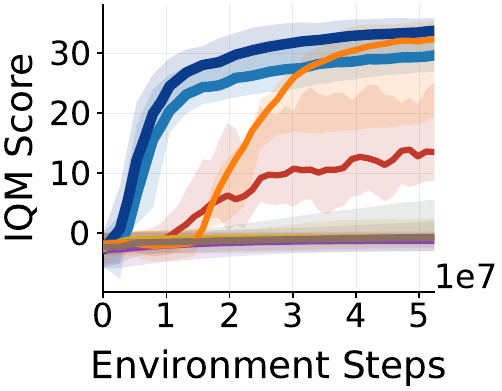}
        \caption{MJX Humanoid}
        \label{fig:main_results_mjx_humanoid}
    \end{subfigure}
    \begin{subfigure}{0.245\textwidth}
        \centering
        \includegraphics[width=\linewidth]{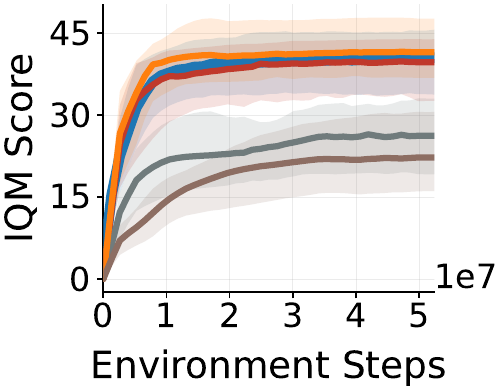}
        \caption{ManiSkill3}
        \label{fig:main_results_maniskill}
    \end{subfigure}
    \begin{subfigure}{0.245\textwidth}
        \centering
        \includegraphics[width=\linewidth]{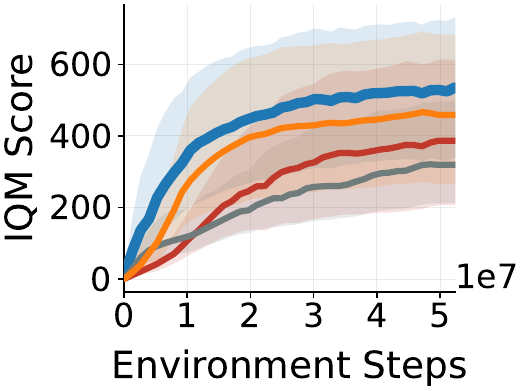}
        \caption{HumanoidBench}
        \label{fig:main_results_humanoid_bench}
    \end{subfigure}%
    }

    \caption{\textbf{On-policy performance comparison.}
    Aggregated IQM learning curves across (a) MuJoCo Playground DMC, (b) MuJoCo Playground humanoid locomotion, (c) ManiSkill3, and (d) HumanoidBench.
    Shaded regions show 95\% bootstrap confidence intervals.}
    \label{fig:main_results_all}
\end{figure*}

\paragraph{On-policy RL results.}
Fig. \ref{fig:main_results_mjx_dmc} visualizes the learning curves of all methods in the MuJoCo Playground DMC benchmark.
\texttt{AMDP} matches DIME's \citep{celik2025dime} performance and converges slightly faster than REPPO in terms of sample efficiency.
Importantly, \texttt{AMDP} has a comparable wallclock speed (Table~\ref{tab:runtime-narrow}) compared to the highly efficient, Gaussian-based REPPO algorithm, indicating that diffusion policies can be efficiently trained with \texttt{AMDP}.
On the more sophisticated humanoid locomotion tasks in MuJoCo playground (Fig. \ref{fig:main_results_mjx_humanoid}), \texttt{AMDP} starts learning significantly faster than all baselines, converging to high returns, although REPPO eventually surpasses \texttt{AMDP}'s final performance slightly.
However, following previous works in the literature \citep{Chen0Y0023,IDQL_HansenEstruch2023,Kang0DPY23,Mao0Z0Z24,FangLZWJ25,dong2026meanflowpolicyoptimization,gao2025behaviorregularized}, we can reduce stochasticity during evaluation by proposing $N=16$ samples to the Q function and choosing the one with the highest return (\texttt{AMDP BoN}).
This stochasticity reduction significantly improves \texttt{AMDP}'s performance above all baselines.
Interestingly, DIME performs poorly, indicating the need for a trust region constrained update mechanism in complex on-policy methods.
In the tabletop robotic manipulation benchmarks from ManiSkill (Fig. \ref{fig:main_results_maniskill}) \texttt{AMDP}, DIME, and REPPO perform similarly, whereas PPO and SPO perform worse.
Contrary to this, Fig. \ref{fig:main_results_humanoid_bench} shows that in the more complex and high-dimensional HumanoidBench environments that include whole-body manipulation tasks, \texttt{AMDP} outperforms all baselines by a large margin even without best-of-N sampling, indicating that \texttt{AMDP} scales well to more complex tasks.

\begin{wrapfigure}[16]{r}{0.38\textwidth}
    \centering
    \vspace{-1.2em}
    {\footnotesize \begin{tabular}{@{}l@{\hspace{0.6em}}l@{}}
  \legendline{offamdp}{AMDP (Ours)} & \legendline{offdime}{DIME}      \\
  \legendline{offqsm}{QSM}          & \legendline{offdiffql}{Diff-QL} \\
  \multicolumn{2}{@{}l@{}}{\legendline{offcac}{Consistency-AC}}
  \end{tabular}}

    \vspace{0.3em}
    \includegraphics[width=0.34\textwidth]{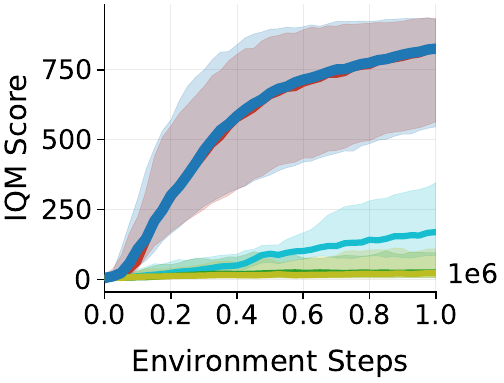}
    \vspace{-0.6em}
    \caption{\textbf{Off-policy performance comparison.}
    Aggregated IQM learning curves on the DMC dog and humanoid tasks, shaded as in Fig.~\ref{fig:main_results_all}.}
    \label{fig:off_policy_dmc_aggregated_main_results}
\end{wrapfigure}
\paragraph{Off-policy RL results.}
In addition to the on-policy benchmarks, we evaluate \texttt{AMDP} on the high-dimensional dog and humanoid tasks (Fig.~\ref{fig:off_policy_dmc_aggregated_main_results}).
\texttt{AMDP} matches DIME's performance and converges slightly faster on the dog environments (see Appendix~\ref{appendix: further experimental results}), showing that \texttt{AMDP}'s adjoint matching objective is a competitive policy optimization loss against DIME's reverse KL loss in the off-policy RL setting.

\paragraph{Algorithm ablations.}
We systematically ablate all changes on the MuJoCo Playground DMC and Humanoid tasks.
We start by analyzing whether the specific SDE design choice (Dirac prior $\mu_0$) is responsible for \texttt{AMDP}'s performance.
For this, we optimize the diffusion policy with DIME's reverse KL objective and similarly do not leverage a trust region constraint.
This provides a direct comparison against DIME, which uses a different SDE structure \cite{richter2023improved}, isolating the design choice.
Fig.~\ref{fig:ablations_reverse_kl_dmc}~and~\ref{fig:ablations_reverse_kl_humanoid} show matching performance, although DIME shows reduced variance.
Furthermore, we switch to an Ornstein-Uhlenbeck (OU) process (see Appendix~\ref{appendix: OU process} for details) with our loss (\texttt{AMDP-OU}) and see comparable performance.
This indicates that \texttt{AMDP}'s performance benefit is not inherent to the chosen SDE.
In contrast, it could potentially be improved by utilizing less noise-demanding SDEs.

Next, we analyze the effect of different squashing functions on the performance in Fig.~\ref{fig:ablations_squashing_dmc}~and~\ref{fig:ablations_squashing_humanoid}.
Hyperbolic tangent squashing performs similarly on the low-dimensional control tasks, indicating good performance for any approximately bang-bang policy.
However, for action selection in the high-dimensional humanoid environments, the error function improves performance significantly, which indicates that the numerical properties (see Fig. \ref{table: error function plots} and Appendix \ref{appendix: action squashing}) benefit optimization within the action bounds.

Moreover, we analyze the need for a trust region control during the policy update in diffusion-based on-policy RL by disabling the trust region and by varying the different trust region constraints $\varepsilon$ in Fig.~\ref{fig:ablations_tr_size_dmc}~and~\ref{fig:ablations_tr_size_humanoid}.
In both cases, the performance drops significantly without the trust-region bound ($\varepsilon=\infty$), whereas smaller bounds ($\varepsilon=0.01$) converge more slowly.

Lastly, we utilize the simulation-free nature of \texttt{AMDP}'s objective and evaluate inference diffusion step counts up to 128 in Fig.~\ref{fig:ablations_diff_steps_dmc} and in Fig.~\ref{fig:ablations_diff_steps_humanoid}.
Additionally, we sample multiple intermediate times within \eqref{eq: Lagrangian loss} (\emph{batch repetition}) for each stored denoised action, which is feasible due to the expectation under the reciprocal projection.
This creates a training signal for a broad range of time, just as the chain backpropagation of the reverse KL.
Note that this data generation does not incur further SDE simulation or evaluation of the $Q$-function, making it extremely cheap to compute.
Using significantly fewer steps can be beneficial in terms of sample efficiency, at the cost of a small final performance degradation in the MuJoCo Playground's control environments (Fig.~\ref{fig:ablations_diff_steps_dmc}), while the complex humanoid tasks require more diffusion steps and improve slightly with the added precision (Fig.~\ref{fig:ablations_diff_steps_humanoid}).

\begin{figure*}[t]
    \centering
    \small
    \legendDiffusionStepsAndReverseKL
    \vspace{0.35em}

    \makebox[\textwidth][c]{%
    \begin{subfigure}{0.245\textwidth}
        \centering
        \includegraphics[width=\linewidth]{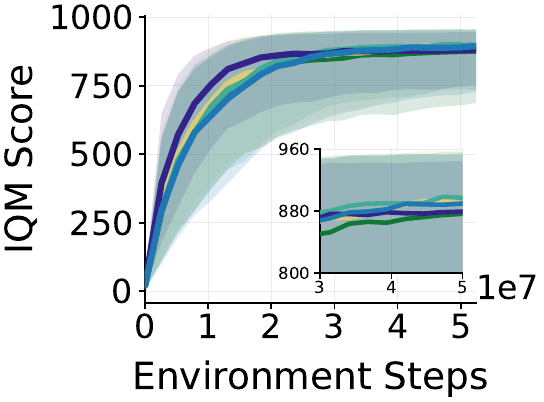}
        \caption{Diff. steps, DMC}
        \label{fig:ablations_diff_steps_dmc}
    \end{subfigure}
    \hspace{0.005\textwidth}
    \begin{subfigure}{0.245\textwidth}
        \centering
        \includegraphics[width=\linewidth]{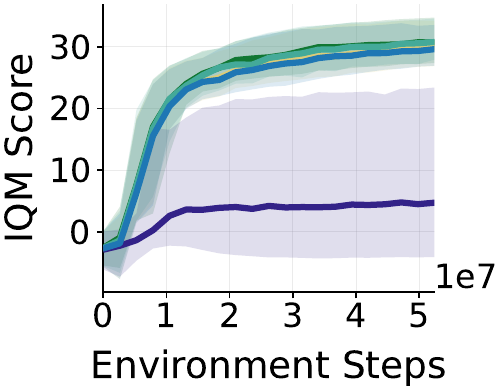}
        \caption{Diff. steps, humanoid}
        \label{fig:ablations_diff_steps_humanoid}
    \end{subfigure}
    \hspace{0.005\textwidth}
    \begin{subfigure}{0.245\textwidth}
        \centering
        \includegraphics[width=\linewidth]{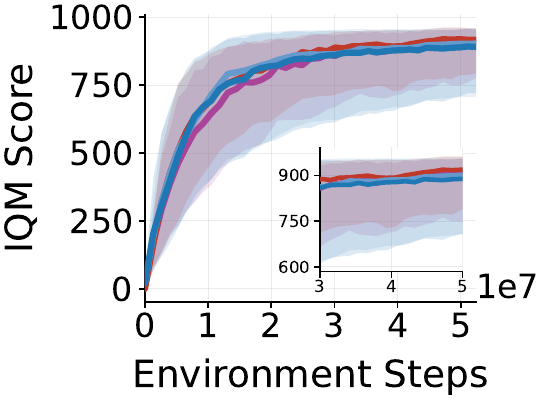}
        \caption{Objective, DMC}
        \label{fig:ablations_reverse_kl_dmc}
    \end{subfigure}
    \hspace{0.005\textwidth}
    \begin{subfigure}{0.245\textwidth}
        \centering
        \includegraphics[width=\linewidth]{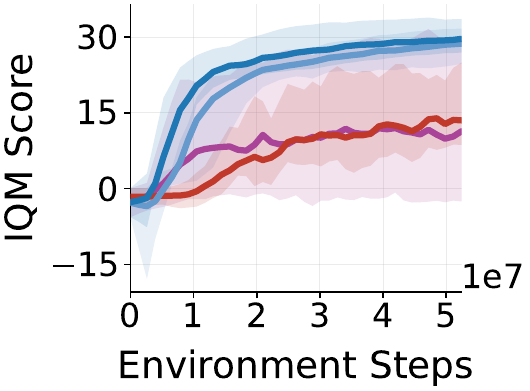}
        \caption{Objective, humanoid}
        \label{fig:ablations_reverse_kl_humanoid}
    \end{subfigure}%
    }

    \caption{\textbf{Ablations on diffusion discretization and policy improvement objective.}
    Left pair varies diffusion steps and reciprocal batch repetitions (one, unless noted); right pair compares the default adjoint-matching loss against reverse-KL optimization and DIME.
    Results are aggregated over MuJoCo Playground DMC and humanoid tasks.}
    \label{fig:ablations_diffusion_and_reverse_kl}
    \vspace{-0.1cm}
\end{figure*}

\begin{table}[b]
    \centering
    \caption{Runtime per training iteration in ms.
    \emph{Environment Rollout Time} covers action sampling and environment stepping, \emph{Network Update Time} the critic and actor updates.
    \emph{rev.\ KL} replaces \eqref{eq: Lagrangian loss} with DIME's reverse KL objective.
    Table~\ref{tab:runtime-results} reports five environments.}
    \label{tab:runtime-narrow}
    \small
    \setlength{\tabcolsep}{5pt}
    \begin{tabular}{@{}llrrrrr@{}}
        \toprule
                               &                          & REPPO    & \multicolumn{2}{c}{\texttt{AMDP}} & \multicolumn{2}{c}{rev.\ KL} \\
        \cmidrule(lr){4-5} \cmidrule(l){6-7}
                               & Diffusion steps          & --       & 16        & 128       & 16        & 128       \\
        \midrule
        CartpoleBalance        & Environment Rollout Time & $84$     & $404$     & $2\,738$  & $404$     & $2\,751$  \\
                               & Network Update Time      & $948$    & $994$     & $991$     & $9\,944$  & $71\,659$ \\
        G1JoystickRoughTerrain & Environment Rollout Time & $9\,605$ & $10\,031$ & $13\,097$ & $10\,041$ & $13\,129$ \\
                               & Network Update Time      & $1\,014$ & $1\,113$  & $1\,117$  & $11\,421$ & $82\,101$ \\
        \bottomrule
    \end{tabular}
\end{table}
\paragraph{Wallclock efficiency and actor scaling capability.}
\texttt{AMDP}'s important benefit is that the gradients do not need to be backpropagated through the diffusion process, allowing efficient training and scaling the training to bigger and more complex models.
We report the wallclock time for the training of the different algorithms with matched critic and actor capacities for \texttt{AMDP}, for the reverse KL objective without a trust region and the fast and performant Gaussian baseline REPPO in Table~\ref{tab:runtime-narrow}.
Our method shows only a small constant overhead of roughly 10\% for all updates compared to REPPO, independently of the number of diffusion steps.
For G1JoystickRoughTerrain, this equates to about 0.1\,s, and with diffusion sampling adding 0.4\,s to the default 16-step rollout, a full iteration is only 5\% slower than REPPO.
For the extremely fast CartpoleBalance simulation, where diffusion sampling is most noticeable, the end-to-end slowdown is 36\%.
In contrast, the reverse KL objective, which samples new actions and backpropagates through the diffusion chain in each minibatch, is already ten times slower with 16 diffusion steps and approximately doubles the runtime, while 128 diffusion steps incur a 72-80x slower update.
We further train larger diffusion policy models (up to 106M parameters, i.e., approximately 300 times larger) as a proof of concept and show the results in Appendix~\ref{appendix: further experimental results}.
As the critic size is kept fixed, scaling only the actor does not yield large performance improvements but this validates that \texttt{AMDP} can train policies with much larger actors efficiently.
Moreover, the memory footprint of \texttt{AMDP} is also independent of the number of diffusion steps and small enough (Appendix~\ref{appendix: vram}) to allow for training a 1B-parameter actor on a single GPU.

\begin{figure*}[t]
    \centering
    \small
    \legendTrustRegionSizeAndSquashing

    \vspace{0.35em}

    \makebox[\textwidth][c]{%
    \begin{subfigure}{0.245\textwidth}
        \centering
        \includegraphics[width=\linewidth]{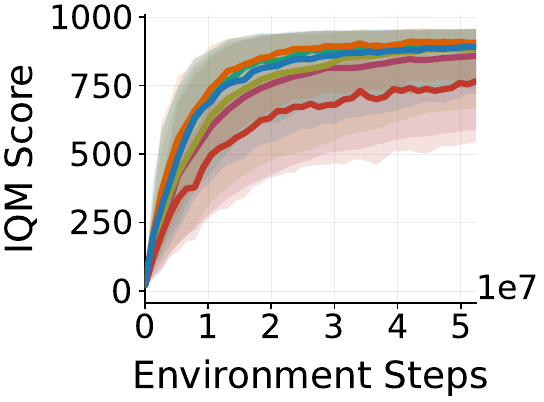}
        \caption{TR size, DMC}
        \label{fig:ablations_tr_size_dmc}
    \end{subfigure}
    \hspace{0.005\textwidth}
    \begin{subfigure}{0.245\textwidth}
        \centering
        \includegraphics[width=\linewidth]{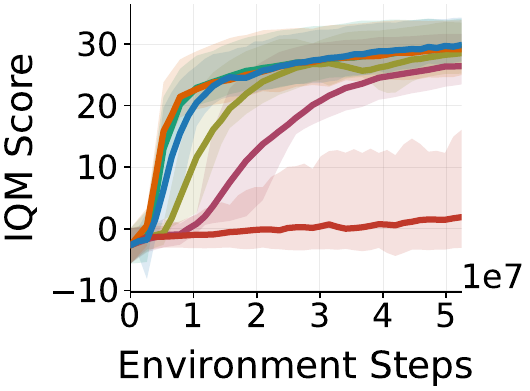}
        \caption{TR size, humanoid}
        \label{fig:ablations_tr_size_humanoid}
    \end{subfigure}
    \hspace{0.005\textwidth}
    \begin{subfigure}{0.245\textwidth}
        \centering
        \includegraphics[width=\linewidth]{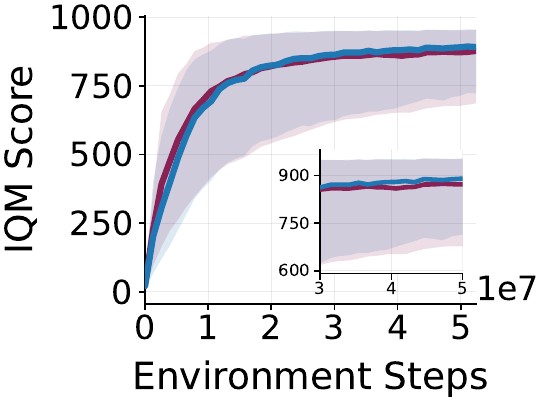}
        \caption{Squashing, DMC}
        \label{fig:ablations_squashing_dmc}
    \end{subfigure}
    \hspace{0.005\textwidth}
    \begin{subfigure}{0.245\textwidth}
        \centering
        \includegraphics[width=\linewidth]{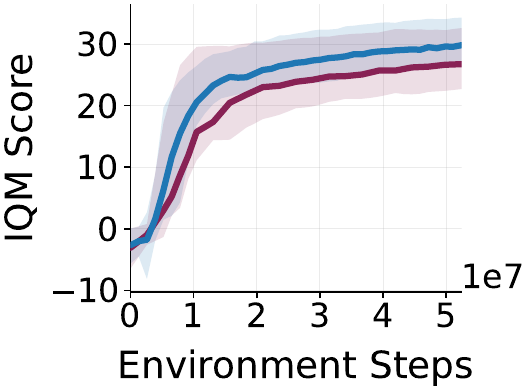}
        \caption{Squashing, humanoid}
        \label{fig:ablations_squashing_humanoid}
    \end{subfigure}%
    }

    \caption{\textbf{Ablations on trust-region size and action squashing.}
    Left pair varies the trust-region bound $\varepsilon$, including the no-trust-region setting $\varepsilon=\infty$.
    Right pair compares error-function squashing against tanh squashing.
    Results are aggregated over MuJoCo Playground DMC and humanoid tasks.}
    \label{fig:ablations_tr_size_and_squashing}
\end{figure*}

\section{Conclusion}
In this work, we introduced Adjoint Matching Diffusion Policy (\texttt{AMDP}), a novel and computationally scalable framework for training diffusion-based policies in continuous reinforcement learning.
By formulating maximum entropy RL as a stochastic optimal control problem, we leveraged reciprocal adjoint matching to train expressive diffusion policies in a simulation-free manner.
This approach effectively circumvents the prohibitively high memory costs associated with backpropagating through the diffusion process and the theoretical shortcomings of truncated Langevin dynamics.
Furthermore, we proposed error function action squashing, which elegantly simplifies the training objective and ensures numerical stability, alongside a trust-region update mechanism that stabilizes policy iteration in online settings while preserving the fixed point of the policy improvement step.
Our empirical results demonstrate that \texttt{AMDP} achieves competitive performance across diverse continuous control benchmarks with low computational overhead, providing a theoretically grounded and computationally efficient pathway for deploying expressive generative policies in RL.

\paragraph{Limitations.}
\texttt{AMDP} uses a deterministic initial condition, whereas diffusion policies pretrained on offline data typically use a Gaussian prior \cite{reuss2023goal, chi2025diffusion}.
The Ornstein-Uhlenbeck extension in Appendix~\ref{appendix: OU process} lifts this restriction, but a full offline-to-online pipeline as in QAM \cite{li2026reinforcement} is left for future work.
Furthermore, we do not show that larger or more expressive actors improve performance, and whether jointly scaling actor and critic \cite{nauman2024bigger} helps remains open.
We discuss both limitations in detail in Appendix~\ref{appendix: broader_impact_limitations}.

\newpage
\input{chapters/acknowledgements.tex}

\bibliographystyle{plainnat}
\bibliography{neurips_2026}

\newpage
\appendix

\startcontents[appendix]

\printcontents[appendix]{}{1}{\section*{Appendix}}
\clearpage

\section{Broader impact and limitations}
\label{appendix: broader_impact_limitations}

\paragraph{Broader impact.}
Our work advances the fundamental capabilities of reinforcement learning by making highly expressive diffusion policies computationally accessible and stable to train.
On a positive note, these improvements can significantly benefit the development of advanced robotic systems, enabling them to solve complex, high-dimensional, and contact-rich tasks---such as those found in manufacturing, healthcare assistance, and household automation---more reliably.
Furthermore, transitioning from simulation-heavy training pipelines to our simulation-free reciprocal adjoint matching reduces the energetic and computational footprint associated with training deep reinforcement learning models.

However, as with any general-purpose advancement in autonomous control, there are potential negative societal impacts.
Highly capable autonomous systems could accelerate the displacement of human labor in certain industrial or logistical sectors.
Furthermore, the ability to efficiently train complex, multi-modal control policies is a dual-use technology, potentially being integrated into autonomous weapons, drones, or surveillance systems without adequate human oversight.
We encourage the machine learning and robotics communities to continue developing rigorous safety, alignment, and ethical deployment frameworks to ensure these powerful control algorithms are utilized responsibly.

\paragraph{Limitations.}
\label{appendix: limitations}
A primary limitation of our current framework lies in its initialization strategy.
To ensure the memoryless property required for tractable adjoint matching, \texttt{AMDP} restricts the reference process to a deterministic initial condition, $X_0=0$, in the main text.
In contrast, existing diffusion policies pre-trained on offline behavioral data typically employ a stochastic Gaussian prior, $\mu_0=\mathcal{N}(0,I)$ \cite{reuss2023goal, chi2025diffusion}.
This architectural discrepancy complicates the integration of our method into offline-to-online RL pipelines, where one might wish to pre-train a diffusion policy on static datasets before fine-tuning it via \texttt{AMDP} in an online environment.
The Ornstein-Uhlenbeck reference process with a Gaussian prior and memoryless noise schedule \cite{domingoenrich2025adjoint} in Appendix~\ref{appendix: OU process} bridges this gap and performs comparably in the on-policy setting (\texttt{AMDP-OU}, Fig.~\ref{fig:ablations_reverse_kl_dmc}), but a full offline-to-online pipeline additionally requires a compatible pretraining setup.
We leave the exploration of this offline-to-online integration, i.e., adapting the exact reciprocal adjoint formulation of \texttt{AMDP} to the setting of QAM \cite{li2026reinforcement} without solving the adjoint ODE, for future work.
Furthermore, while \texttt{AMDP} enables training large diffusion actors at low computational cost, we do not show that larger or more expressive actors improve performance, as scaling the actor with a fixed critic does not yield large improvements (Appendix~\ref{appendix: larger models}).
Whether jointly scaling actor and critic \cite{nauman2024bigger, lee2025hyperspherical, palenicek2026xqc} improves performance is left for future work.

\section{Technical details}
\label{appendix: technical details}

\subsection{Additional notation}
\label{appendix: additional notation}
We denote by $\mathcal{U} \subset C(\R^{d} \times \mathcal{S} \times [0,1] ; \R^{d})$ the set of admissible controls and by $\mathcal{P}$ the set of all probability measures on $C([0,1],\R^{d})$.
We define the state-conditioned path space measure $\Pi(\cdot|s) \in \mathcal{P}$ as the law of an $\R^{d}$-valued stochastic process $X = (X_\tau)_{\tau\in[0,1]}$ and we denote by $\Pi_\tau(\cdot|s)$ the marginal distribution at time $\tau$.
For vectors $v_1, v_2 \in \R^{d}$, we denote by $\|v\|$ the Euclidean norm and by $v_1 \cdot v_2$ the Euclidean inner product.
For a sufficiently smooth function $f \colon \R^{d} \times [0,1] \to \R$, we denote by $\nabla_x f$ its gradient w.r.t.\@ the spatial variables $x \in \R^{d}$.
We denote by $\mathcal{N}(\mu,\Sigma)$ a multivariate normal distribution with mean $\mu \in \R^{d}$ and covariance matrix $\Sigma \in \R^{d \times d}$.
For random variables $X_1$, $X_2$, we denote by $\E[X_1]$ and $\V[X_1]$ the expectation and variance of $X_1$ and by $\E[X_1|X_2]$ the conditional expectation of $X_1$ given $X_2$.

\subsection{Technical assumptions}
\label{appendix: technical assumptions}
 
Throughout our work, we adapt standard assumptions from stochastic optimal control to ensure all objects considered are well-defined \cite{nusken2021solving, domingo2024stochastic, blessing2025trust}.
Namely, we assume that:
\begin{enumerate}
    \item The set $\mathcal{U}$ of \textit{admissible controls} is given by
    \begin{align}
        \mathcal{U} = \{ u \in C^1(\R^{d} \times \mathcal{S} \times [0,1] ; \R^{d}) \mid \exists C > 0, \, \forall (x,s,\tau), \, \|u(x,s,\tau)\| \leq C(1+\|x\|) \}. 
    \end{align}
    \item The base drift $b$ and diffusion coefficient $\sigma$ of the controlled process
    \begin{align}
        \dd X_\tau = \big(b(X_\tau,\tau) + \sigma(\tau) u(X_\tau, s,\tau)\big) \dd \tau + \sigma(\tau) \dd B_\tau, \quad X_0 \sim \mu_0,
        \label{eq: controlled sde base drift}
    \end{align}
    are continuous.
    In our primary setting, we assume a zero base drift ($b=0$), which recovers~\eqref{eq: controlled sde}, and a time-dependent, state-independent diffusion coefficient $\sigma(\tau) > 0$ for all $\tau \in [0, 1]$, which trivially satisfies the standard linear growth and ellipticity conditions.
    The Ornstein-Uhlenbeck reference \eqref{eq: ou reference} of Appendix~\ref{appendix: OU process} has the linear base drift $b(x,\tau) = -\tfrac{1}{2}\beta(\tau)x$.
\end{enumerate}
As in prior work on maximum entropy RL with continuous actions \cite{haarnoja2017reinforcement,haarnoja2018soft,celik2025dime}, the policy iteration guarantees of Section~\ref{sec policy iteration} hold for exact representations of $Q$ and $u$ and for convergence within the policy class induced by $\mathcal{U}$, whereas a practical implementation performs incremental stochastic updates with function approximators, for which no convergence guarantee is given.
Similarly to above, we assume for well-defined objects in maximum entropy RL that:
\begin{enumerate}
    \item The reward $r$ is bounded, the action space $\AS = [-1,1]^d$ is bounded, and the discount factor satisfies $\gamma < 1$, which implies an upper bound on the soft $Q$-values \cite{haarnoja2017reinforcement,haarnoja2018soft}.
     The admissible control class $\mathcal{U}$ implies a lower bound on the entropy surrogate $\mathcal{L}^s_{\mathrm{ENT}}(u)$ and thus on the soft $Q$-values.
    \item The temperature satisfies $\alpha > 0$ and the log-partition function $\log \Z(s)$ in \eqref{eq: energy policy} is finite for all $s \in \SP$.
\end{enumerate}

\subsection{Useful identities}
\label{appendix: useful identities}

\begin{definition}[Controlled SDEs]
Let $u \in \mathcal{U}$ be a control function.
Conditioned on an MDP state $s \in \mathcal{S}$, we consider the controlled and uncontrolled (reference) stochastic processes defined via the SDEs:
\begin{align}
    \dd X_\tau &= \sigma(\tau) u(X^u_\tau, s, \tau) \dd \tau + \sigma(\tau) \dd B_\tau, \quad && X_0 = 0, \\
    \dd X_\tau &=  \sigma(\tau) \dd B_\tau, \quad && X_0 = 0,
\end{align}
where $X \sim \Pi^u(\cdot|s)$ and $X \sim \Pi^0(\cdot|s)$, with $\Pi^u$ and $\Pi^0$ denoting the respective path space measures, and $B_\tau$ is a standard Brownian motion.
\end{definition}

\begin{theorem}[Girsanov's theorem for path measures] \label{theorem: girsanov}
Let $u \in \mathcal{U}$ and assume $\sigma(\tau)$ is invertible.
The Radon-Nikodym derivative of the controlled measure $\Pi^u(\cdot|s)$ with respect to the reference measure $\Pi^0(\cdot|s)$, evaluated along a path $X$, is given by:
\begin{align}
\label{eq: girsanov}
    \log \frac{\dd \Pi^{u}}{\dd \Pi^0}(X^u|s) = \int_0^1 u(X_\tau,s,\tau) \cdot \dd B_\tau + \frac{1}{2}\int_0^1 \|u(X_\tau,s,\tau)\|^2 \dd \tau.
\end{align}
\end{theorem}

\begin{proof}
    See \cite{oksendal2003stochastic, nusken2021solving}.
\end{proof}

Taking the expectation of \eqref{eq: girsanov} over $\Pi^u$ directly yields the Kullback-Leibler divergence between the path measures, which forms the basis for the entropy lower bound and adjoint matching objectives:
\begin{equation}
    \KL(\Pi^u(\cdot|s) \mid \Pi^0(\cdot|s)) = \E_{X \sim \Pi^u}\left[ \int_0^1 \frac{1}{2}\|u(X_\tau, s,\tau)\|^2 \dd \tau \right].
\end{equation}

\section{Proofs}
\label{appendix: proofs}

\paragraph{Proof of Proposition~\ref{prop: entropy lower bound}.}
The proof relies on the relationship between path-space measures and their terminal marginals.
First, by Girsanov's theorem (see~\Cref{theorem: girsanov}), the Kullback-Leibler (KL) divergence between the path-space measure of the controlled process $\Pi^u(\cdot|s)$ and the uncontrolled reference process $\Pi^0(\cdot|s)$ over the interval $[0, 1]$ is given by the expected kinetic energy of the control drift:
\begin{equation}
    \KL(\Pi^u(\cdot|s) \mid \Pi^0(\cdot|s)) = \E_{X \sim \Pi^u}\left[ \int_0^1 \frac{1}{2}\|u(X_\tau, s,\tau)\|^2 \dd \tau \right].
\end{equation}
By the data processing inequality (or the chain rule for KL divergence) \cite{cover1999elements}, the KL divergence between the full path measures strictly upper bounds the KL divergence between their terminal marginal distributions at $\tau = 1$:
\begin{equation}
\label{eq: kl marginal bound}
    \KL(\Pi^u_1(\cdot|s) \mid \Pi^0_1(\cdot|s)) \leq \KL(\Pi^u(\cdot|s) \mid \Pi^0(\cdot|s)).
\end{equation}
Expanding the left-hand side of \eqref{eq: kl marginal bound}, we have:
\begin{align}
    \KL(\Pi^u_1(\cdot|s) \mid \Pi^0_1(\cdot|s)) &= \E_{X_1 \sim \Pi^u_1}\left[ \log \Pi^u_1(X_1|s) - \log \Pi^0_1(X_1|s) \right] \nonumber \\
    &= -\mathcal{H}(\Pi^u_1(\cdot|s)) - \E_{X_1 \sim \Pi^u_1}\left[ \log \Pi^0_1(X_1|s) \right],
\end{align}
where $\mathcal{H}(\Pi^u_1(\cdot|s))$ is the differential entropy of the terminal latent state.
Substituting this expansion into \eqref{eq: kl marginal bound} and rearranging the terms yields a lower bound on the latent entropy:
\begin{equation}
\label{eq: latent entropy bound}
    \mathcal{H}(\Pi^u_1(\cdot|s)) \geq - \E_{X \sim \Pi^u}\left[ \int_0^1 \frac{1}{2}\|u(X_\tau, s,\tau)\|^2 \dd \tau + \log \Pi^0_1(X_1|s) \right].
\end{equation}
Next, we relate this latent entropy to the entropy of the executed policy.
The action $a$ is obtained via the bijective squashing transformation $a = f(X_1)$.
By the change of variables formula for differential entropy, the entropy of the squashed policy $\pi^u(a|s)$ is given by:
\begin{equation}
    \mathcal{H}(\pi^u(\cdot|s)) = \mathcal{H}(\Pi^u_1(\cdot|s)) + \E_{X_1 \sim \Pi^u_1}\left[ \log \big|\det J_f(X_1)\big| \right].
\end{equation}
Adding the expected log-determinant of the Jacobian to both sides of the inequality \eqref{eq: latent entropy bound} yields:
\begin{align}\label{eq: squashed entropy bound}
    \mathcal{H}(\pi^u(\cdot|s)) &\geq - \E_{X \sim \Pi^u}\left[ \int_0^1 \frac{1}{2}\|u(X_\tau, s,\tau)\|^2 \dd \tau + \log \Pi^0_1(X_1|s) \right] + \E_{X_1 \sim \Pi^u_1}\left[ \log \big|\det J_f(X_1)\big| \right] \nonumber \\
    &= - \E_{X \sim \Pi^u}\left[ \int_0^1 \frac{1}{2}\|u(X_\tau, s,\tau)\|^2 \dd \tau + \log \frac{\Pi^0_1(X_1|s)}{\big|\det J_f(X_1)\big|} \right] \coloneqq \mathcal{L}^s_{\mathrm{ENT}}(u).
\end{align}
This completes the proof.

\begin{remark}[Tightness of the bound]\label{remark: tightness}
The lower bound is tight (i.e., holds with strict equality) if and only if $\KL(\Pi^u_1 \mid \Pi^0_1) = \KL(\Pi^u \mid \Pi^0)$.
From the chain rule of KL divergence, this equality occurs exactly when the conditional path distributions match: $\Pi^u_{|1}(X|X_1, s) = \Pi^0_{|1}(X|X_1, s)$.
In this regime, the controlled process $\Pi^u$ factors into the terminal marginal and the uncontrolled bridge process, meaning the optimal control solution forms a reciprocal class with the reference process.
In particular, the fixed point of reciprocal adjoint matching \cite{havens2025adjoint}, i.e., of the policy improvement step, belongs to this reciprocal class, such that the bound is tight and the fixed point of \eqref{eq: soft bellman} is the soft $Q$-function of the policy.
\end{remark}

\begin{remark}[Relation of Kinetic Energy and Entropy]\label{remark:kinetic energy and entropy}
Depending on the squashing function~$f$, a policy with higher kinetic energy can sample $X_1$ with such distribution that the term $\log \nicefrac{\Pi^0_1(X_1|s)}{|\det J_f(X_1)|}$ decreases on average and the entropy increases.
However, for the \emph{scaled error function} squashing, as defined in \eqref{eq: erf}, the kinetic energy directly upper bounds the entropy loss compared to the uniform distribution entropy at zero-control initialization (see Remark \ref{remark: erf entropy lower bound}).
\end{remark}

\paragraph{Proof of Proposition~\ref{prop: fixed point preservation}.}
Let $\Phi(u_{\mathrm{old}})$ be the unique minimizer of the unconstrained loss $\mathcal{L}^s_{\mathrm{AMDP}}(u)$ evaluated at $u_{\mathrm{old}}$.
Due to the quadratic nature of the AMDP and trust-region losses, the minimizer of $\mathcal{L}^s_\mathrm{LAG}(u, \lambda)$ is a convex combination:
\begin{equation}
    u^* = \frac{1}{1+\lambda} \Phi(u_{\mathrm{old}}) + \frac{\lambda}{1+\lambda} u_{\mathrm{old}}.
\end{equation}
Setting $c = \frac{1}{1+\lambda} \in (0, 1)$, we obtain the update rule $u_{i+1} = c \Phi(u_i) + (1-c) u_i$.
This corresponds to a damped fixed-point iteration.
Since $\Phi$ is an operator whose unique fixed-point is $u^*_{\mathrm{AMDP}}$, and the iteration is a convex combination of the current iterate and the target mapping, it follows that $u_{i+1} = u_i$ if and only if $u_i = \Phi(u_i)$, which occurs exactly at the optimal control.

\section{Related works}
\label{appendix: extended related work}

Diffusion models \cite{sohl2015deep,song2020denoising, ho2020denoising,karras2022elucidating} have been generally used in the context of Reinforcement Learning as trajectory planners \citep{JannerDTL22,AjayDGTJA23}, distributional critic models \citep{agrawalla2026floq,dong2026value}, synthetic data generation in the maximum entropy RL setting for Gaussian policies \citep{ishfaq2025langevin} or as expressive policy representations \citep{WangHZ23,Kang0DPY23,LuCEP2023,Mao0Z0Z24,FangLZWJ25,PsenkaEA024,ding2024diffusionbased,FangLZWJ25,celik2025dime,RenLAS0MBDS25,ding2026genpo,ma2025reinforcement,mcallister2026flow,lv2026flowbased,zhang2026reinflow,ZhangZG25,ma2025efficient,li2026qlearning,dong2026meanflowpolicyoptimization,qiu2026scoreflowcompletedistributionalcontrol,gong2026proximalpolicyoptimizationpath,lv2026flacmaximumentropyrl,goo2022knowboundariesnecessityexplicit,ding2024consistency,dong2025maximum,black2024training,liu2026flowgrpo,wang2025enhanceddaceralgorithmhigh,chen2025onestepflowpolicymirror,gao2025behaviorregularized,jain2025sampling,zhang2026sac,gao2026flowmatchingpolicyentropy,zhong2026reparameterizationflowpolicyoptimization,zhang2026gorl,le2026trudi}.
Here, we provide an overview of existing diffusion-based methods categorized into offline and online RL methods.

\subsection{Diffusion policies in online reinforcement learning}
Recently, several approaches have been proposed for training diffusion models within the online reinforcement learning paradigm.
In the following, we categorize these existing methods into distinct groups based on their underlying optimization algorithms.

\textbf{Reverse Kullback-Leibler divergence.}
Minimizing the reverse KL divergence between the diffusion policy and the target distribution naturally leads to the maximum entropy (or entropy-regularized) RL formulation \cite{haarnoja2017reinforcement,haarnoja2018soft}.
Early approaches optimize this objective by directly differentiating the expected state-action value, which inherently requires computationally expensive backpropagation through the entire diffusion process \cite{WangHZ23,wang2024diffusion, celik2025dime, lv2026flacmaximumentropyrl, ding2024consistency, zhang2026sac, lv2026flowbased, wang2025enhanceddaceralgorithmhigh}, or, for Stein variational gradient descent samplers, through all particle updates \cite{messaoud2024s2ac, messaoud2026particles}.
TruDi \cite{le2026trudi} carries this pathwise approach to the on-policy setting by additionally constraining the KL divergence between the joint denoising trajectories of consecutive policies, an upper bound on the intractable marginal KL.
Alternatively, to avoid backpropagating through time, policy gradient (REINFORCE) methods employ a different gradient estimator for the reverse KL that relies on advantage estimation rather than path-wise derivatives \cite{mcallister2026flow, RenLAS0MBDS25, black2024training, liu2026flowgrpo, gong2026proximalpolicyoptimizationpath}.
Particularly in the on-policy domain \cite{zhang2026reinflow, ding2026genpo, zhang2026gorl}, these gradient estimators necessitate trust regions to stabilize updates and prevent premature convergence \cite{peters2010relative, schulman2015trust, otto2021differentiable}, often relying on a  PPO-style clipped loss \cite{schulman2017proximal}.
While AMDP is similarly grounded in the maximum entropy RL formulation, it substantially differs from both paradigms.
Unlike the former, AMDP's policy improvement step completely bypasses backpropagation through the diffusion chain, making it highly memory-efficient and scalable.
Unlike the latter, AMDP operates directly on a learned state-action value function \cite{voelcker2026relative} rather than advantage estimators, and stabilizes policy updates via a constrained trust-region optimization scheme that is rigorously tailored to the underlying adjoint matching objective.

\textbf{Importance weighting.}
An alternative to avoid backpropagating through the diffusion process is the supervised matching objective that, given an action target, updates the parameters based on the error between the predicted action and the target action.
Because ground truth target actions are unavailable in online RL, existing methods leverage previously generated action samples as targets and weight the matching loss with the state-action value \cite{ding2024diffusionbased} or the exponential of the value \cite{ma2025efficient, chen2025onestepflowpolicymirror, dong2025maximum, gao2026flowmatchingpolicyentropy}.
Because these objectives are based on importance weights, they usually suffer from high variance for complex or high-dimensional problems \cite{snyder2008obstacles}.
AMDP's policy improvement objective conceptually differs from these prior objectives in that we match the score of the state-action value function, circumventing an importance-weighted loss.

\textbf{$Q$-score methods.}\quad Methods utilizing the $Q$-score, $\nabla_a Q(a,s)$, typically fall into two sub-categories.
The first leverages Langevin dynamics to sample from the Boltzmann distribution \cite{ki2026directsoftpolicysamplinglangevin}, its interacting-particle counterpart Stein variational gradient descent \cite{haarnoja2017reinforcement, messaoud2024s2ac, messaoud2026particles}, or augment existing objectives \cite{wang2025enhanceddaceralgorithmhigh}.
Langevin dynamics was also used to generate posterior samples from the $Q$-function for a distributional RL approach for training a Gaussian policy \cite{ishfaq2025langevin}.
However, as previously noted, these approaches only converge asymptotically, offering limited theoretical guarantees under practical, finite-time truncation.
The second category employs  ``matching''-based regression objectives to the $Q$-score, which are often motivated to sample high $Q$-function values \cite{PsenkaEA024}, but do not guarantee sampling from the correct action distribution.
Moreover, \cite{jain2025sampling, akhound-sadegh2024iterated} use biased regression targets to guide generation toward high-value regions.
These methods often disregard diffusion time-dependencies or suffer from biased score estimation, lacking strict guarantees of sampling from the true target action distribution.
Closest to our work, \cite{li2026qlearning} also applies an Adjoint Matching objective for offline-to-online RL.
However, their formulation requires costly numerical calculation of the adjoint state at each step, preventing simulation-free optimization.
In contrast, AMDP leverages the reciprocal adjoint loss \cite{havens2025adjoint}, enabling a theoretically grounded, highly efficient, and completely simulation-free policy learning process.

\subsection{Diffusion policies in offline reinforcement learning}
In the Offline and Offline to Online Reinforcement Learning setting, a ground truth and reward-labeled data is generally available to pretrain the diffusion model using supervised learning techniques.
Here, many approaches \citep{JannerDTL22,Chen0Y0023,IDQL_HansenEstruch2023,LuCEP2023,Mao0Z0Z24,DIPOYang2023,DDiffPG_Li2024,RenLAS0MBDS25,zhang2026reinflow,park2025flow,li2026reinforcement,li2026qlearning,AjayDGTJA23,goo2022knowboundariesnecessityexplicit,chen2024diffusion,zhan2026mean,chen2024score,dong2026expo,gao2025behaviorregularized,zhang2026sac} train an imitation learning policy with behavioral cloning to regularize and stabilize the offline RL optimization.
Another line of work extends the dataset with new transitions \citep{park2025flow,li2026reinforcement,zhan2026mean} or use rejection sampling \citep{Chen0Y0023,IDQL_HansenEstruch2023,Kang0DPY23,Mao0Z0Z24,park2025flow,li2026reinforcement,ma2025efficient,dong2026meanflowpolicyoptimization,goo2022knowboundariesnecessityexplicit,zhan2026mean,dong2026expo,gao2026flowmatchingpolicyentropy}.

In the offline and offline-to-online setting, where a dataset of demonstrations is available, many approaches \citep{JannerDTL22,Chen0Y0023,IDQL_HansenEstruch2023,LuCEP2023,Mao0Z0Z24,DIPOYang2023,DDiffPG_Li2024,RenLAS0MBDS25,zhang2026reinflow,park2025flow,li2026reinforcement,li2026qlearning,AjayDGTJA23,goo2022knowboundariesnecessityexplicit,chen2024diffusion,zhan2026mean,chen2024score,dong2026expo,gao2025behaviorregularized,zhang2026sac} utilize generative modeling to train an imitation learning policy with behavioral cloning.
Further methods combine the behaviour policy with a trained Gaussian \citep{ishfaq2025langevin,dong2026expo,li2026qlearning} or a one-step policy \citep{Kang0DPY23,park2025flow,li2026reinforcement,li2026qlearning,chen2024diffusion,ChenLZ24} that optimizes the state-action value function.
A subset of methods \citep{RenLAS0MBDS25,zhang2026sac} use behavior cloning for initialization and then adjust that expressive policy, while other methods \citep{zhang2026reinflow,qiu2026scoreflowcompletedistributionalcontrol} finetune pretrained models purely online, which entirely skips the offline RL phase and therefore skips the critic learning.

\textbf{Exploration control in diffusion-based policy representations.}
Various methods exist for controlling the policy's stochasticity, including principled Maximum Entropy RL \citep{celik2025dime,dong2026meanflowpolicyoptimization,ma2025efficient,ChenLZ24,dong2025maximum,lv2026flacmaximumentropyrl,messaoud2024s2ac,le2026trudi} and Entropy Regularized RL \citep{wang2024diffusion,dong2026expo,jain2025sampling,gao2026flowmatchingpolicyentropy}.
Other works rely on additive noise \citep{wang2024diffusion,zhang2026reinflow,ma2025efficient,zhang2026sac}, an entropy bonus on a latent Gaussian policy \citep{zhang2026gorl}, and other regularization \citep{ding2024diffusionbased} to incentivize exploration.
For maximizing the return, however, the stochasticity is ideally reduced which lead to a rejection and Best-of-N sampling from a batch of proposal actions only during eval \citep{Chen0Y0023,IDQL_HansenEstruch2023,Kang0DPY23,Mao0Z0Z24,FangLZWJ25,dong2026meanflowpolicyoptimization,gao2025behaviorregularized}, or even during training for exploration or exploitation \citep{ding2024diffusionbased,li2026reinforcement,ma2025efficient,goo2022knowboundariesnecessityexplicit,dong2026expo}.

\section{Further algorithmic details}

\subsection{A brief introduction to stochastic optimal control}
\label{appendix: soc intro}

In this section, we provide a brief overview of the stochastic optimal control (SOC) framework as it applies to our formulation which is based on \cite{domingo2024stochastic,berner2022optimal,nusken2021solving}.
SOC considers the problem of controlling a stochastic process $(X_\tau)_{\tau \in [0, 1]}$ governed by the SDE $\dd X_\tau = \sigma(\tau)u(X_\tau, s, \tau) \dd \tau + \sigma(\tau) \dd B_\tau$.
The goal is to find an admissible control $u \in \mathcal{U}$ that minimizes the expected cumulative cost:
\begin{equation}
    \mathcal{J}(u) = \E_{X \sim \Pi^u}\left[ \int_0^1 \frac{1}{2}\|u(X_\tau, s,\tau)\|^2 \dd \tau + g(X_1, s)\right].
\end{equation}
Here, the integral of $\frac{1}{2}\|u\|^2$ represents the kinetic energy or control effort required to steer the system, and $g(x, s)$ acts as a state-conditioned terminal cost evaluated at $\tau=1$.

\paragraph{Equality with reverse KL minimization.}
The SOC objective can be elegantly re-framed as an approximate inference problem \cite{domingoenrich2025adjoint,domingo2024stochastic}.
From Girsanov's theorem (\Cref{theorem: girsanov}), the expected control effort corresponds exactly to the Kullback-Leibler (KL) divergence between the controlled path measure $\Pi^u$ and the reference measure $\Pi^0$.
Consequently, the total cost can be written as:
\begin{equation}
    \mathcal{J}(u) = \KL(\Pi^u(\cdot|s) \mid \Pi^0(\cdot|s)) + \E_{X \sim \Pi^u}\left[ g(X_1, s) \right].
\end{equation}
If we construct a target path measure $\Pi^*(\cdot|s)$ such that its Radon-Nikodym derivative with respect to the reference process satisfies $\frac{\dd \Pi^*}{\dd \Pi^0}(X|s) \propto \exp(-g(X_1, s))$, minimizing the expected cost $\mathcal{J}(u)$ becomes mathematically equivalent to minimizing the reverse KL divergence between the controlled measure and the target measure:
\begin{equation}
    \min_{u} \mathcal{J}(u) \iff \min_{u} \KL(\Pi^u(\cdot|s) \mid \Pi^*(\cdot|s)).
\end{equation}
This directly follows from the fact that
\begin{align}
    \KL(\Pi^u(\cdot|s) \mid \Pi^*(\cdot|s)) & = \E_{X\sim\Pi^u}\left[\log \frac{\dd \Pi^u(X|s)}{\dd \Pi^*(X|s)}\right] 
    \\ & = \E_{X\sim\Pi^u}\left[\log \frac{\dd \Pi^u(X|s)}{\dd \Pi^0(X|s)} + \log \frac{\dd \Pi^0(X|s)}{\dd \Pi^*(X|s)}\right],
\end{align}
when evaluating the first term on the right using Girsanov's theorem (\Cref{theorem: girsanov}) and using that $\frac{\dd \Pi^*}{\dd \Pi^0}(X|s) \propto \exp(-g(X_1, s))$ by construction.

\paragraph{Optimal change of measure and reciprocal classes.}
Because the target measure $\Pi^*$ is defined by re-weighting the uncontrolled process $\Pi^0$ purely based on the terminal boundary condition at $\tau=1$ \cite{dai1991stochastic,zhang2021path}, the internal transition dynamics remain unchanged.
Therefore, the optimal change of measure preserves the bridge distributions.
This means the optimal path space measure $\Pi^{u^*}(\cdot|s)$ forms a \textit{reciprocal class} \cite{leonard2013survey} of the uncontrolled process $\Pi^0(\cdot|s)$.
Formally, it factors exactly into the optimal terminal marginal and the uncontrolled backward transitions:
\begin{equation}
    \Pi^{u^*}(X|s) = \Pi^{u^*}_1(X_1|s) \Pi^0_{|1}(X|X_1, s).
\end{equation}
This property is foundational to our approach, as it theoretically justifies the simulation-free reciprocal adjoint matching loss \cite{havens2025adjoint} utilized in AMDP, allowing us to generate intermediate states directly from $\Pi^0_{\tau|1}$ rather than simulating the full SDE.

\subsection{Schr\"odinger bridges and their connection to stochastic optimal control}
\label{appendix: sb connection}

The Schr\"odinger Bridge (SB) problem \cite{schrodinger1931umkehrung,leonard2013survey,shi2023diffusion,liu2025adjoint} seeks to find the most probable path distribution that transports a given source distribution $\mu_0$ to a target distribution over the time interval $\tau \in [0,1]$.
Formally, this is cast as a distributionally constrained optimization problem as described in \eqref{eq: sb problem} and restated here for completeness:
\begin{equation}
    \min_{u} \ \E_{X\sim\Pi^u}\left[\int_0^1 \frac{1}{2}\|u(X_\tau, s,\tau)\|^2\dd \tau\right]\quad \mathrm{s.t.} \quad X_0 \sim \mu_0, \ X_1 \sim \pi^*(\cdot|s).
\end{equation}

\paragraph{The Schr\"odinger system.}
The optimal control $u^*$ that solves the SB problem is uniquely characterized by the Schr\"odinger system \cite{leonard2013survey} The optimal drift takes the form $u^*(x, s, \tau) = \sigma(\tau) \nabla_x \log \varphi_\tau(x,s)$ \cite{liu2025adjoint}, where the forward and backward SB potentials $\varphi_\tau$ and $\hat{\varphi}_\tau$ satisfy the integral equations:
\begin{equation}
    \varphi_\tau(x,s) = \int \Pi^0_{1|\tau}(y \mid x,s) \varphi_1(y,s) \dd y, \quad \hat{\varphi}_\tau(x,s) = \int \Pi^0_{\tau|0}(x \mid y,s) \hat{\varphi}_0(y,s) \dd y,
\end{equation}
which are coupled by the boundary conditions $\varphi_0(x,s)\hat{\varphi}_0(x,s) = \mu(x)$ and $\varphi_1(x,s)\hat{\varphi}_1(x,s) = \Pi^*(x|s)$.
Because of these tightly coupled boundaries, solving the Schr\"odinger system directly is generally challenging.

\paragraph{SOC formulation and memorylessness.}
A key theoretical observation is that the SB problem can be equivalently formulated as an unconstrained Stochastic Optimal Control (SOC) problem \cite{chen2016relation,liu2025adjoint}.
Specifically, the kinetic-optimal drift $u^*$ also minimizes the SOC objective with a specifically chosen terminal cost:
\begin{equation}
    \min_{u} \E_{X \sim \Pi^u}\left[ \int_0^1 \frac{1}{2}\|u(X_\tau, s,\tau)\|^2 \dd \tau + \log \frac{\hat{\varphi}_1(X_1,s)}{\Pi^*_1(X_1|s)}\right].
\end{equation}
To decouple this objective from the intractable backward potential $\hat{\varphi}_1$, we can enforce a \textit{memoryless condition} on the reference process.
The memoryless condition fundamentally assumes statistical independence between $X_0$ and $X_1$ in the reference process, i.e., $\Pi^0_{0,1}(X_0,X_1)=\Pi^0_{0}(X_0)\Pi^0_{1}(X_1)$ where $\Pi^0_{0,1}$ is the joint distribution between $X_0$ and $X_1$.
In our setting, this is achieved by choosing a deterministic initial state $X_0 = 0$ (i.e., $\mu_0(x) = \delta_0(x)$).
Under this condition, the initial value function bias \cite{domingoenrich2025adjoint} is eliminated and the backward potential simplifies exactly to the terminal marginal of the uncontrolled process \cite{liu2025adjoint}, i.e., $\hat{\varphi}_1(X_1,s) = \Pi^0_1(X_1,s)$.

Substituting this simplification, the terminal cost elegantly becomes $\log \frac{\Pi^0_1(X_1,s)}{\Pi^*_1(X_1|s)}$.
When defining our target distribution $\Pi^*_1(X_1|s)$ based on the squashed Boltzmann distribution of the $Q$-function, the SOC objective perfectly reduces to the one employed in \eqref{eq: path kl}:
\begin{equation}
    \min_u \E_{X \sim \Pi^u}\left[ \int_0^1 \frac{1}{2}\|u(X_\tau, s,\tau)\|^2 \dd \tau - \log \Pi^*_1(X_1|s) + \log \Pi^0_1(X_1,s) \right].
\end{equation}
Thus, the memoryless reduction cleanly sidesteps the need to solve the coupled Schr\"odinger system.

\subsection{Reference process and conditional marginals}
\label{appendix: reference process}
In our setting, the reference (uncontrolled) process is defined by setting the drift to zero, i.e., $u=0$, and the initial condition to a deterministic $X_0 = 0$.
The resulting SDE is a scaled Brownian motion $\dd X_\tau = \sigma(\tau) \dd B_\tau$.
Because the increments of Brownian motion are independent and Gaussian, the marginal distribution at any time $\tau \in (0, 1]$ is given by:
\begin{equation}
    \Pi^0_\tau(x) = \mathcal{N}(x \mid 0, \bar{\sigma}^2_\tau), \quad \text{where} \quad \bar{\sigma}^2_\tau = \int_0^\tau \sigma^2(t) \dd t.
\end{equation}
Specifically, the terminal marginal distribution is $\Pi^0_1(x) = \mathcal{N}(x \mid 0, \bar{\sigma}^2_1)$.
To perform simulation-free training via the reciprocal adjoint matching loss, we require the conditional distribution $\Pi^0_{\tau|1}(X_\tau | X_1)$.
This can be derived using the property that the joint distribution of $(X_\tau, X_1)$ is Gaussian:
\begin{equation}
    \begin{pmatrix} X_\tau \\ X_1 \end{pmatrix} \sim \mathcal{N} \left( \begin{pmatrix} 0 \\ 0 \end{pmatrix}, \begin{pmatrix} \bar{\sigma}^2_\tau & \bar{\sigma}^2_\tau \\ \bar{\sigma}^2_\tau & \bar{\sigma}^2_1 \end{pmatrix} \right),
\end{equation}
where the covariance $\mathrm{Cov}(X_\tau, X_1) = \E[X_\tau(X_\tau + \int_\tau^1 \sigma(t) \dd B_t)] = \bar{\sigma}^2_\tau$ due to the independence of increments.
Applying the standard formula for Gaussian conditioning, we obtain:
\begin{equation}
\label{eq: conditional reference}
    \Pi^0_{\tau|1}(X_\tau | X_1) = \mathcal{N} \left( X_\tau \mid \frac{\bar{\sigma}^2_\tau}{\bar{\sigma}^2_1} X_1, \bar{\sigma}^2_\tau - \frac{(\bar{\sigma}^2_\tau)^2}{\bar{\sigma}^2_1} \right).
\end{equation}
This allows us to sample intermediate states $X_\tau$ directly from the terminal action $X_1$ without simulating the SDE, which is the foundation for the efficiency of the AMDP objective.

\subsection{Extension to Ornstein-Uhlenbeck processes}
\label{appendix: OU process}
\paragraph{Denoising diffusion sampler.}
Consider the controlled Ornstein-Uhlenbeck process
\begin{align}
\dd X_\tau = -\frac{1}{2}\sigma(\tau)^2X_\tau\dd \tau + \eta\ \sigma(\tau)u(X_\tau, s,\tau) \dd \tau + \eta\ \sigma(\tau) \dd B_\tau, \quad X_0 \sim \mathcal{N}(0,\eta^2 I),
\label{eq: controlled sde dds}
\end{align}
with path space measure $\Pi^u$.
The uncontrolled process with $\Pi^0$ is given by
\begin{align}
\dd X_\tau = -\frac{1}{2}\sigma(\tau)^2X_\tau\dd \tau + \eta\ \sigma(\tau) \dd B_\tau, \quad X_0 \sim \mathcal{N}(0,\eta^2 I),
\label{eq: uncontrolled sde dds}
\end{align}
which satisfies $\Pi^0_t=\mathcal{N}(0,\eta^2 I)$ for all $t\in[0,1]$ and therefore $\Pi^0_1(X_1|s) = \mathcal{N}(0,\eta^2 I)$.
The reciprocal adjoint state is available in closed form,
\begin{equation}
    a(X,s,\tau) = \exp\left({\int_\tau^1-\frac{1}{2}\sigma(v)^2\dd v}\right)\nabla_{x}\left(\log \frac{\Pi^0_1(X_1|s)}{\big|\det J_f(X_1)\big|}-\tfrac{1}{\alpha}Q^{\pi}(f(X_1),s)\right)
\end{equation}
resulting in
\begin{equation}
\label{eq: ram loss action space dds}
    \mathcal{L}^{s,f}_{\mathrm{AMDP}}(u) = \int_0^1\E\left[ \frac{1}{2}\|u(X_\tau,s,\tau)+\eta \ \sigma(\tau)a(X,s,\tau)\|^2\right]\dd \tau.
\end{equation}

\paragraph{Extension to Gaussian priors via memoryless noise schedules.}
\label{subsec: memoryless noise schedules}
The reciprocal adjoint matching loss employed in \textsc{AMDP} requires that the reference path measure $\Pi^0$ is \emph{memoryless}, i.e., the endpoints satisfy
\begin{equation}
\Pi^0_{0,1}(X_0, X_1) = \Pi^0_0(X_0)\, \Pi^0_1(X_1).
\label{eq: memoryless property}
\end{equation}
Throughout the main paper we ensured this property by choosing a deterministic (Dirac) prior $\mu_0(x) = \delta_0(x)$, so that $X_0 = 0$ is independent of $X_1$ trivially.
As pointed out in the limitations, this choice is incompatible with the standard practice of pretraining diffusion policies from a Gaussian prior $\mu_0 = \mathcal{N}(0,I)$.
We now explain in detail how the framework can be extended to a Gaussian prior by adopting the \emph{memoryless noise schedule} construction of Domingo-Enrich~et~al.~\cite{domingoenrich2025adjoint}, which preserves \eqref{eq: memoryless property} while allowing $X_0 \sim \mathcal{N}(0,I)$.

\paragraph{Construction of the memoryless reference SDE.}
Consider a more general reference process of the form
\begin{align}
\dd X_\tau = -\tfrac{1}{2}\beta(\tau)\, X_\tau\, \dd \tau + \sigma(\tau)\, \dd B_\tau, \qquad X_0 \sim \mathcal{N}(0, I),
\label{eq: ou reference}
\end{align}
which is a time-inhomogeneous Ornstein--Uhlenbeck (OU) process with drift coefficient $\beta(\tau)\ge 0$ and diffusion coefficient $\sigma(\tau)>0$.
Without loss of generality, we consider a unit-variance prior, since rescaling $X_\tau$ and $\sigma(\tau)$ by $\eta^{-1}$ absorbs a prior $\mathcal{N}(0,\eta^2 I)$.
The marginal of \eqref{eq: ou reference} at time $\tau$ is Gaussian, $\Pi^0_\tau = \mathcal{N}(0, \Sigma_\tau I)$ with
\begin{align}
\Sigma_\tau \;=\; e^{-(B(0)-B(\tau))} + \int_0^\tau e^{-(B(s)-B(\tau))}\, \sigma^2(s)\, \dd s,\qquad B(\tau) := \int_\tau^1 \beta(s)\, \dd s.
\label{eq: ou marginal variance}
\end{align}
The crucial covariance between the endpoints is
\begin{align}
\operatorname{Cov}(X_0, X_1) \;=\; e^{-\tfrac{1}{2}B(0)}\, I,
\label{eq: ou covariance}
\end{align}
so the reference is memoryless if and only if $\operatorname{Cov}(X_0, X_1) = 0$, i.e., $B(0) = \int_0^1 \beta(\tau)\, \dd\tau = \infty$.
As in \cite{domingoenrich2025adjoint}, where the memoryless diffusion coefficient is infinite at the noise end, we require a non-integrable singularity at $\tau = 0$ only, i.e., $B(\tau)<\infty$ for $\tau>0$.
In addition, $\sigma^2$ must be non-integrable at $\tau=0$, since otherwise $\Sigma_\tau \le \int_0^\tau \sigma^2(v)\, \dd v \to 0$, while $\Sigma_1 = \int_0^1 e^{-B(v)}\sigma^2(v)\, \dd v$ must remain finite.
Choosing $\sigma^2=\beta$ yields $\Sigma_\tau=1$ and, for example,
\begin{align}
\beta(\tau) \;=\; \sigma^2(\tau) \;=\; \frac{c}{\tau}, \qquad c\ge 1,
\label{eq: memoryless schedule}
\end{align}
with $e^{-\tfrac{1}{2}B(\tau)} = \tau^{c/2}$, which is the memoryless denoising diffusion probabilistic model (DDPM) process of \cite{domingoenrich2025adjoint} for a standard normal data distribution.
Under any such \emph{memoryless noise schedule}, the joint $(X_0, X_1)$ factorises,
\begin{align}
\Pi^0_{0,1}(X_0, X_1) = \mathcal{N}(X_0|0, I)\, \mathcal{N}(X_1|0, \Sigma_1\, I),
\end{align}
even though $X_0 \sim \mathcal{N}(0,I)$ is genuinely random.
Memorylessness is therefore decoupled from the choice of prior and absorbed entirely into the noise schedule.

\paragraph{Reciprocal adjoint matching with a Gaussian prior.}
Replacing the reference SDE used in the main paper by \eqref{eq: ou reference} with a memoryless schedule, the controlled process becomes
\begin{align}
\dd X_\tau = \Big[-\tfrac{1}{2}\beta(\tau)\, X_\tau + \sigma(\tau)\, u(X_\tau, s, \tau)\Big]\, \dd \tau + \sigma(\tau)\, \dd B_\tau, \qquad X_0 \sim \mathcal{N}(0, I),
\label{eq: controlled sde memoryless}
\end{align}
and all derivations in the main paper as well as the appendix carry over with the following two modifications.

\emph{(i) Reciprocal class and bridge sampling.}
The optimal target measure still factorises as
\begin{align}
\Pi^{u^*}(X|s) = \Pi^{u^*}_1(X_1|s)\, \Pi^0_{|1}(X|X_1, s),
\end{align}
because the memoryless reference still belongs to a reciprocal class.
The conditional bridge $\Pi^0_{\tau|1}$ for the OU process \eqref{eq: ou reference} is again Gaussian and can be derived in closed form by Gaussian conditioning on the joint $(X_\tau, X_1)$.
Letting $C_\tau := \operatorname{Cov}(X_\tau, X_1)$, one obtains
\begin{align}
\Pi^0_{\tau|1}(X_\tau|X_1) = \mathcal{N}\!\left(X_\tau \,\Big|\, \frac{C_\tau}{\Sigma_1} X_1,\ \Sigma_\tau - \frac{C_\tau^2}{\Sigma_1}\right).
\end{align}
For memoryless schedules, $C_\tau = e^{-\tfrac{1}{2}B(\tau)}\Sigma_\tau > 0$ for $\tau > 0$, so intermediate states remain coupled to $X_1$ and can still be sampled simulation-free given a terminal sample $X_1$.
For \eqref{eq: memoryless schedule}, this gives $\Pi^0_{\tau|1}(X_\tau|X_1) = \mathcal{N}(X_\tau | \tau^{c/2} X_1, 1-\tau^c)$.

\emph{(ii) Adjoint state and loss.}
The reciprocal adjoint state of \eqref{eq: controlled sde memoryless} satisfies the backward ODE
\begin{align}
\dd a(X, s, \tau) = \tfrac{1}{2}\beta(\tau)\, a(X, s, \tau)\, \dd \tau, \qquad a(X, s, 1) = \nabla_x \!\left( \log\!\frac{\Pi^0_1(X_1|s)}{|\det J_f(X_1)|} - \tfrac{1}{\alpha} Q^\pi(f(X_1), s)\right),
\end{align}
which integrates to $a(X, s, \tau) = e^{-\tfrac{1}{2}B(\tau)}\, a(X, s, 1)$.
The reciprocal adjoint matching loss thus becomes
\begin{align}
\mathcal{L}^{s,f}_{\mathrm{AMDP}}(u) = \int_0^1 \mathbb{E}\!\left[\, \tfrac{1}{2}\, \big\| u(X_\tau, s, \tau) + \sigma(\tau)\, a(X, s, \tau) \big\|^2 \,\right]\!\dd \tau,
\label{eq: ram loss memoryless}
\end{align}
which has \emph{the same form} as Eq.~\eqref{eq: ram loss action space}, with the only changes being (a) sampling $X_0 \sim \mathcal{N}(0,I)$ instead of $X_0=0$, and (b) using the OU bridge \eqref{eq: ou marginal variance} for simulation-free training.
The scaled error function \eqref{eq: erf} with $k=(2\Sigma_1)^{-1/2}$ again cancels the log-ratio in $a(X,s,1)$.

\paragraph{Theoretical guarantees.}
The contraction of the soft Bellman operator (Appendix~\ref{appendix: policy iteration}) and Proposition~\ref{prop: fixed point preservation} are algebraic in the operator and the quadratic losses, and only require $\mathcal{L}^s_{\mathrm{ENT}}(u)$ and the losses to be finite.
Proposition~\ref{prop: entropy lower bound} follows from the Girsanov identity, the data processing inequality and the change of variables.
It holds for any prior shared by $\Pi^u$ and $\Pi^0$ and does not require memorylessness.
Memorylessness is required to identify the terminal marginal of $\Pi^{u^*}$ with the target and for the tightness in Remark~\ref{remark: tightness} at the fixed point.
That $u^*$ is the fixed point of the policy improvement step follows as in \cite[Prop.~7]{domingoenrich2025adjoint} and \cite[Thm.~C.3]{havens2025adjoint}, whose arguments apply for each $\tau>0$ with the linear base drift, since $B(\tau)<\infty$.

\paragraph{Practical considerations.}
The regression target $\sigma(\tau)a(X,s,\tau)$ stays bounded under \eqref{eq: memoryless schedule} for bounded $a(X,s,1)$.
In practice, one starts from $X_\epsilon\sim\mathcal{N}(0,I)$ on $[\epsilon,1]$ for a small $\epsilon > 0$ or offsets the singularity as in \cite{domingoenrich2025adjoint}.
This leaves $\operatorname{Cov}(X_\epsilon,X_1)=\epsilon^{c/2}I$, so that the schedule is only approximately memoryless.
The benefit of this extension is that \textsc{AMDP} can now be initialised from any pretrained Gaussian-prior diffusion or flow-matching policy, fully bridging offline-to-online reinforcement learning while retaining the simulation-free training of the reciprocal adjoint matching objective.

\paragraph{Ablation experiment.}
To ablate the influence of the SDE choice (Fig.~\ref{fig:ablations_reverse_kl_dmc}~and~\ref{fig:ablations_reverse_kl_humanoid}), we implement and evaluate the OU-process for AMDP as described above but do not switch to a differently shaped noise schedule with a singularity.
In particular, we increase the diffusion strength of the same geometric noise schedule, see Appendix~\ref{appendix: implementation details}, such that the endpoint covariance of \eqref{eq: ou covariance} is approximately zero.

\subsection{Policy iteration for adjoint matching diffusion policy}
\label{appendix: policy iteration}

Recall the soft Bellman operator for a fixed policy $\pi^u$

\begin{align}\label{eq soft bellman appdx}
(\mathcal{T}^{\pi^u} Q)(a,s) \coloneqq r(s,a) + \gamma \mathbb{E}_{s' \sim p, a' \sim \pi^u} \big[Q(a',s') + \alpha  \mathcal{L}^{s'}_{\mathrm{ENT}}(u)\big]
\end{align}
and the entropy lower bound on the marginal entropy
\begin{equation}
     \mathcal{H}(\pi^u(a|s)) \geq \mathcal{L}^s_{\mathrm{ENT}}(u) \coloneqq -\E_{X\sim \Pi^u}\left[\int_0^1 \frac{1}{2}\|u(X_\tau, s,\tau)\|^2 \dd \tau + \log \frac{\Pi^0_1(X_1|s)}{\big|\det J_f(X_1)\big|}\right].
\end{equation}

Iteratively applying the ``soft'' Bellman operator converges to the fixed point
\begin{equation}
    Q^{\pi^u}(a, s) = r(s, a) + \mathbb{E}_{\pi} \left[ \sum_{t=1}^\infty \gamma^t (r(s_t, a_t) + \alpha \mathcal{L}^{s_t}_{\mathrm{ENT}}(u) \right],
\end{equation}
using the entropy lower bound under the current policy $\pi^u$, which coincides with the soft $Q$-function of $\pi^u$ if the bound is tight (Remark~\ref{remark: tightness}).

For any $s$ and $a$ and bounded $Q$-functions $Q_1$ and $Q_2$ their absolute difference is
\begin{align}
    |(\mathcal{T}^{\pi^u} Q_1)(a,s) - (\mathcal{T}^{\pi^u} Q_2)(a,s) | = \gamma |\mathbb{E}_{s'\sim p, a'\sim\pi^u}\left[Q_1(a',s')-Q_2(a',s')\right]|, 
\end{align}
which can be easily seen by inserting the Operator in Eq. \ref{eq soft bellman appdx} into the left-hand side.
Further, we can write using the triangle inequality
\begin{align}  \label{bound_}
    \gamma |\mathbb{E}_{s'\sim p, a'\sim\pi^u}\left[Q_1(a',s')-Q_2(a',s')\right]| \leq \gamma \mathbb{E}_{s'\sim p, a'\sim\pi^u}\left[|Q_1(a',s')-Q_2(a',s')|\right].
\end{align}
Next, we can find the supremum over the absolute difference
\begin{align}
    \sup_{(a',s')} |Q_1(a',s')-Q_2(a',s')| = ||Q_1 - Q_2||_\infty,
\end{align}
which is the worst-case scenario with the highest difference between the two bounded $Q$-values.
As this is the highest possible value, the random variables on the right-hand side of Eq. \ref{bound_} are also smaller than or equal to the $\| \cdot \|_\infty$ norm.
Using this relation, we can write
\begin{align}
    \gamma \mathbb{E}_{s'\sim p, a'\sim\pi^u}\left[|Q_1(a',s')-Q_2(a',s')|\right] \leq \gamma ||Q_1-Q_2||_\infty, 
\end{align}
where the expectation reduces to the constant itself, since $||Q_1-Q_2||_\infty$ does not depend on $s',a'$.
Hence, the $||\cdot ||_\infty$ norm on the Bellman operator difference is bounded by the infinity norm of the Q values as
\begin{align}
    ||(\mathcal{T}^{\pi^u} Q_1)(a,s) - (\mathcal{T}^{\pi^u} Q_2)(a,s) ||_\infty \leq \gamma ||Q_1-Q_2||_\infty,
\end{align}
which converges to a unique fixed point despite using the entropy lower bound $\mathcal{L}^s_{\mathrm{ENT}}(u)$ as introduced in Section \ref{sec policy iteration}.
This proof for soft policy updates can also be found, for example, in \cite{haarnoja2018soft, lv2026flacmaximumentropyrl}.

\subsection{Temperature auto-tuning}
\label{appendix: temp autotuning}

The temperature parameter $\alpha$ in the target distribution \eqref{eq: energy policy} is critical as it controls the exploration-exploitation trade-off during training.
Utilizing a fixed temperature can be problematic because the scale of $Q$-function values often changes significantly as the policy improves.
To address this, we follow the approach of prior work \citep{haarnoja2019softactorcriticalgorithmsapplications} by formulating a constrained optimization problem to automatically determine $\alpha$ based on a desired target entropy $\bar{\mathcal{H}}$.
Since the marginal entropy is intractable for diffusion policies, we utilize the entropy lower bound $\mathcal{L}^s_{\mathrm{ENT}}(u)$ defined in \eqref{eq: ent lb} as a drop-in surrogate,
\begin{equation}
    \max_{u} \mathbb{E}_{\pi^u} \left[\sum_{t=0}^\infty \gamma^t r(s_t, a_t) \right] \quad \text{s.t.} \quad \mathcal{L}^{s_t}_{\mathrm{ENT}}(u) \geq \bar{\mathcal{H}} \quad \forall t.
\end{equation}
This leads to the following dual objective for learning the temperature $\alpha$:
\begin{equation}
    J(\alpha) = \mathbb{E}_{s} \left[ \alpha \left( \mathcal{L}^s_{\mathrm{ENT}}(u) - \bar{\mathcal{H}} \right) \right].
\end{equation}
In practice, we treat $\alpha$ as a learnable Lagrangian multiplier, parameterized as $\alpha = \exp(\theta_\alpha)$ following \cite{celik2025dime,voelcker2026relative}, which ensures $\alpha > 0$ in \eqref{eq: amdp loss}, and optimize $\theta_\alpha$ via gradient descent.
By minimizing $J(\alpha)$, the temperature is automatically increased if the entropy surrogate falls below the target $\bar{\mathcal{H}}$ and decreased otherwise, ensuring stable exploration throughout the learning process.

\subsection{Further details on action squashing}
\label{appendix: action squashing}

Squashing the actions into the valid action bounds is commonly done using the bijective elementwise $\tanh$ transformation in the literature \cite{haarnoja2018soft}.
As a result, the transformation yields a modified density function $\pi^*(a|s)$ that can be described in terms of the base distribution $ \tilde \pi^*(X_1|s)$ using the change of variables formula, which is already detailed in Section \ref{AM for maxent RL} as
\begin{align}\label{eq CoV}
    \tilde \pi^*(X_1|s) = \pi^*(a|s) \big|\det J_f(X_1)\big|, \quad \text{where} \quad J_f(X_1) = \frac{\dd f(X_1)}{\dd X_1}.
\end{align}
Here, the determinant of the Jacobian $\det J_f(X_1)$ describes the scaling of changes in the variable given by the action transforming function $a=f(X_1)$, with $X_1 \in\mathbb{R}^d$.
Intuitively, when the slope of $f$ is low, then a broad range of inputs is squashed into a narrow range of outputs such that the squashed space density $\pi^*(a|s)$ must be higher.

Since the target distribution is only defined explicitly in the squashed space, the shape of the squashing function determines how this distribution will be \emph{unsquashed} into $\tilde \pi^*(X_1|s)$.
A typical squashing function is monotone and has vanishing gradient for inputs towards $\pm \infty$, but the difference between, e.g., $\tanh$ and $\mathrm{erf}$ lies in how the Jacobian determinant decays.
The relevant trade-off for our use-case is between stretching the extremities of the action space strongly or stretching ``medium'' actions more.
The usual $\tanh$ transform decays the Jacobian rather slowly, thus causing less extreme stretching of the limits, but already distorts for fairly moderate actions.
As a result, the $\tanh$-unsquashed target density has comparatively broad support with heavier tails.
While the difference in function and derivative shape, shown in Fig. \ref{table: error function plots}, may seem arbitrary or subjective, we highlight three reasons to favor error function squashing, two of which also apply to non-expressive Gaussian policies:
\begin{enumerate}
    \item The error function perfectly squashes a zero-mean Gaussian with $\sigma^2=\tfrac{1}{2}$ into the uniform distribution over $[-1, 1]$, as shown in Fig.~\ref{table: error function plots}.
    This proves beneficial for initialization and maximum-entropy considerations.
    \item Under mild assumptions on $\pi^*(a|s)$ (see Prop.~\ref{prop:score_and_tail_bounds}), the unsquashed $\tilde \pi^*$ has sub-gaussian tails compared to the heavier sub-exponential tails of $\tanh$-squashing.
    \item Rescaling inputs to emulate 1. both simplifies our entropy lower-bound \eqref{eq: squashed entropy bound} into Eq.~\eqref{eq: erf entropy lower bound} and ensures a monotonic coupling of kinetic energy and marginal entropy (see Remark \ref{remark:kinetic energy and entropy}).
\end{enumerate}

\input{tables/appendix_action_squashing_analysis}

\begin{remark}[Squashing of Gaussian is Uniform]
The determinant of the error function Jacobian is inverse to the density of the Gaussian $\tilde{p}(X) = \mathcal{N}(X \mid 0, \tfrac{1}{2})$ up to the constant factor $\tfrac{1}{2}$.
Thus the density $p(a)$ of $a = \mathrm{erf}(X), X \sim \tilde{p}$ is uniform $p(a) = \mathcal{U}(-1, 1)$, as
\begin{equation}
    p(a) = \tilde{p}(X) \cdot | \det J_{\mathrm{erf}}(X)|^{-1} = \frac {1}{\sqrt {2\pi \sigma^{2}}} \exp\left(-\frac {x^2}{2\sigma^2} \right) \cdot \left| \frac{2}{\sqrt{\pi}}\exp(-x^2) \right|^{-1} = \frac{1}{2}.
\end{equation}
\end{remark}

Given that the squashed target density is defined by the Q-function, that is modeled by a neural network, with bounded rewards, we can infer the following about the unsquashed distribution.

\begin{proposition}[Score and Tail Behavior of Unsquashed Policies]\label{prop:score_and_tail_bounds}
Assume the target squashed density $\pi^*(a|s)$ is continuously differentiable on $(-1, 1)$ with a bounded derivative $|\nabla_a \pi^*(a|s)| \le L$ (Lipschitz continuity), and is bounded away from zero near the action bounds, meaning $\pi^*(a|s) \ge c > 0$ as $|a| \to 1$.
Then, the unsquashed density score function, $\nabla_{X_1} \log \tilde{\pi}^*(X_1|s)$, is asymptotically entirely dominated by the squashing function's Jacobian, yielding:
\begin{enumerate}
    \item \emph{A linear score} for error function squashing ($f = \mathrm{erf}$): \\
    $\lim_{X_1 \to \pm \infty} \nabla_{X_1} \log \tilde{\pi}^*(X_1|s) = -2 X_1$.
    \item \emph{A constant score} for hyperbolic tangent squashing ($f = \tanh$): \\
    $\lim_{X_1 \to \pm \infty} \nabla_{X_1} \log \tilde{\pi}^*(X_1|s) = \mp 2$.
\end{enumerate}
Consequently, $\mathrm{erf}$-squashing yields sub-Gaussian tails, whereas $\tanh$-squashing yields strictly sub-exponential tails.
\end{proposition}

\begin{proof}
Applying the change of variables formula to the score of $\tilde{\pi}$ for $f \in \{\mathrm{erf}, \tanh\}$, where $| \det J_f(X_1)| = f'(X_1)$, and the log-derivative trick yields
\begin{align}\label{eq: score_expansion}
    \nabla_{X_1} \log \tilde{\pi}^*(X_1|s) 
    &= \nabla_{X_1} \log \pi^*(f(X_1)|s) + \nabla_{X_1} \log f'(X_1) \nonumber \\
    &= \frac{\nabla_a \pi^*(a|s)}{\pi^*(a|s)} \cdot f'(X_1) + \frac{f''(X_1)}{f'(X_1)}.
\end{align}
Under our assumptions, the term $\frac{\nabla_a \pi^*(a|s)}{\pi^*(a|s)}$ is strictly bounded in magnitude by $L/c$.
Because both squashing functions have vanishing gradients at the extremities, $\lim_{X_1 \to \pm \infty} f'(X_1) = 0$, the first term of Eq. \ref{eq: score_expansion} vanishes in the limit:
\begin{equation}
    \lim_{X_1 \to \pm \infty} \left| \frac{\nabla_a \pi^*(a|s)}{\pi^*(a|s)} \cdot f'(X_1) \right| \le \lim_{X_1 \to \pm \infty} \frac{L}{c} f'(X_1) = 0.
\end{equation}
The asymptotic behavior of the score is thus completely determined by the second term, $\frac{f''(X_1)}{f'(X_1)}$, which is the score of the Jacobian determinant.

For error function squashing, $f(X_1) = \mathrm{erf}(X_1)$, the derivative is $f'(X_1) = \frac{2}{\sqrt{\pi}}\exp(-X_1^2)$ and thus
\begin{equation}
    \frac{f''(X_1)}{f'(X_1)} = \frac{\dd}{\dd X_1} \left( \log \frac{2}{\sqrt{\pi}} - X_1^2 \right) = -2X_1.
\end{equation}
Since the score is asymptotically linear, the density decays as $\exp(-X_1^2)$, proving sub-Gaussian tails.

For hyperbolic tangent squashing, $f(X_1) = \tanh(X_1)$, the derivative is $f'(X_1) = 1 - \tanh^2(X_1)$ and thus
\begin{equation}
    \frac{f''(X_1)}{f'(X_1)} = \frac{-2\tanh(X_1)(1-\tanh^2(X_1))}{1-\tanh^2(X_1)} = -2\tanh(X_1).
\end{equation}
Taking the limit as $X_1 \to \pm \infty$, the score approaches $\mp 2$ as $\tanh(X_1) \to \pm 1$.
An asymptotically constant score implies the density decays as $\exp(-2|X_1|)$, proving strictly sub-exponential tails.
\end{proof}

\begin{remark}[Linear restoring force]
Note that the usage of the error function for squashing creates a linear spring-like restoring force $-2X_1$ (see Proposition \ref{prop:score_and_tail_bounds}) for the score function in the squashed space.
This improves robustness to outlier critic gradients and is also relevant for Gaussian policy MaxEnt-RL algorithms, such as SAC \citep{haarnoja2018soft}, which also follow the score implicitly in their objective through backpropagation through the change-of-variables and since $\log \pi^*(a | s) = Q^\pi(a, s)/\alpha$.
\end{remark}

As also detailed in Section \ref{AM for maxent RL}, the error function provides an opportunity to choose the scaling factor $k$ (see \eqref{eq: erf}) such that the modified adjoint matching loss (Eq. \eqref{eq: ram loss action space}) can be simplified to (Eq. \eqref{eq: amdp loss}), by canceling the gradient of the log ratio $\log \tfrac{\Pi^0_1(X_1|s)}{\big|\det J_f(X_1)\big|}$.
Here, we show this specific design choice for the one-dimensional case for simplicity, where the element-wise extension is straightforward.
To this end, recall that for a constant scaling factor $k$ the scaled error function was defined as
\begin{align}
    \text{erf}(kx) = \frac{2}{\sqrt{\pi}}\int_0^{kx} \exp{\left(-v^2\right)}dv,
\end{align}
which gives us the Jacobian
\begin{align}
    J_{\mathrm{erf}}(x) =  \frac{\dd ~\mathrm{erf} (kx)}{\dd x}\cdot k = \frac{2k}{\sqrt{\pi}}\exp{\left(-k^2x^2\right)}. 
\end{align}
The scaling parameter $k$ provides a degree of freedom that allows us to shape the Jacobian such that the gradient of the log ratio in Eq. \eqref{eq: ram loss action space} cancels.
More concretely, the specific choice $k=\frac{1}{\sqrt{2 \bar{\sigma}_1^2 }}$ with $ \bar{\sigma}_1^2 = \int_0^1\sigma^2(\tau)\dd \tau $ (see Appendix \ref{appendix: reference process}) yields
\begin{align}
     J_{\mathrm{erf}}(X_1) = \frac{2}{\sqrt{2\pi \bar{\sigma}_1^2 }} \exp{\left(-\frac{X_1^2}{2 \bar{\sigma}_1^2 }\right)}, 
\end{align}
which also equals the absolute determinant of the Jacobian in the one-dimensional case.
We can now write the log ratio in Eq. \eqref{eq: ram loss action space}  as
\begin{align}
    \log \frac{\Pi^0_1(X_1|s)}{\big|\det J_f(X_1)\big|} =&-\frac{1}{2}\log\left(2\pi \bar{\sigma}_1^2 \right)  - \frac{1}{2 \bar{\sigma}_1^2 } X_1^2 - \log 2 +\frac{1}{2}\log\left(2\pi \bar{\sigma}_1^2 \right) +\frac{x^2_1}{2 \bar{\sigma}_1^2 } \\
    =& -\log2, \label{eq:erf constant log ratio}
\end{align}
which is a constant such that its gradient is zero.

\begin{remark}[Error function squashing simplifies Entropy Lower Bound]\label{remark: erf entropy lower bound}
Applying Eq. \eqref{eq:erf constant log ratio} to the marginal, squashed entropy lower bound \eqref{eq: squashed entropy bound} reveals
\begin{equation}
\label{eq: erf entropy lower bound}
    \mathcal{H}(\pi^u(\cdot|s)) \ge - \E_{X \sim \Pi^u}\left[ \int_0^1 \frac{1}{2}\|u(X_\tau, s,\tau)\|^2 \dd \tau \right] + \log 2 = \mathcal{L}^s_{\mathrm{ENT}}(u),
\end{equation}
thus directly connecting the kinetic energy of the diffusion to the entropy loss compared against the uniform marginal, $\mathcal{H}(\mathcal{U}(-1, 1)) = \log 2$, that is achieved for zero control at initialization.
\end{remark}

\subsection{Implementation details}
\label{appendix: implementation details}
\textbf{Loss scaling.}\quad
We utilize a geometric noise schedule which covers a broad range of noise levels, i.e., $\sigma(\tau)$, which directly affects the predicted neural network output scale, as the Adjoint Matching targets scale proportionally with $\sigma(\tau)$.
Thus we scale the neural network output by $\sigma(\tau)$ during inference and loss computation and scale the mean squared error (MSE) loss with $\frac{1}{\sigma(\tau)}$ to weight the different time steps, which does not affect the optimal solution \citep{havens2025adjoint}.
An observant reader might notice the $\tfrac{1}{\alpha}$ in the Adjoint Matching targets and worry about numerical and optimization instability at lower temperatures when the critic itself is smooth.
Our empirical experience shows that scaling the loss with $\alpha$, a direct loss scaling of all times and samples that clearly does not affect the optimal solution, slows down training far too much as $\alpha$ decays.

\pagebreak
\subsection{Detailed algorithmic description}\label{appdx detailed algorithmic description}
We provide a full algorithm description for the on-policy case, that is heavily inspired by REPPO \citep{voelcker2026relative} below.
Note that our implementation is not fully optimized but rather simple.
As an example, we always sample new actions for the next state to determine the temporal-difference (TD-$\lambda$) returns, although it would be possible to reuse the next transitions action given the on-policy setting and only sample fresh actions for the truncation case.
\begin{algorithm}[H] 
    \caption{On-policy \texttt{AMDP}: Adjoint Matching Diffusion Policy}
    \label{alg: onpolicy amdp full}
    \SetAlgoLined
    \DontPrintSemicolon
    \SetKwInOut{Input}{Input}
    \SetKwInOut{Output}{Output}

    \Input{Initialized parameters for policy $\theta$ such that $u_\theta \approx 0$, critic $\phi$, temperature $\alpha$ and trust region Lagrangian $\lambda$}

    \For{each iteration}{
        \tcp{Collect Parallel Policy Rollouts into Buffer $\mathcal{B}$}
        $\mathcal{B} = \emptyset$\;
        \For{step $k=1$ \KwTo $K$} {
            \tcp{Sample action for each environment $i$}
            Solve \eqref{eq: controlled sde} for $X_{1,t}$ using Euler-Maruyama in state $s_t^i$, integrate $\|u(X_\tau, s_t^i)\|^2$ along the path and determine the inner term $h_i$ of \eqref{eq: ent lb}\;
            Apply error function squashing $a_t^i \leftarrow f_{\mathrm{erf}}(X_1^i)$ with \eqref{eq: erf}\;
            Collect transition $(s_t,a_t,r_t,s_{t+1})_i$ in each environment $i$\;
            \tcp{Compute values for TD-$\lambda$}
            Solve \eqref{eq: controlled sde} for $X_1'$ in state $s_{t+1}^i$, apply squashing for ${a_t^i}'$, and evaluate next state value $Q_{t+1} = Q_{\phi}({a_t^i}', s_{t+1}^i)$\;
            Add to buffer $\mathcal{B} \leftarrow \mathcal{B} \cup \{(s_t,a_t,r_t,s_{t+1}, X_{1,t}, Q_{t+1}, h)_i \}_i$\;
        }
        Compute TD-$\lambda$ returns $V_t$ using $r_t, Q_{t+1}$\;
        \For{each critic epoch} {
            \For{each minibatch in $\mathcal{B}$} {
            Optimize critic $Q_\phi(a_t,s_t)$ with TD-$\lambda$ returns $V_t$ and entropy lower bound $h_i$ using HL-Gauss \citep{farebrother2024stop} loss and auxiliary loss like REPPO \citep{voelcker2026relative}.
            }
        }
        Compute $Q$-scores $\nabla_{x}Q_\phi(f_{\mathrm{erf}}(X_{1,t}),s_t)$ once and add to buffer $\mathcal{B}$. \;
        \For{each actor epoch} {
            \For{each minibatch in $\mathcal{B}$} {
            Note that $X_{1,t}$ from Buffer are distributed as $\Pi^{\bar u}_1(X_1 | s_t)$ \;
            Sample random time $\tau \sim \mathcal{U}[0,1]$, determine noise level $\sigma(\tau)$ and sample $X_{\tau,t}\sim\Pi^0_{\tau|1}$\;
            Compute targets $\sigma(\tau) \nabla_{x} \tfrac{1}{\alpha} Q_\phi(f_{\mathrm{erf}}(X_{1,t}),s_t)$ using current temperature $\alpha$ and precomputed $Q$-score\;
            Evaluate old model control $u_{\mathrm{old}}(X_{\tau,t}, s_t, \tau)$\;
            Update actor $u_\theta(X_\tau, s, \tau)$ and $\lambda$ by dual descent on \eqref{eq: Lagrangian loss} in the same optimizer step \;
            Auto-tune temperature $\alpha$ (see Appendix~\ref{appendix: temp autotuning}).\;
            }
        }
       Set $u_{old} \leftarrow u_\theta$. \;
        
    }
\end{algorithm}

\section{Experimental setup}
\label{appendix: experimental setup}

Our code for reproducing all experiments is available at \url{https://github.com/serathilg/amdp}.
Table~\ref{tab:algo_params} lists the default hyperparameters for all used algorithms that form the basis for the on-policy experiments.
Following the reference REPPO codebase\footnote{\url{https://github.com/cvoelcker/reppo}}, we make the following changes per benchmark: The discount factor is adjusted to $\gamma=0.97$ for the Mujoco Playgrounds Humanoid tasks and for ManiSkill a heuristic, based on the varying episode length $L$, $\gamma = 1 - \frac{10}{L}$ is used for \texttt{AMDP}, REPPO and DIME.
The distributional critic value bin range are also adjusted to the (discounted) return range to prevent saturation.

Due to the CPU-based simulation of HumanoidBench, we use only 128 parallel environments with 128 environment steps in each iteration to achieve a buffer size of 16384.
Thus we switch to 8 minibatches and 8 epochs with a batch size of 2048.

For FPO we use the reference implementation of \citet{mcallister2026flow}.
On the MuJoCo Playground humanoid tasks, REPPO, DIME, SPO, PPO and DPPO use 16 instead of 64 minibatches, and PPO uses the rsl\_rl configuration (256 ELU units, learning rate $10^{-3}$ with an adaptive KL schedule, max.\ grad.\ norm 1.0, entropy coefficient 0.01).
On ManiSkill, REPPO and DIME use 4 epochs over 128 minibatches, and SPO and PPO use 4 instead of 8 epochs and Tanh activations.

For \texttt{AMDP-OU} (Appendix~\ref{appendix: OU process}), we use $\eta=1$, extend the geometric noise schedule from $\sigma \in [0.01, 2]$ to $[0.01, 5]$, and do not scale the network output by $\sigma(\tau)$, i.e., the network directly parametrizes $u$, while keeping the $\tfrac{1}{\sigma(\tau)}$ loss weighting (Appendix~\ref{appendix: implementation details}).
The error function squashing uses $k=1/\sqrt{2}$ in \eqref{eq: erf} to match the unit-variance terminal marginal $\Pi^0_1=\mathcal{N}(0,I)$.
\begin{table*}[!htbp]
    \centering
    \caption{Hyperparameters for all algorithms.
    LR denotes the learning rate and GAE denotes generalized advantage estimation.}
    \label{tab:algo_params}
    \small
    \setlength{\tabcolsep}{1pt}
    \renewcommand{\arraystretch}{1.10}
    \begin{tabularx}{\textwidth}{@{}%
        >{\raggedright\arraybackslash}p{0.16\textwidth}
        YYYYYYY@{}}
        \toprule
        \textbf{Parameter}          & \makecell{\textbf{AMDP}\\\textbf{(Ours)}} & \textbf{REPPO}                & \textbf{DIME}         & \textbf{SPO}                       & \textbf{PPO}                    & \textbf{FPO}                     & \textbf{DPPO}                   \\
        \midrule

        \multicolumn{8}{@{}l}{\textbf{\textit{Actor Network}}} \\
        Hidden dim.                 & 512                                       & 512                           & 512                   & 128                                & 512                             & 32                               & 128                             \\
        Hidden layers               & 3                                         & 3                             & 3                     & 3                                  & 3                               & 4                                & 3                               \\
        Activation                  & GeLU                                      & GeLU                          & GeLU                  & Tanh                               & ELU/Tanh                        & SiLU                             & Mish                            \\
        Flow/diff. steps            & 16                                        & N/A                           & 16                    & N/A                                & N/A                             & 10                               & 5                               \\

        \midrule
        \multicolumn{8}{@{}l}{\textbf{\textit{Critic Network}}} \\
        Critic type                 & \makecell{Dist.\\(HL-Gauss)}              & \makecell{Dist.\\(HL-Gauss)}  & N/A                   & MSE                                & MSE                             & MSE                              & MSE                             \\
        Hidden dim.                 & 512                                       & 512                           & 512                   & 128                                & 512                             & 256                              & 256                             \\
        Hidden layers               & 3                                         & 3                             & 3                     & 3                                  & 3                               & 5                                & 3                               \\
        Ensemble/bins               & 151 bins                                  & 151 bins                      & 100                   & 1                                  & 1                               & 1                                & 1                               \\
        Activation                  & GeLU                                      & GeLU                          & GeLU                  & Tanh                               & ELU                             & SiLU                             & ELU                             \\

        \midrule
        \multicolumn{8}{@{}l}{\textbf{\textit{Optimization}}} \\
        Optimizer                   & Adam                                      & Adam                          & Adam                  & Adam                               & Adam                            & Adam                             & Adam                            \\
        Actor LR                    & $3{\times}10^{-4}$                        & $3{\times}10^{-4}$            & $3{\times}10^{-4}$    & $3{\times}10^{-4}$                 & $3{\times}10^{-4}$              & $3{\times}10^{-4}$               & $3{\times}10^{-4}$              \\
        Critic LR                   & $3{\times}10^{-4}$                        & $3{\times}10^{-4}$            & $3{\times}10^{-4}$    & $3{\times}10^{-4}$                 & $3{\times}10^{-4}$              & $3{\times}10^{-4}$               & $3{\times}10^{-4}$              \\
        Max grad. norm              & 0.5                                       & 0.5                           & N/A                   & 0.5                                & 0.5                             & N/A                              & 0.5                             \\
        Constraint                  & \makecell{$\epsilon=0.1$\\KL}             & \makecell{$\epsilon=0.1$\\KL} & N/A                   & \makecell{$\epsilon=0.2$\\Penalty} & \makecell{$\epsilon=0.2$\\Clip} & \makecell{$\epsilon=0.05$\\Clip} & \makecell{$\epsilon=0.2$\\Clip} \\

        \midrule
        \multicolumn{8}{@{}l}{\textbf{\textit{Additional Settings}}} \\
        Envs $\times$ steps         & $1024{\times}128$                         & $1024{\times}128$             & $1024{\times}128$     & $1024{\times}128$                  & $1024{\times}128$               & $1024{\times}1024$               & $1024{\times}128$               \\
        Buffer size                 & 131072                                    & 131072                        & 131072                & 131072                             & 131072                          & 1048576                          & 131072                          \\
        Epochs $\times$ minibatches & $4{\times}128$                            & $8{\times}64$                 & $8{\times}64$         & $8{\times}64$                      & $8{\times}64$                   & $16{\times}16$                   & $8{\times}64$                   \\
        (Mini)batch size            & 1024                                      & 2048                          & 2048                  & 2048                               & 2048                            & 65536                            & 2048                            \\
        Discount                    & 0.99                                      & 0.99                          & 0.99                  & 0.99                               & 0.99                            & 0.995                            & 0.99                            \\
        GAE/TD $\lambda$            & 0.95                                      & 0.95                          & 0.95                  & 0.95                               & 0.95                            & 0.95                             & 0.95                            \\
        Target entropy              & $-1.0\dim(\mathcal{A})$                   & $-0.5\dim(\mathcal{A})$       & $-4\dim(\mathcal{A})$ & N/A                                & N/A                             & N/A                              & N/A                             \\
        \bottomrule
    \end{tabularx}
\end{table*}

\begin{table*}[!htbp]
    \centering
    \caption{Hyperparameters of \texttt{AMDP} and diffusion-based baselines for the DMC Dog and Humanoid benchmark suites in the off-policy setting.}
    \label{tab:off_policy_diffusion_params}
    \small
    \setlength{\tabcolsep}{2.5pt}
    \renewcommand{\arraystretch}{1.12}
    \begin{tabularx}{\textwidth}{@{}%
        >{\raggedright\arraybackslash}p{0.20\textwidth}
        YYYYY@{}}
        \toprule
        \textbf{Parameter}    & \makecell[l]{\textbf{AMDP}\\\textbf{(Ours)}} & \textbf{DIME}         & \textbf{QSM}       & \textbf{Diff-QL}   & \makecell[l]{\textbf{Consistency}\\\textbf{-AC}} \\
        \midrule

        \multicolumn{6}{@{}l}{\textbf{\textit{Actor Network}}} \\
        Hidden dim.           & 512                                          & 256                   & 256                & 256                & 256                                              \\
        Hidden layers         & 3                                            & 3                     & 3                  & 4                  & 4                                                \\
        Flow/diff. steps      & 32                                           & 16                    & 15                 & 5                  & N/A                                              \\
        Prior distr.          & Dirac                                        & $\mathcal{N}(0,2.5)$  & $\mathcal{N}(0,1)$ & N/A                & N/A                                              \\

        \midrule
        \multicolumn{6}{@{}l}{\textbf{\textit{Critic Network}}} \\
        Hidden dim.           & 2048                                         & 2048                  & 2048               & 256                & 256                                              \\
        Hidden layers         & 2                                            & 2                     & 2                  & 2                  & 3                                                \\
        Bins/quantiles        & 100                                          & 100                   & N/A                & N/A                & N/A                                              \\

        \midrule
        \multicolumn{6}{@{}l}{\textbf{\textit{Optimization}}} \\
        Optimizer             & Adam                                         & Adam                  & Adam               & Adam               & Adam                                             \\
        Actor LR              & $3{\times}10^{-4}$                           & $3{\times}10^{-4}$    & $3{\times}10^{-4}$ & $1{\times}10^{-5}$ & $1{\times}10^{-5}$                               \\
        Critic LR             & $3{\times}10^{-4}$                           & $3{\times}10^{-4}$    & $3{\times}10^{-4}$ & $3{\times}10^{-4}$ & $3{\times}10^{-4}$                               \\
        Temperature LR        & $1{\times}10^{-3}$                           & $1{\times}10^{-3}$    & N/A                & N/A                & N/A                                              \\

        \midrule
        \multicolumn{6}{@{}l}{\textbf{\textit{Additional Settings}}} \\
        Update-to-data ratio  & 2                                            & 2                     & 1                  & 1                  & 1                                                \\
        Discount              & 0.99                                         & 0.99                  & 0.99               & 0.99               & 0.99                                             \\
        Batch size            & 256                                          & 256                   & 256                & 256                & 256                                              \\
        Buffer size           & $10^6$                                       & $10^6$                & $10^6$             & $10^5$             & $10^5$                                           \\
        Target entropy        & $-\dim(\mathcal{A})$                         & $-4\dim(\mathcal{A})$ & N/A                & N/A                & N/A                                              \\
        Exploration steps     & 5                                            & 5000                  & $10^4$             & $10^4$             & $10^4$                                           \\
        Score-Q align. factor & N/A                                          & N/A                   & 50                 & N/A                & N/A                                              \\
        \bottomrule
    \end{tabularx}
\end{table*}

\FloatBarrier

\subsection{Benchmarks}

\paragraph{DeepMind Control Suite.}
DeepMind Control Suite (DMC)~\citep{tunyasuvunakool2020dm_control} is a collection of continuous-control environment built on the MuJoCo physics engine.
It provides standardized tasks covering locomotion, balance, and motor with continuous state and action spaces.
Representative DMC domains are show in Figure \ref{fig:dmc_benchmark}.

For our off-policy evaluation, we use seven high-dimensional locomotion tasks: four tasks using the dog embodiment and three tasks using the humanoid embodiment.
These experiments use the original CPU-based \textit{dm\_control} implementation and are distinct from our on-policy MuJoCo Playground experiments, which use GPU-accelerated MJX reimplementations of DMC tasks.
The complete task set is summarized in Table~\ref{tab:off_policy_dmc_tasks}.

\begin{figure}[!htbp]
    \centering
    \includegraphics[width=0.7\linewidth]{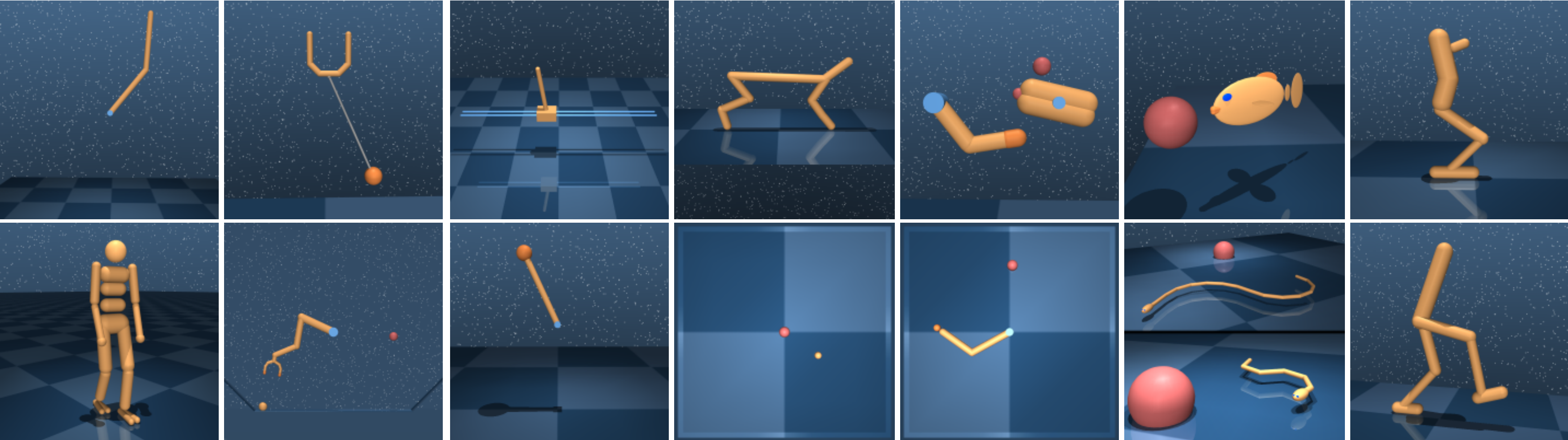}
    \caption{\textbf{DeepMind Control Suite benchmark domains.}
    The suite covers a diverse set of continuous-control problems, ranging from classic control tasks such as Acrobot, Cart-pole, and Pendulum to locomotion and manipulation-style domains such as Cheetah, Walker, Humanoid, Manipulator, and Fish.}
    \label{fig:dmc_benchmark}
\end{figure}

\begin{table*}[!htbp]
    \centering
    \caption{DM Control evaluation tasks used in our off-policy experiments.
    Observation dimensions correspond to flattened state observations.}
    \label{tab:off_policy_dmc_tasks}
    \small
    \setlength{\tabcolsep}{3pt}
    \renewcommand{\arraystretch}{1.05}
    \begin{tabularx}{\textwidth}{
        @{}p{0.13\textwidth}p{0.24\textwidth}cc
        >{\raggedright\arraybackslash}X@{}
    }
        \toprule
        \textbf{Category}                  & \textbf{Task}           & \textbf{Obs.} & \textbf{Act.} & \textbf{Description}                                   \\
        \midrule

        \multirow{4}{*}{\texttt{Dog}}      & \texttt{dog-stand}      & 223           & 38            & Maintain an upright standing posture.                  \\

                                           & \texttt{dog-walk}       & 223           & 38            & Move forward at the walking target speed.              \\

                                           & \texttt{dog-trot}       & 223           & 38            & Move forward at the trotting target speed.             \\

                                           & \texttt{dog-run}        & 223           & 38            & Move forward at the running target speed.              \\
        \midrule

        \multirow{3}{*}{\texttt{Humanoid}} & \texttt{humanoid-stand} & 67            & 21            & Maintain an upright posture while minimizing movement. \\

                                           & \texttt{humanoid-walk}  & 67            & 21            & Walk at the target horizontal speed.                   \\

                                           & \texttt{humanoid-run}   & 67            & 21            & Run at the target horizontal speed.                    \\

        \bottomrule
    \end{tabularx}
\end{table*}

\paragraph{MuJoCo Playground: DMC and Humanoid.}
We evaluate on a set of tasks from MuJoCo Playground, a GPU-accelerated robot-learning benchmark suite built on  MuJoCo XLA (MJX).
Our evaluation includes 20 environments from the MuJoCo Playground reimplementation of the DeepMind Control (DMC) Suite, spanning classic control, planar locomotion, target reaching, and simple manipulation-style control tasks.
We sometimes abbreviate \emph{MuJoCo Playground} as \emph{MJX}, to distinguish \emph{MJX DMC} from the off-policy (CPU-based) \emph{MuJoCo DMC} tasks.

We additionally evaluate four joystick-based humanoid locomotion tasks using the Unitree G1 and Booster T1 embodiments on flat and rough terrain, which we denote as \emph{MuJoCo Playground (or MJX) Humanoid}.

\begin{figure}[!htbp]
    \centering
    \includegraphics[width=0.7\linewidth]{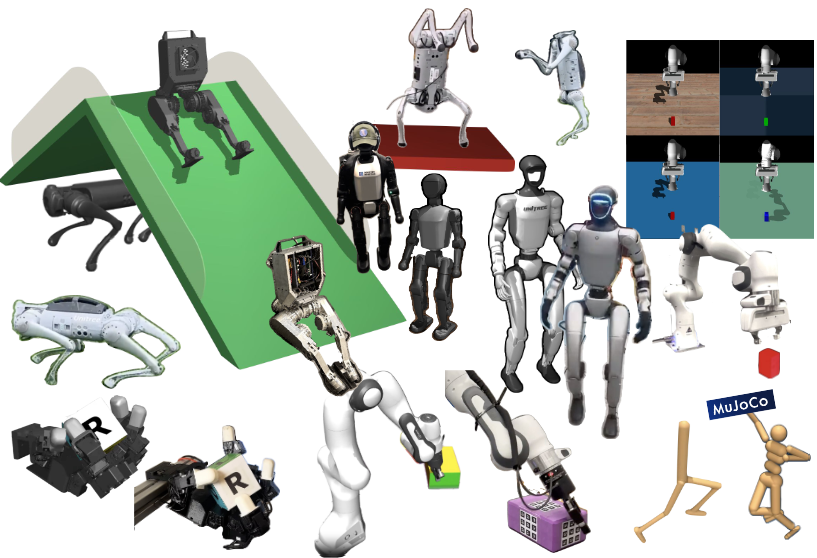}
    \caption{\textbf{Representative MuJoCo Playground environments.}
    The benchmark provides diverse continuous-control tasks across classic control, locomotion, and manipulation-style settings.}
    \label{fig:mujoco_playground}
\end{figure}

A summary of all used MuJoCo Playground tasks can be found in Table~\ref{tab:mujoco_playground_tasks}, each representing a distinct skill.

\input{tables/appendix_mujoco_playground}

\FloatBarrier

\paragraph{ManiSkill3.}
\label{subsec:maniskill3}

ManiSkill3~\citep{tao2025maniskill} is a large-scale GPU-parallelized simulation benchmark designed for scalable training of embodied agents.
It offers diverse object-centric contact-rich manipulation tasks with diverse object geometries such as grasping, assembling, and tool use, with support for both imitation and reinforcement learning.

\begin{figure}[!htbp]
    \centering
    \includegraphics[width=0.9\linewidth]{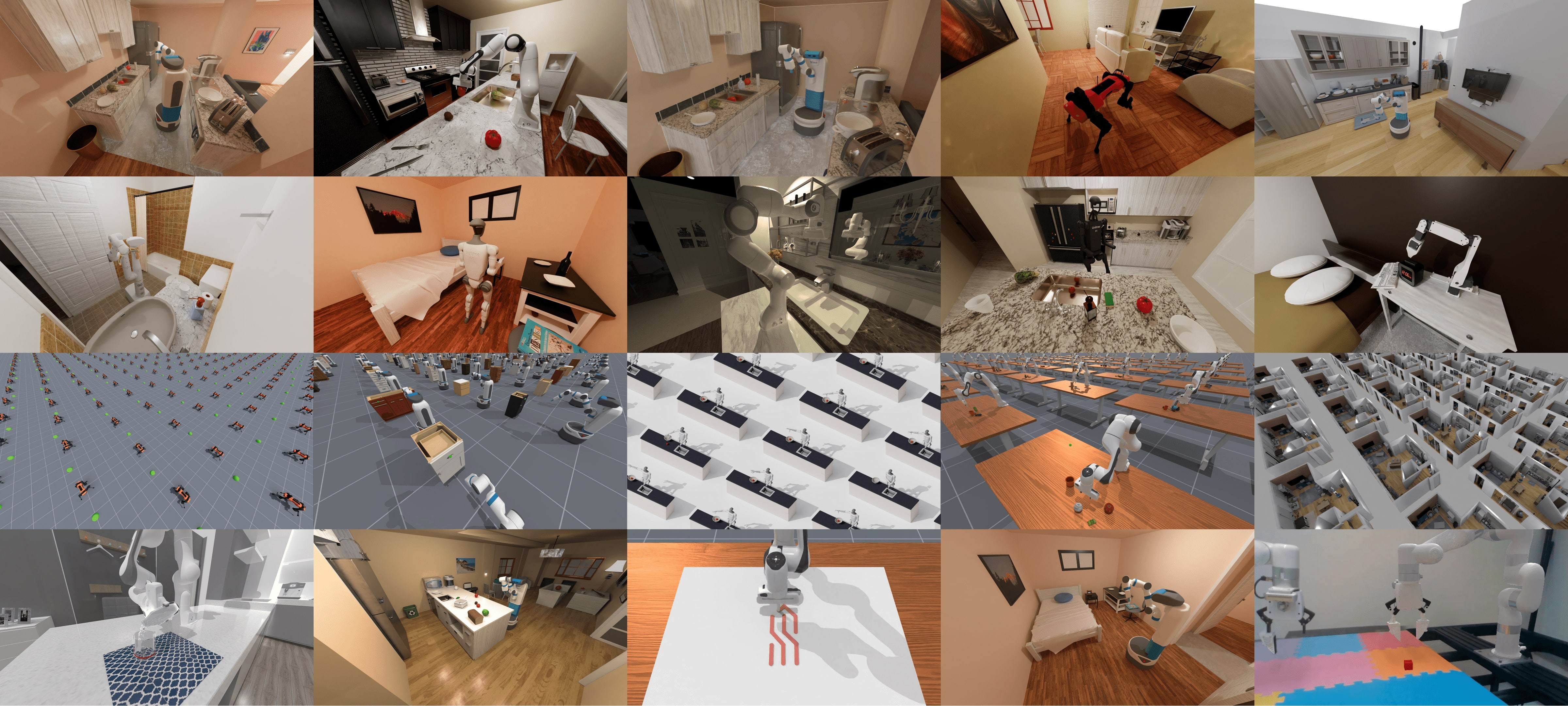}
    \caption{\textbf{Overview of ManiSkill3 Simulation Environments.}
    Example object-centric manipulation tasks illustrating the diversity of interactions supported by ManiSkill3.}
    \label{fig:maniskill_benchmark}
\end{figure}

A summary of all tasks included in our reported ManiSkill3 benchmark can be found in Table~\ref{tab:maniskill3_tasks}, each representing a distinct skill.

\input{tables/appendix_maniskill}

\newpage

\paragraph{HumanoidBench.}
HumanoidBench~\citep{sferrazza2024humanoidbench} evaluates whole-body humanoid control across locomotion, agility, and manipulation tasks.
Compared with standard locomotion benchmarks, these tasks require the policy to combine balance, long-horizon coordination, and task-directed behavior.
We use this benchmark to test whether the expressivity of diffusion policies improves performance in complex high-dimensional control.

A summary of all 29 used HumanoidBench tasks can be found in Table~\ref{tab:humanoidbench_tasks}, each representing a distinct skill.

\begin{figure}[!htbp]
    \centering
    \includegraphics[width=0.6\linewidth]{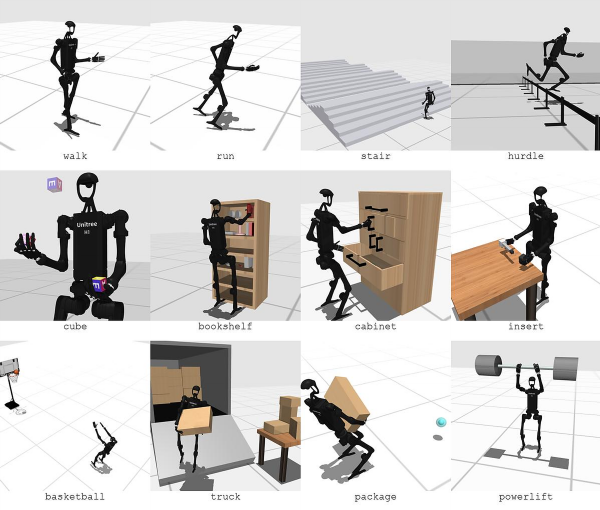}
    \caption{\textbf{Overview of HumanoidBench Simulation Environments.}
    The selected tasks cover locomotion, static manipulation, and dynamic manipulation categories.}
    \label{fig:humanoid_benchmark}
\end{figure}

\input{tables/appendix_humanoidbench}

\clearpage
\subsection{Performance benchmarking details}
\label{appendix: performance benchmarking}

We select five tasks from the fast GPU- and JAX-based MuJoCo simulation with varying action dimensionality and simulation complexity, remove all evaluation related computation from our JAX \cite{bradbury2021jax} implementation and measure the time per iteration in the on-policy setting.
We measure the combined time for inference diffusion sampling and environment stepping, to accurately portray the speed under just-in-time (JIT) optimization, and the combined time for all network update related computations, i.e., target value computation, critic updates and actor updates, per iteration.
We accumulate these times for chunks of 20 iterations, skipping the first chunk for warmup and considering the next 20 chunks.
The measured time per chunk is then divided by 20 to obtain the time spent per iteration.
This yields 80 measurements across the four seeds, except for the slow \emph{reverse KL 128 steps} configuration due to timeouts, representing 1600 iterations.

All benchmarks were performed on one NVIDIA A100 SXM4 40GB GPU and 12 AMD Zen 2 CPU cores, using the same type of server compute nodes, equipped with dual AMD EPYC Rome 7402 CPUs and four NVIDIA A100 GPUs in total.
All runs were performed with exclusive allocation of the node and four seeds running in parallel.

\subsection{VRAM requirements}
\label{appendix: vram}
In Appendix \ref{appendix: larger models}, larger actor models are evaluated and we demonstrate the (low) VRAM requirements of our method, even for larger actor network architectures, below.
Utilizing \texttt{XLA\_PYTHON\_CLIENT\_PREALLOCATE=false} and \texttt{nvidia-smi} for VRAM measurement, we see 2722 MiB, 2788 MiB and 4710 MiB respectively for our standard actor of roughly 290k parameters, 4-block 256 wide ($\approx$ 2.7M) and 10-block 1024 wide ($\approx$106M) residual actors (see Appendix \ref{appendix: larger models}).
Furthermore, we scale to a 1B-parameter actor (24-block 2048 wide) and see 24.1\,GiB without any reduced-precision or offloading.
Note that this closely tracks the storage needed for the old and current policy parameters, the Adam optimizer state, and the activations in single precision, and remains constant as the number of diffusion steps increases.

\begin{table*}[t]
    \centering
    \caption{Runtime performance comparison across environments and configurations.
    All values are reported in milliseconds as mean $\pm$ one standard deviation over four runs.
    See Appendix~\ref{appendix: performance benchmarking} for the experiment setup.}
    \label{tab:runtime-results}
    \small
    \setlength{\tabcolsep}{2.5pt}
    \renewcommand{\arraystretch}{1.10}
    \begin{tabularx}{\textwidth}{@{}%
        >{\raggedright\arraybackslash}p{0.16\textwidth}
        c
        YYYYY@{}}
        \toprule
        \textbf{Method} & \textbf{Steps} & \makecell{\textbf{Cartpole}\\\textbf{Balance}} & \makecell{\textbf{Hopper}\\\textbf{Hop}} & \makecell{\textbf{Walker}\\\textbf{Run}} & \makecell{\textbf{G1Joystick}\\\textbf{Rough}} & \makecell{\textbf{T1Joystick}\\\textbf{Flat}} \\
        \midrule

        \multicolumn{7}{@{}l}{\textbf{\textit{Environment Rollout Time}}} \\
        REPPO           & --             & $84 \pm 0.3$                                   & $1309 \pm 14$                            & $1657 \pm 10$                            & $9605 \pm 21$                                  & $4107 \pm 37$                                 \\
        \texttt{AMDP}   & 16             & $404 \pm 0.4$                                  & $1694 \pm 7.4$                           & $2013 \pm 4$                             & $10031 \pm 28$                                 & $4525 \pm 41$                                 \\
        rev.\ KL        & 16             & $404 \pm 0.7$                                  & $1671 \pm 18$                            & $2027 \pm 17$                            & $10041 \pm 31$                                 & $4545 \pm 19$                                 \\
        \texttt{AMDP}   & 128            & $2738 \pm 5$                                   & $4375 \pm 13$                            & $4743 \pm 22$                            & $13097 \pm 35$                                 & $7501 \pm 43$                                 \\
        rev.\ KL        & 128            & $2751 \pm 8$                                   & $4418 \pm 29$                            & $4734 \pm 34$                            & $13129 \pm 42$                                 & $7505 \pm 51$                                 \\
        \midrule

        \multicolumn{7}{@{}l}{\textbf{\textit{Network Update Time}}} \\
        REPPO           & --             & $948 \pm 1.5$                                  & $943 \pm 1.6$                            & $958 \pm 1.7$                            & $1014 \pm 1$                                   & $999 \pm 3$                                   \\
        \texttt{AMDP}   & 16             & $994 \pm 7.7$                                  & $1030 \pm 11$                            & $1032 \pm 8$                             & $1113 \pm 14$                                  & $1123 \pm 11$                                 \\
        \texttt{AMDP}   & 128            & $991 \pm 7.5$                                  & $1023 \pm 11$                            & $1035 \pm 8$                             & $1117 \pm 13$                                  & $1122 \pm 10$                                 \\
        rev.\ KL        & 16             & $9944 \pm 11$                                  & $10443 \pm 17$                           & $10520 \pm 19$                           & $11421 \pm 29$                                 & $11258 \pm 17$                                \\
        rev.\ KL        & 128            & $71659 \pm 134$                                & $76239 \pm 176$                          & $76725 \pm 197$                          & $82101 \pm 182$                                & $80750 \pm 173$                               \\
        \bottomrule
    \end{tabularx}
\end{table*}

\subsection{Compute resources}
\label{appendix: compute resources}
Experiments were performed on a range of hardware, from local workstations with consumer NVIDIA GPUs (RTX 3060, RTX 2080 Ti, RTX 3080, RTX 3090, RTX 5090) to compute clusters, equipped with NVIDIA A100 40GB.
Due to the computational efficiency of our method and the GPU-acceleration of the majority of on-policy environments, each experiment requires little resources and does not require high-performance computing (HPC) resources.
A single MuJoCo Playground environment takes roughly 20 minutes, a MuJoCo Humanoid or ManiSkill environment 1-3 hours and only off-policy and HumanoidBench evaluations require multiple hours of training.

\newpage
\section{Further experimental results}
\label{appendix: further experimental results}
This appendix provides additional experimental results.
Appendix \ref{appendix: larger models} investigates scaling the diffusion actor architecture.
Moreover, we provide full training curves of all 63 On-Policy and 7 Off-Policy environments in Appendix~\ref{appendix: individual mjx},\ref{appendix: individual maniskill},\ref{appendix: individual humanoidbench} and \ref{appendix: individual offpolicy}.

\subsection{Large actor model scaling proof-of-concept}
\label{appendix: larger models}
To demonstrate the efficiency of our method with large networks, we scale the diffusion control network while maintaining the critic model size for fair comparison, although we believe that our methods' cheaply realizable expressivity can create meaningful performance improvements when adjusting the critic's expressiveness \cite{nauman2024bigger, lee2025hyperspherical, palenicek2026xqc}.
In particular, we experiment with a residual network architecture inspired by DiT~\citep{Peebles2022DiT}, using a residual architecture with Gated GeLU~\citep{shazeer2020glu} blocks with three layers of 8/3 expansion factor, and Adaptive LayerNorm (AdaLN)~\citep{perez2017filmvisualreasoninggeneral,Peebles2022DiT} for conditioning using context embeddings that combine observation and embedded time fourier features.
The context is created with a two-layer GeLU multilayer perceptron (MLP) (using the same hidden dimension as the residual stream) which is then projected with another linear layer for each AdaLN normalization at the start of the residual block on the stream (like in Pre-LN) before the normalized and shifted hidden features are passed to the GeGLU block.

In Fig~\ref{fig: ablation action model and scale} we compare our default 290k parameter model with a 2.7M parameter model using 4 blocks and 256 hidden dimension on both MuJoCo Playground suites.
On the Humanoid suite, we additionally evaluate an extremely large 106M-parameter actor,  using ten blocks and a hidden dimension of 1024.

As expected, just scaling the actor while keeping the critic fixed does not yield large downstream performance improvements but the ability to train these model sizes, keeping all other hyperparameters identical, at all shows a qualitative difference of our scalable algorithm.

\begin{figure*}[!htbp]
    \centering
    \small
    \legendModelArchitectures

    \vspace{0.35em}

    \begin{subfigure}{0.47\textwidth}
        \centering
        \includegraphics[width=\linewidth]{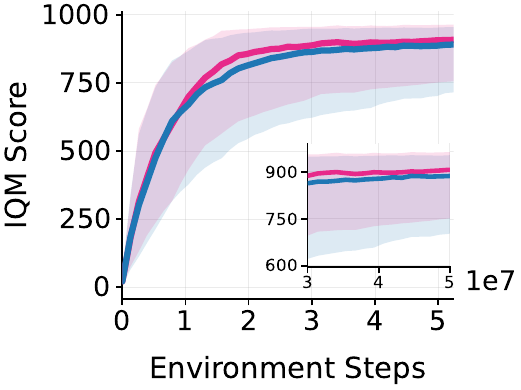}
        \caption{MJX DMC}
        \label{fig:ablations_model_scaling_dmc}
    \end{subfigure}
    \hfill
    \begin{subfigure}{0.47\textwidth}
        \centering
        \includegraphics[width=\linewidth]{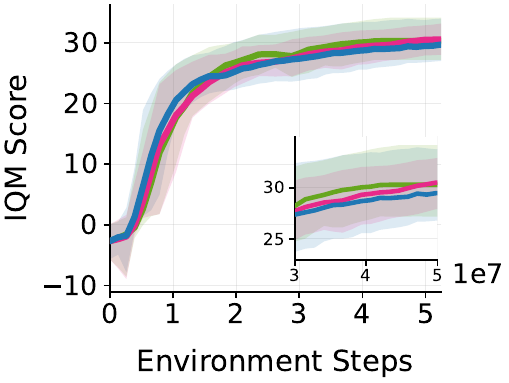}
        \caption{MJX Humanoid}
        \label{fig:ablations_model_scaling_humanoid}
    \end{subfigure}

    \caption{\textbf{Ablation results on scaling the actor model.}
    Aggregated results for the MuJoCo Playground DMC and Humanoid suites.}
    \label{fig: ablation action model and scale}
\end{figure*}

\clearpage

\subsection{MuJoCo Playground per-task results}
\label{appendix: individual mjx}

The per-task DMC results in Fig.~\ref{fig:main_results_mjx_dmc_per_task} show that AMDP is competitive across most environments.
Several tasks, including Cartpole, Reacher, and Walker variants, approach saturated performance across different methods.
Differences are more pronounced on the harder locomotion tasks, such as AcrobotSwingup, FishSwim and CartpoleSwingupSparse, where AMDP performs strongly, whereas DIME and REPPO perform better on AcrobotSwingupSparse and the Hopper tasks.
AMDP-BoN evaluation improves the performance in comparison to AMDP on several tasks while having little effect where returns are already saturated.

\begin{figure*}[!htbp]
    \centering
    \small
    \legendMainResultsAMDPBoN

    \vspace{0.35em}

    \includegraphics[width=0.24\textwidth]{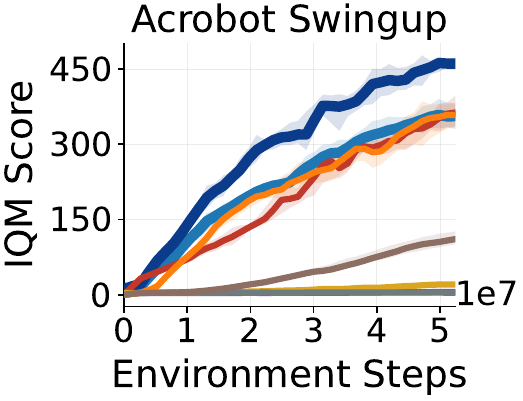}
    \includegraphics[width=0.24\textwidth]{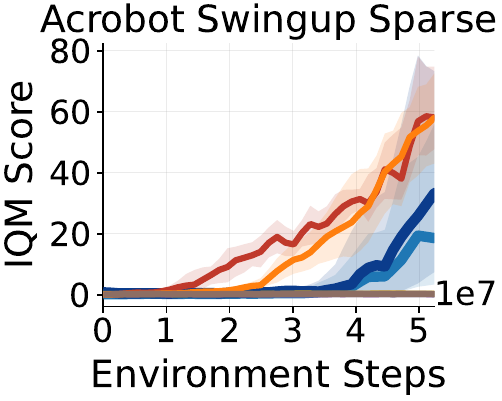}
    \includegraphics[width=0.24\textwidth]{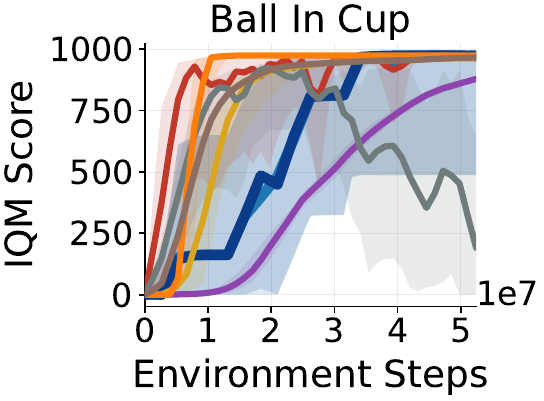}
    \includegraphics[width=0.24\textwidth]{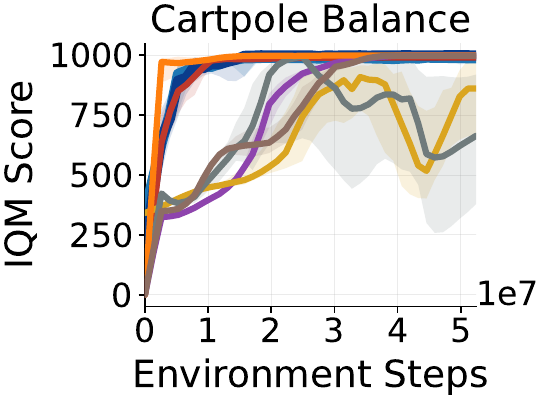}

    \vspace{0.1em}

    \includegraphics[width=0.24\textwidth]{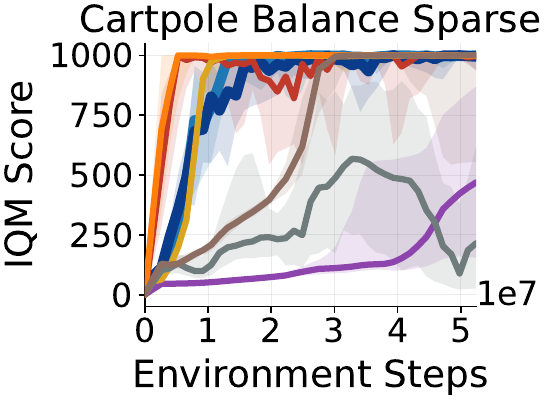}
    \includegraphics[width=0.24\textwidth]{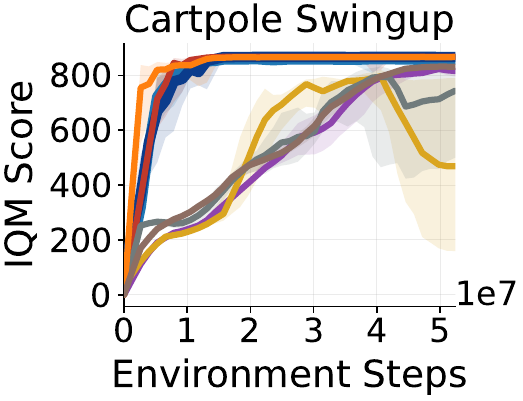}
    \includegraphics[width=0.24\textwidth]{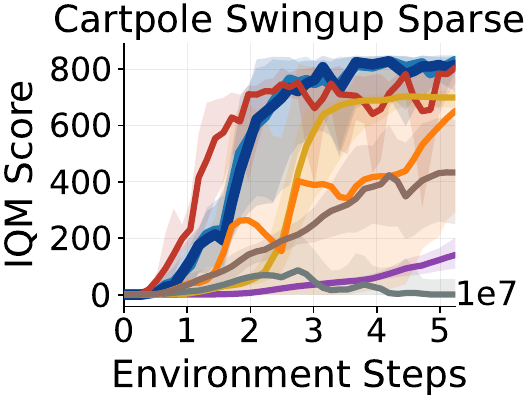}
    \includegraphics[width=0.24\textwidth]{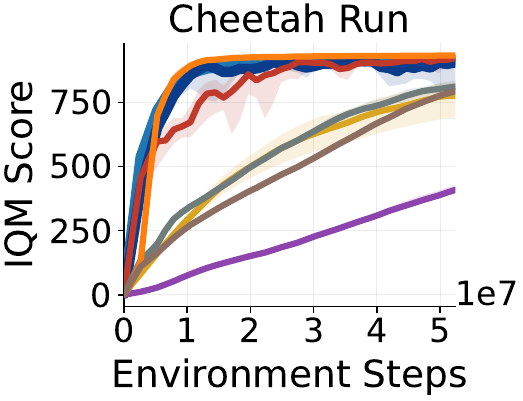}

    \vspace{0.1em}

    \includegraphics[width=0.24\textwidth]{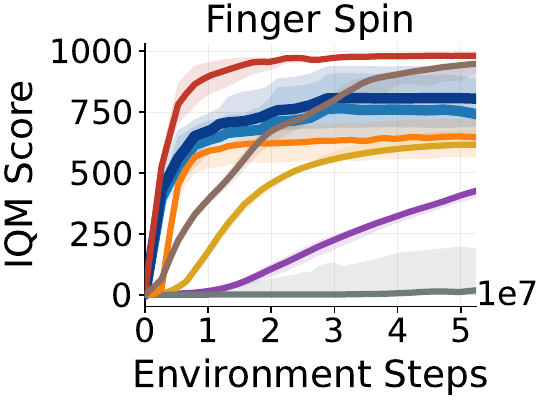}
    \includegraphics[width=0.24\textwidth]{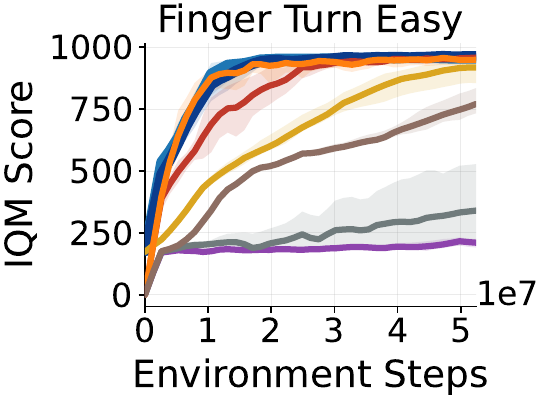}
    \includegraphics[width=0.24\textwidth]{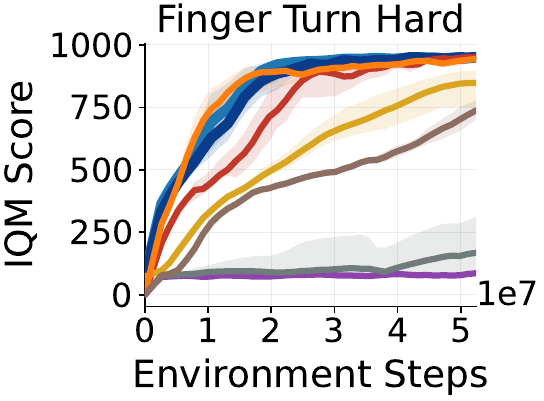}
    \includegraphics[width=0.24\textwidth]{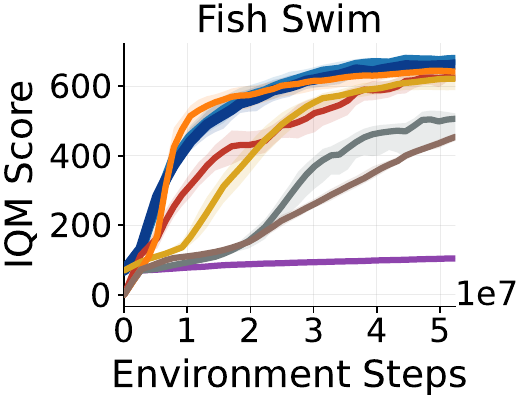}

    \vspace{0.1em}

    \includegraphics[width=0.24\textwidth]{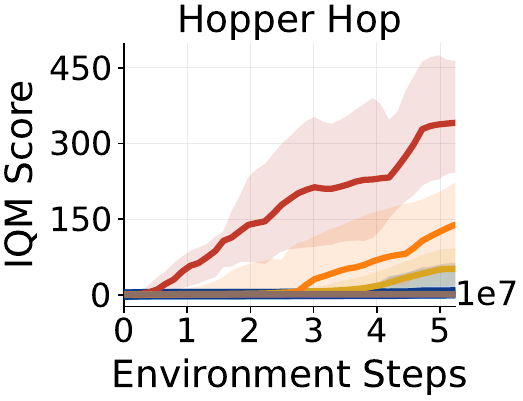}
    \includegraphics[width=0.24\textwidth]{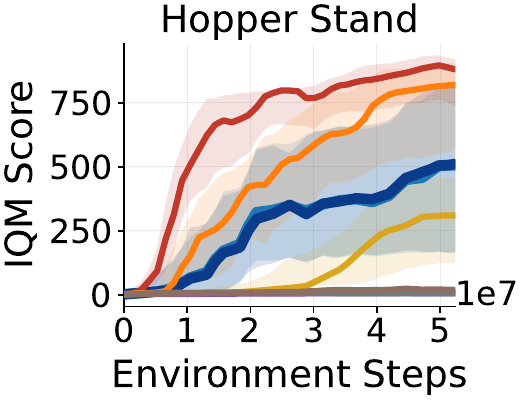}
    \includegraphics[width=0.24\textwidth]{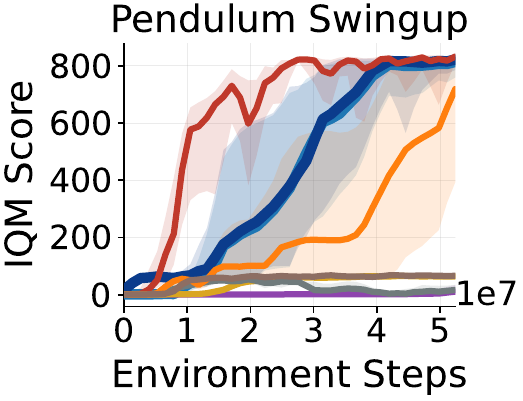}
    \includegraphics[width=0.24\textwidth]{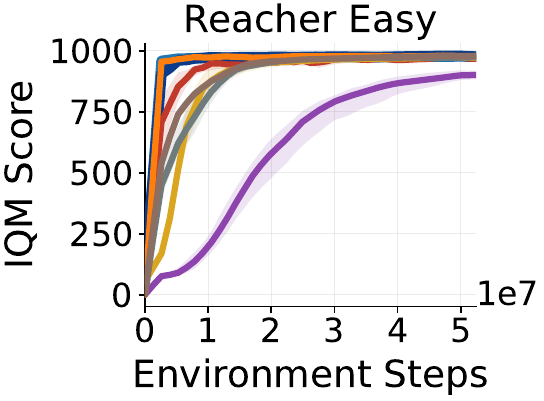}

    \vspace{0.1em}

    \includegraphics[width=0.24\textwidth]{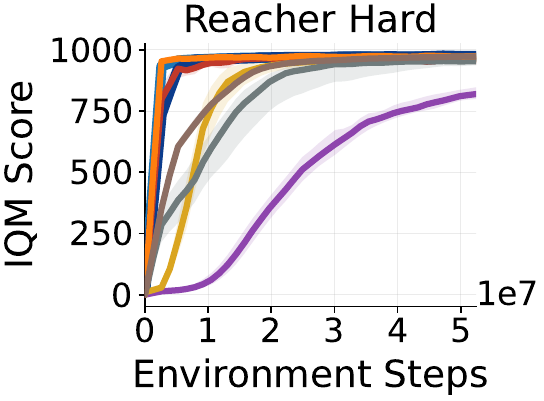}
    \includegraphics[width=0.24\textwidth]{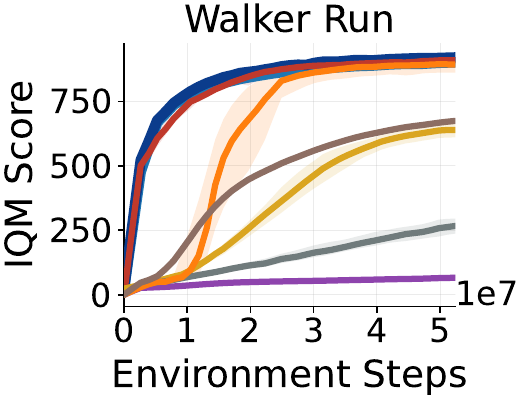}
    \includegraphics[width=0.24\textwidth]{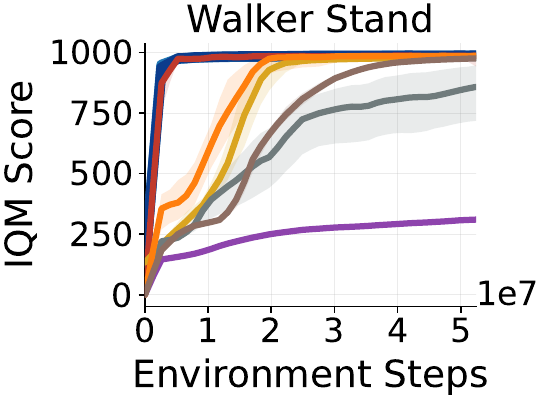}
    \includegraphics[width=0.24\textwidth]{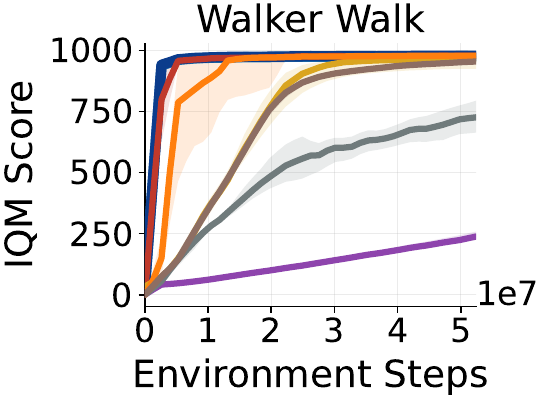}

    \caption{\textbf{MuJoCo Playground DMC per-task results.}
    Single-task IQM learning curves for the main comparison.
    Shaded regions show 95\% bootstrap confidence intervals.}
    \label{fig:main_results_mjx_dmc_per_task}
\end{figure*}

The Humanoid results in Fig.~\ref{fig:main_results_mjx_humanoid_per_task} exhibit a consistent sample-efficiency advantage of AMDP across all tasks in comparison to all other algorithms, where AMDP-BoN evaluation consistently improves its final performance.

\begin{figure*}[!htbp]
    \centering
    \small
    \legendMainResultsAMDPBoN

    \vspace{0.35em}

    \includegraphics[width=0.24\textwidth]{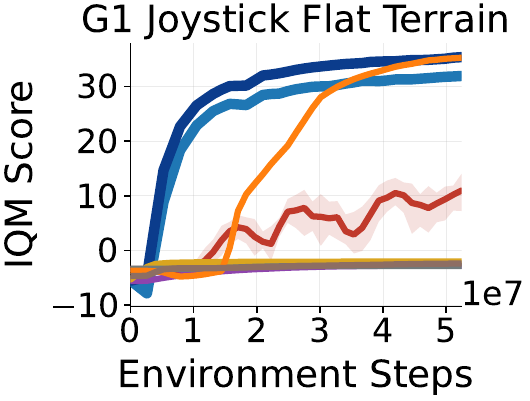}
    \includegraphics[width=0.24\textwidth]{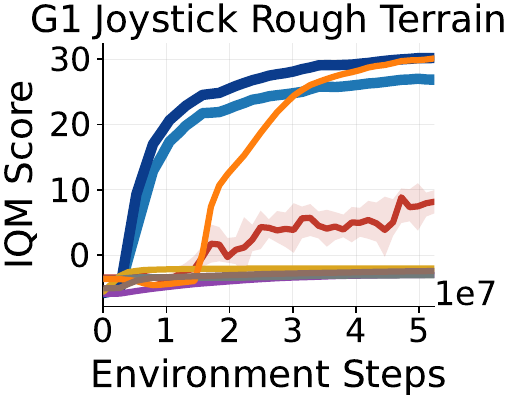}
    \includegraphics[width=0.24\textwidth]{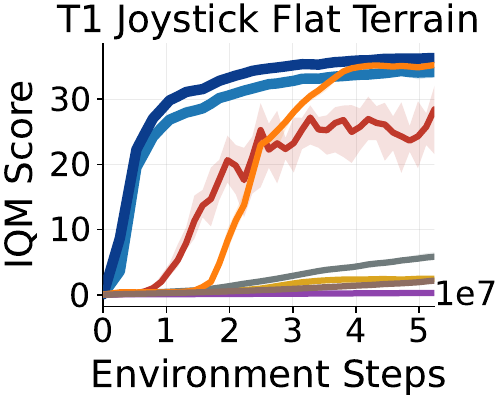}
    \includegraphics[width=0.24\textwidth]{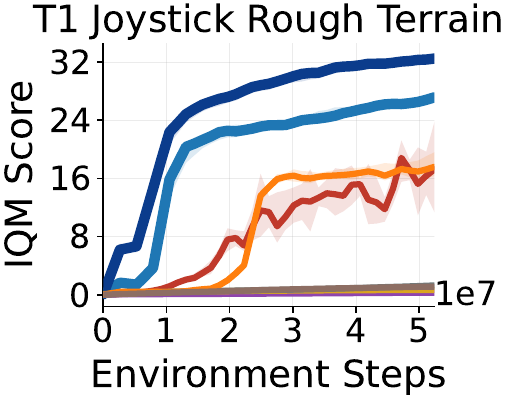}
    \caption{\textbf{MuJoCo Playground Humanoid per-task results.}
   Single-task IQM learning curves for the main comparison.
   Shaded regions show 95\% bootstrap confidence intervals.}
    \label{fig:main_results_mjx_humanoid_per_task}
\end{figure*}

\FloatBarrier

\subsection{ManiSkill3 per-task results}
\label{appendix: individual maniskill}

The per-task results in Fig.~\ref{fig:main_results_maniskill3_per_task} support similar aggregate performance of AMDP, DIME, and REPPO.
All three methods exhibit better sample-efficiency in comparison to all other methods on most of the reported tasks, including \texttt{LiftPegUpright}, \texttt{PegInsertionSide}, \texttt{PlaceSphere}, \texttt{PokeCube}, \texttt{RollBall}, and \texttt{StackCube}.

\begin{figure*}[!htbp]
    \centering
    \small
    \legendManiSkillResults

    \vspace{0.35em}

    \includegraphics[width=0.24\textwidth]{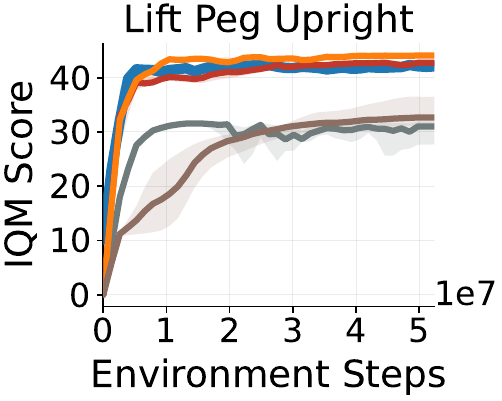}%
    \includegraphics[width=0.24\textwidth]{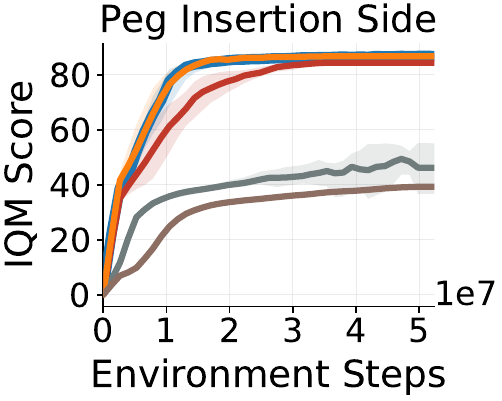}%
    \includegraphics[width=0.24\textwidth]{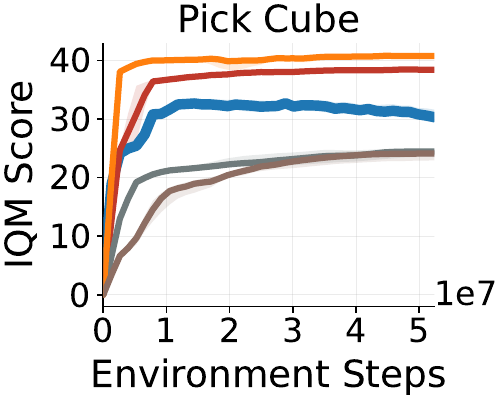}%
    \includegraphics[width=0.24\textwidth]{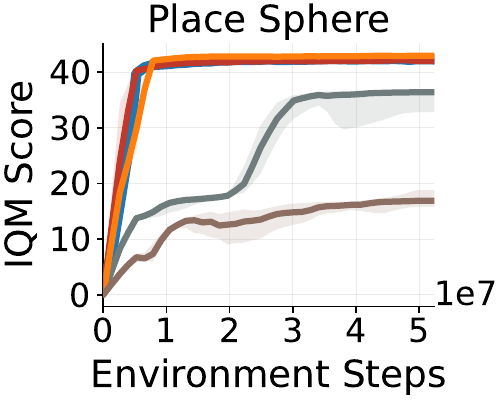}%

    \vspace{0.1em}
    
    \includegraphics[width=0.24\textwidth]{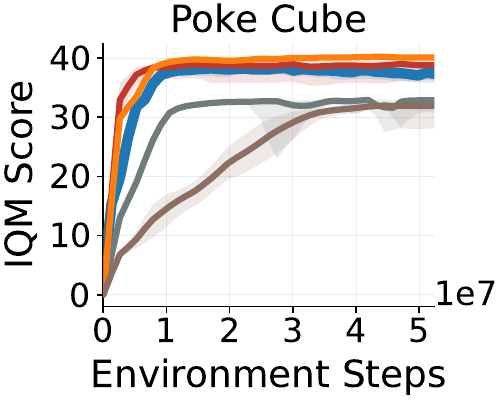}%
    \includegraphics[width=0.24\textwidth]{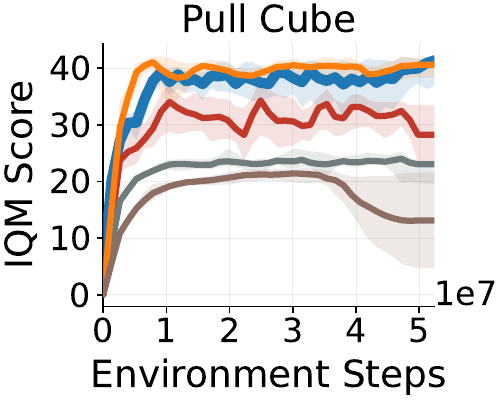}%
    \includegraphics[width=0.24\textwidth]{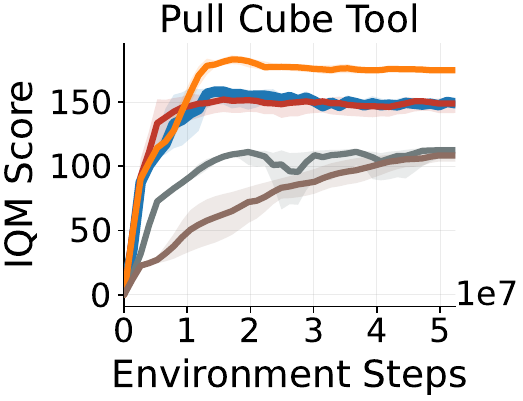}%
    \includegraphics[width=0.24\textwidth]{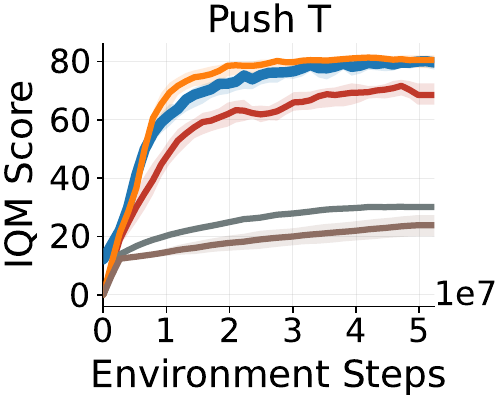}%

    \vspace{0.1em}

    \makebox[\textwidth][c]{%
        \makebox[0.96\textwidth][l]{%
            \includegraphics[width=0.24\textwidth]{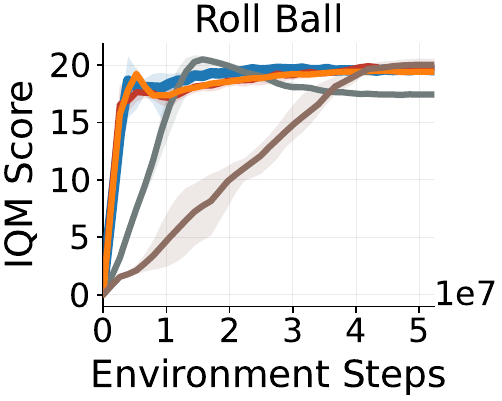}%
            \includegraphics[width=0.24\textwidth]{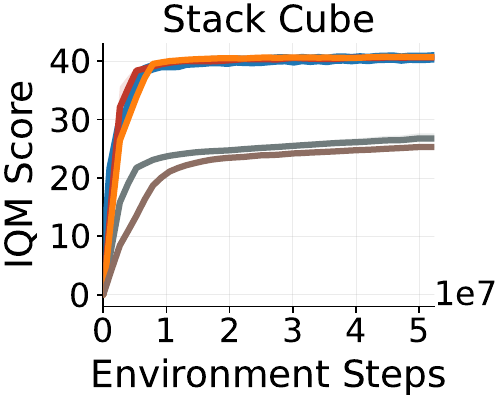}%
        }%
    }

    \caption{\textbf{ManiSkill3 per-task results.}
    Single-task IQM learning curves for the main comparison.
    Shaded regions show 95\% bootstrap confidence intervals.}
    \label{fig:main_results_maniskill3_per_task}
\end{figure*}

\clearpage

\subsection{HumanoidBench per-task results}
\label{appendix: individual humanoidbench}

Figure~\ref{fig:main_results_humanoid_bench_per_task} show that AMDP clearly outperforms all baselines on several challenging tasks, including balance-simple, balance-hard, cabinet, cube, push, stair and window.
These tasks combine high-dimensional control space of the humanoid embodiment with demanding balance, long-horizon reasoning, and in many cases, contact-rich manipulation.
These results suggest that AMDP's expressive policy representation is particularly beneficial for complex whole-body control and contact-rich manipulation.

\begin{figure*}[!htbp]
    \centering
    \small
    \legendHumanoidBenchResults

    \vspace{0.35em}

    \includegraphics[width=0.235\textwidth]{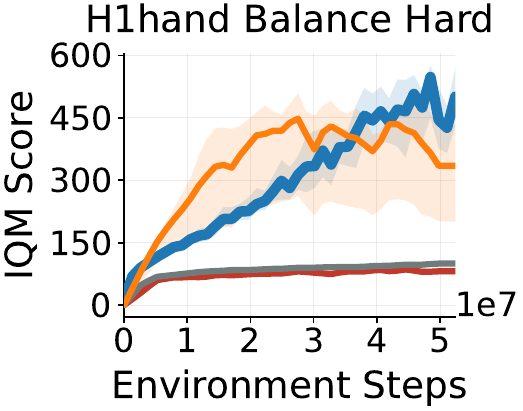}%
    \includegraphics[width=0.24\textwidth]{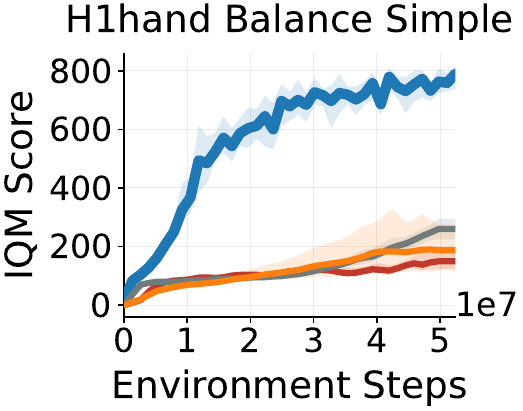}%
    \includegraphics[width=0.24\textwidth]{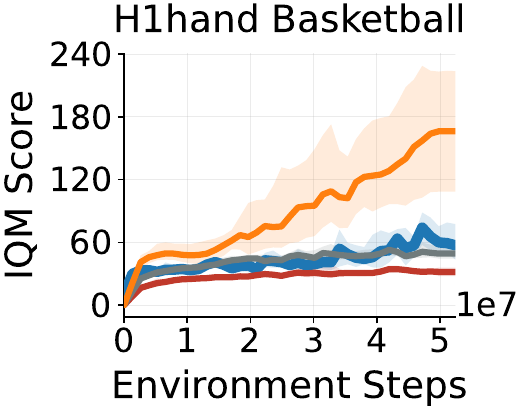}%
    \includegraphics[width=0.24\textwidth]{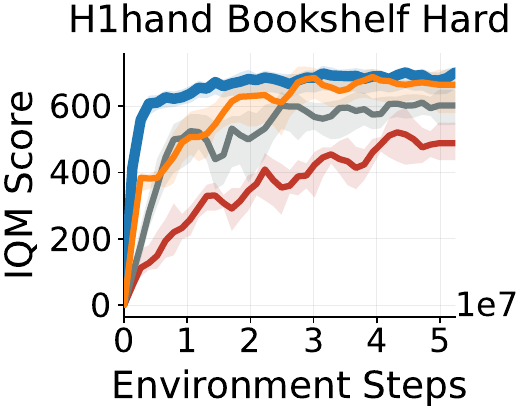}%

    \vspace{0.1em}

    \includegraphics[width=0.24\textwidth]{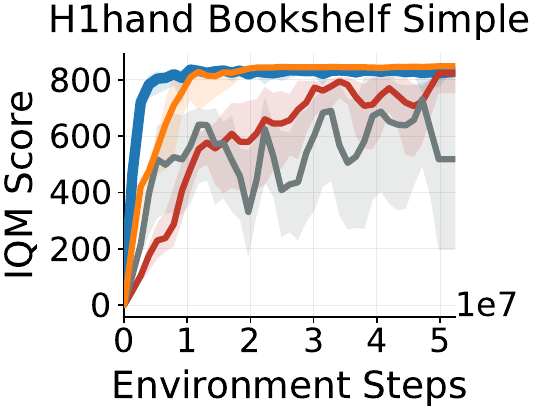}%
    \includegraphics[width=0.24\textwidth]{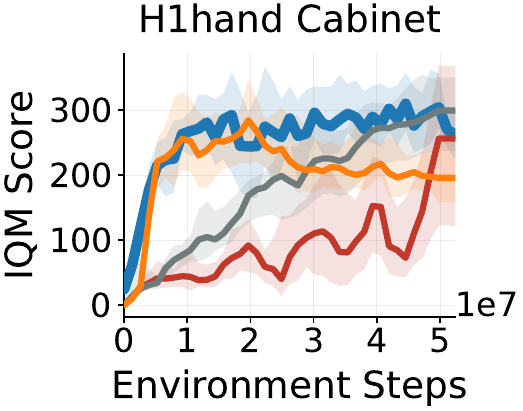}%
    \includegraphics[width=0.24\textwidth]{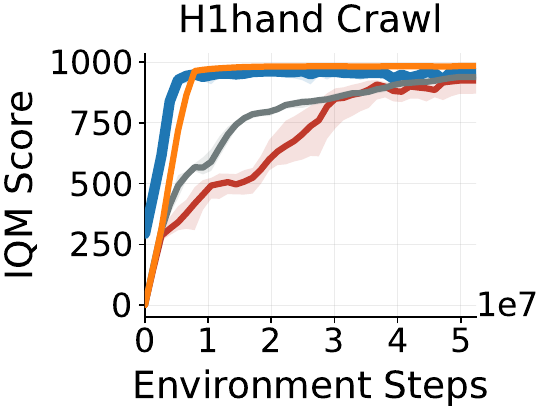}%
    \includegraphics[width=0.24\textwidth]{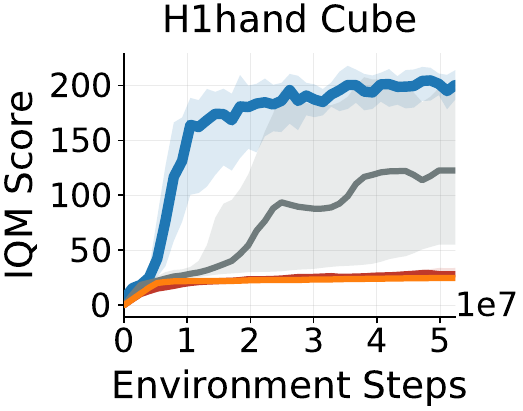}%

    \vspace{0.1em}

    \includegraphics[width=0.24\textwidth]{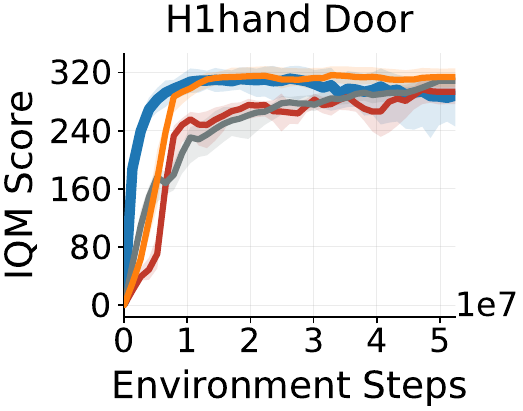}%
    \includegraphics[width=0.24\textwidth]{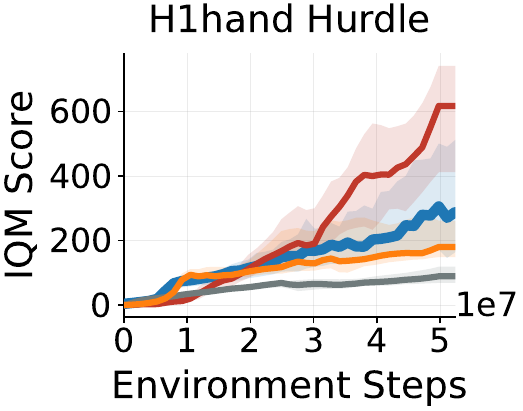}%
    \includegraphics[width=0.24\textwidth]{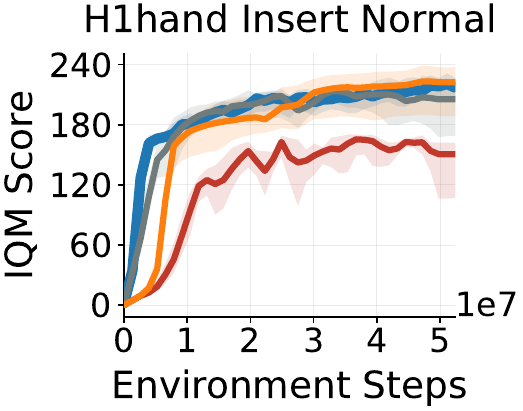}%
    \includegraphics[width=0.24\textwidth]{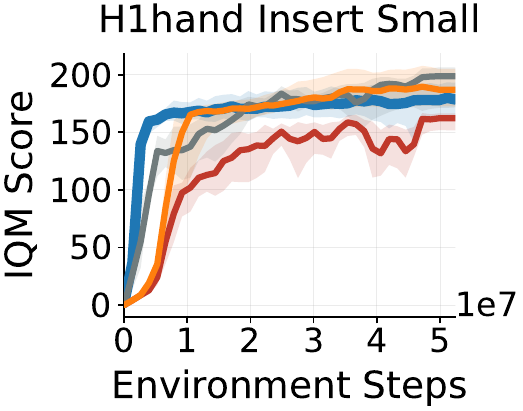}%

    \vspace{0.1em}

    \includegraphics[width=0.24\textwidth]{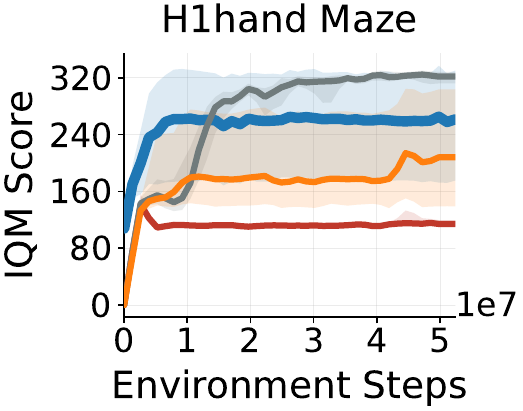}%
    \includegraphics[width=0.24\textwidth]{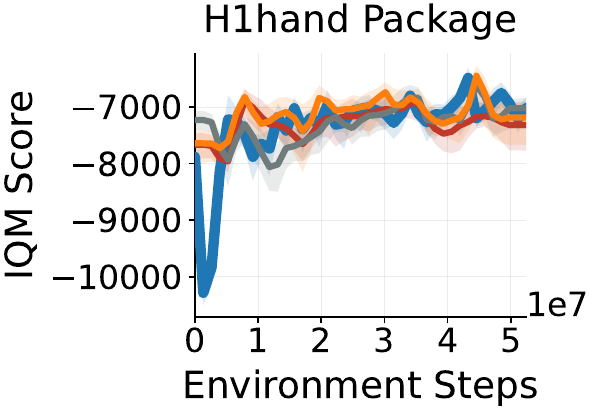}%
    \includegraphics[width=0.24\textwidth]{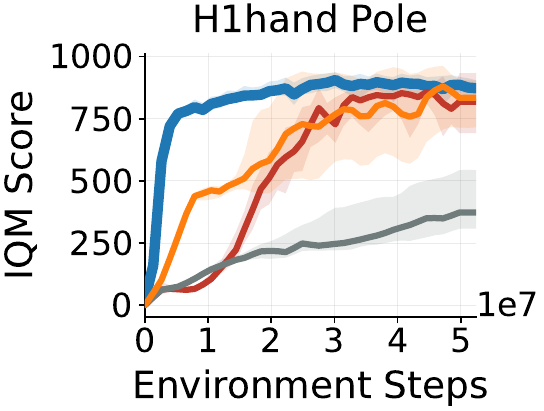}%
    \includegraphics[width=0.24\textwidth]{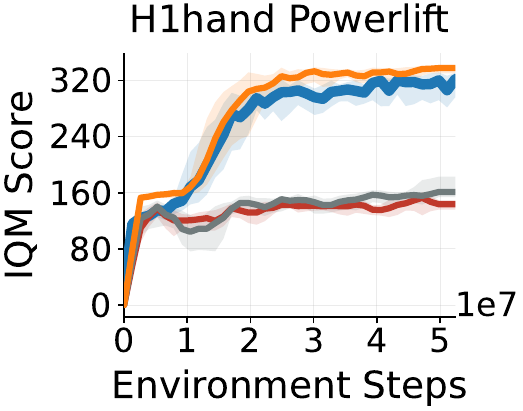}%

    \vspace{0.1em}

    \includegraphics[width=0.24\textwidth]{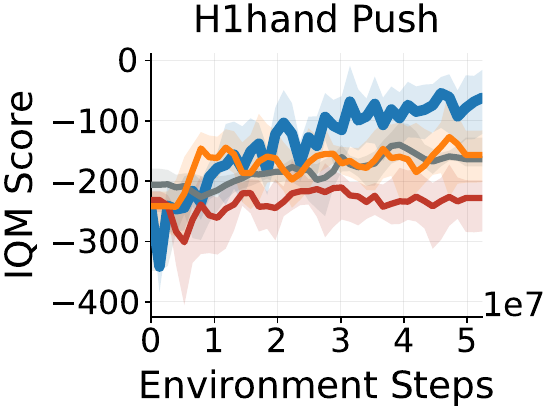}%
    \includegraphics[width=0.24\textwidth]{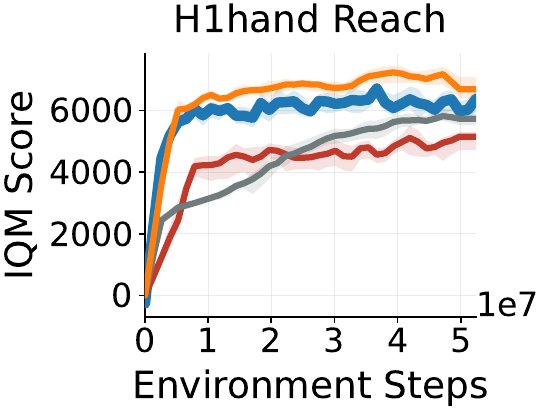}%
    \includegraphics[width=0.24\textwidth]{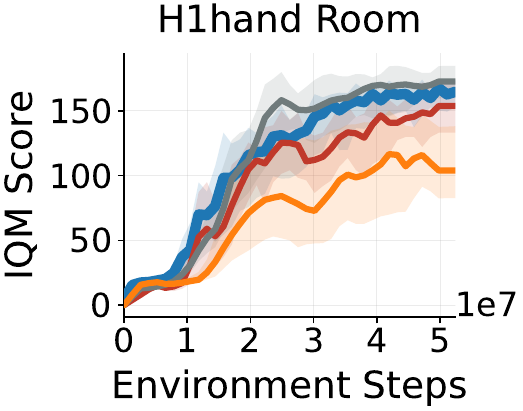}%
    \includegraphics[width=0.24\textwidth]{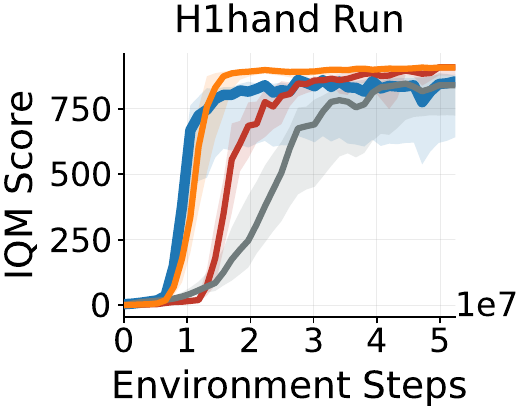}%

    \vspace{0.1em}

    \includegraphics[width=0.24\textwidth]{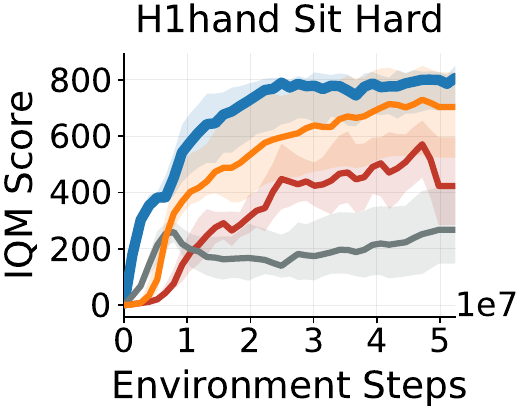}%
    \includegraphics[width=0.24\textwidth]{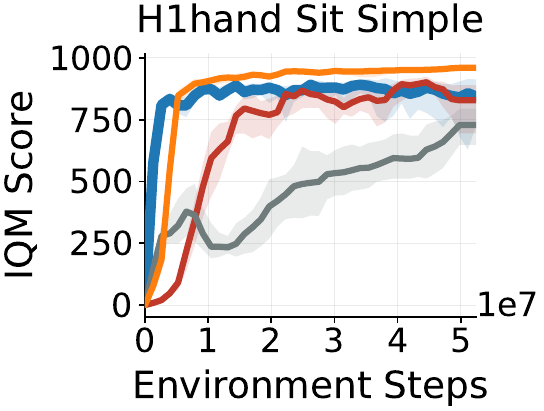}%
    \includegraphics[width=0.24\textwidth]{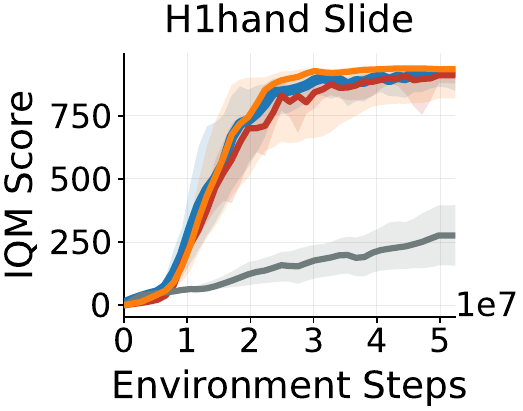}%
    \includegraphics[width=0.24\textwidth]{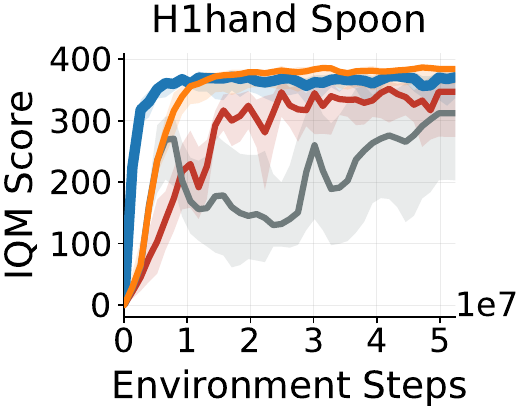}%

    \vspace{0.1em}

    \includegraphics[width=0.24\textwidth]{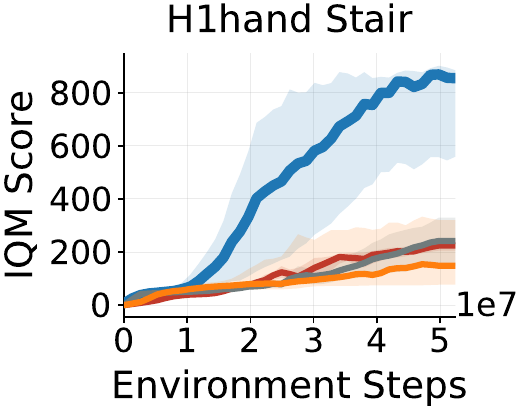}%
    \includegraphics[width=0.24\textwidth]{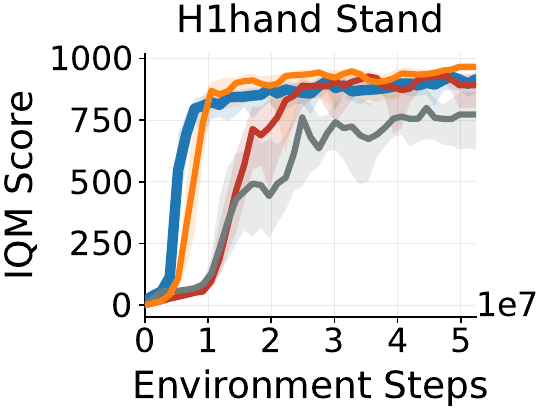}%
    \includegraphics[width=0.24\textwidth]{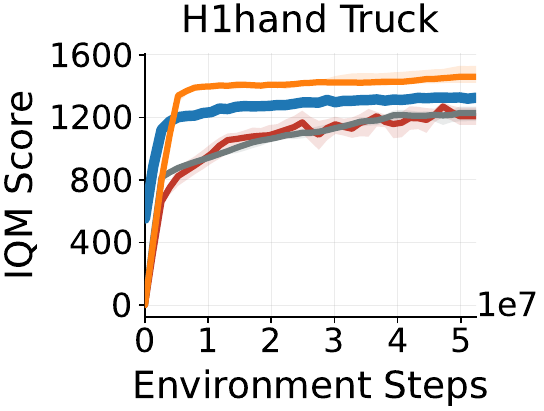}%
    \includegraphics[width=0.24\textwidth]{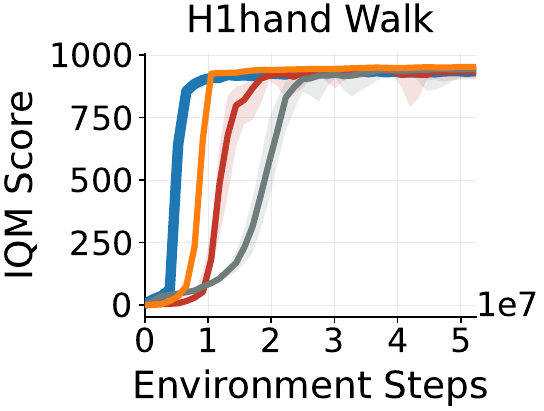}%

    \vspace{0.1em}

    \makebox[\textwidth][c]{%
        \makebox[0.96\textwidth][l]{%
            \includegraphics[width=0.24\textwidth]{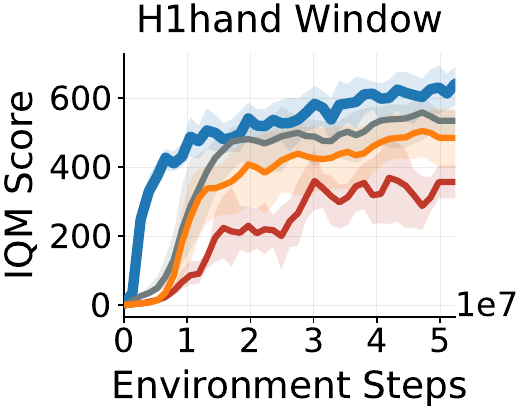}%
        }%
    }

    \caption{\textbf{HumanoidBench per-task results.}
    Single-task IQM learning curves for the main comparison.
    Shaded regions show 95\% bootstrap confidence intervals.}
    \label{fig:main_results_humanoid_bench_per_task}
\end{figure*}

\clearpage

\subsection{Off-policy DMC per-task results}
\label{appendix: individual offpolicy}

The individual curves in Fig.~\ref{fig:off_policy_per_task_results} indicate that AMDP and DIME are the strongest methods across all off-policy methods, where AMDP consistently matches DIME's performance across the dog and humanoid tasks.
These results demonstrate that AMDP's adjoint-matching objective is competitive with DIME's reverse-KL objective in the off-policy setting.

\begin{figure*}[!htbp]
    \centering
    \small
    \legendOffPolicy

    \vspace{0.3cm}

    \includegraphics[width=0.24\textwidth]{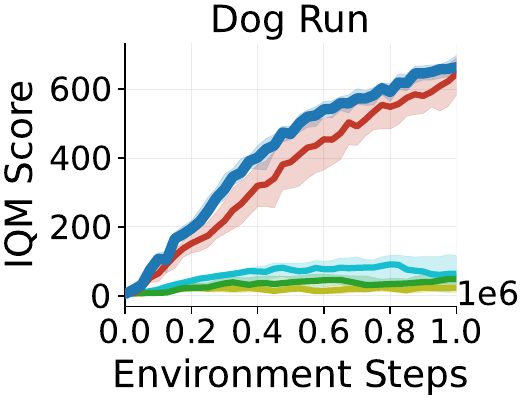}
    \includegraphics[width=0.24\textwidth]{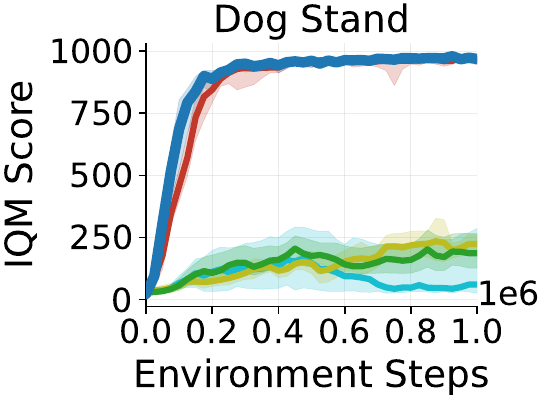}
    \includegraphics[width=0.24\textwidth]{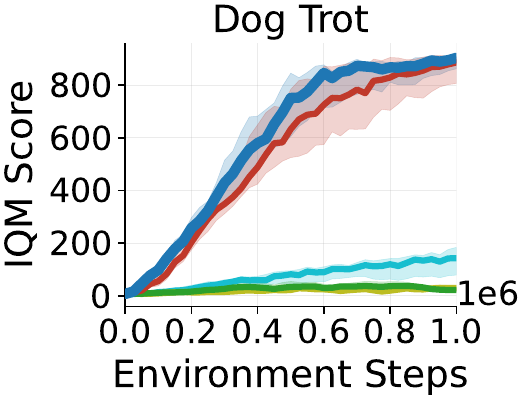}
    \includegraphics[width=0.24\textwidth]{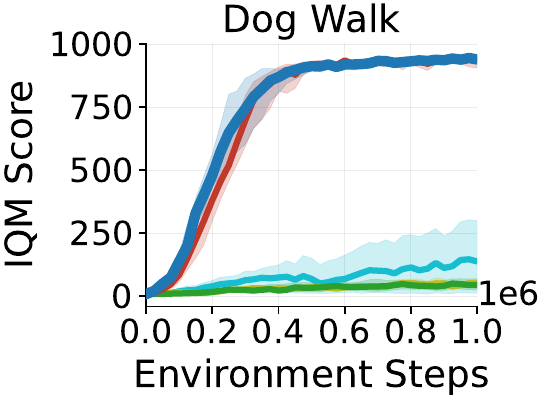}

    \vspace{0.3cm}

    \makebox[\textwidth][c]{%
        \makebox[0.98\textwidth][l]{%
            \includegraphics[width=0.24\textwidth]{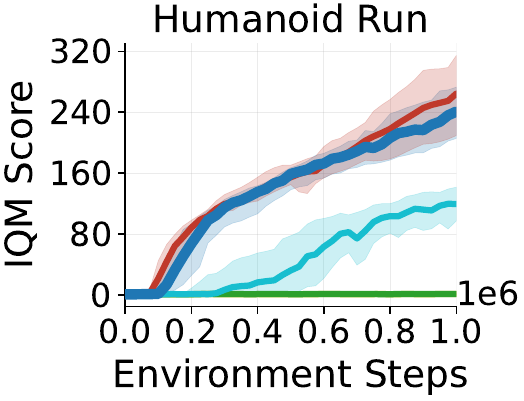}
            \includegraphics[width=0.24\textwidth]{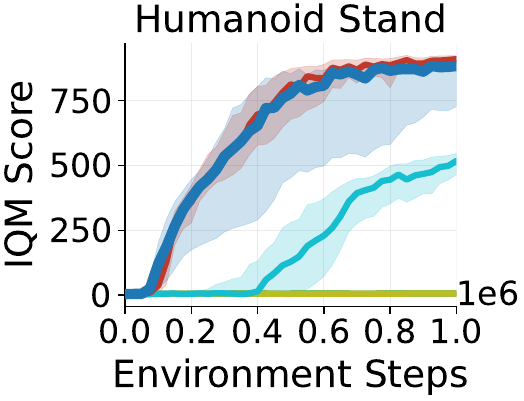}
            \includegraphics[width=0.24\textwidth]{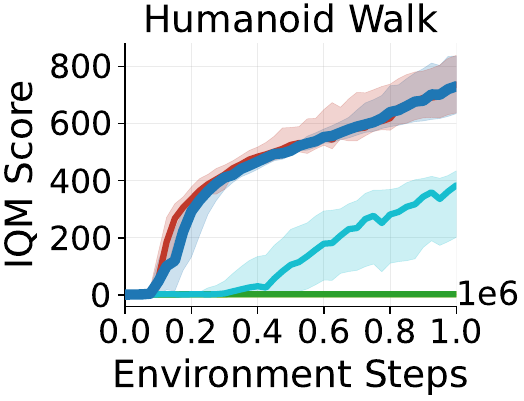}
        }%
    }

    \caption{\textbf{Off-policy per-task performance comparison.}
    Single-task IQM learning curves for the off-policy comparison.
    Shaded regions indicate uncertainty estimates across seeds.}
    \label{fig:off_policy_per_task_results}
\end{figure*}

\end{document}

%% file: figures/latex_plot_colors.tex
\definecolor{colAMDP}{HTML}{1F77B4}
\definecolor{colAMDPBoN}{HTML}{0B3C8C}
\definecolor{colTruDi}{HTML}{2E6F40}
\definecolor{colREPPO}{HTML}{FF7F0E}
\definecolor{colDIME}{HTML}{C0392B}
\definecolor{colSPO}{HTML}{8D6E63}
\definecolor{colPPO}{HTML}{707B7C}
\definecolor{colDPPO}{HTML}{8E44AD}
\definecolor{colFPO}{HTML}{DAA520}

\definecolor{colQSM}{HTML}{17BECF}
\definecolor{colDiffQL}{HTML}{2CA02C}
\definecolor{colDIPO}{HTML}{E377C2}
\definecolor{colConsistencyAC}{HTML}{BCBD22}

\colorlet{colOffAMDP}{colAMDP}
\colorlet{colOffDIME}{colDIME}
\colorlet{colOffQSM}{colQSM}
\colorlet{colOffDiffQL}{colDiffQL}
\colorlet{colOffDIPO}{colDIPO}
\colorlet{colOffConsistencyAC}{colConsistencyAC}

\definecolor{colAbl01}{HTML}{332288}
\definecolor{colAbl02}{HTML}{88CCEE}
\definecolor{colAbl03}{HTML}{44AA99}
\definecolor{colAbl04}{HTML}{117733}
\definecolor{colAbl05}{HTML}{DDCC77}
\definecolor{colAbl06}{HTML}{CC6677}
\definecolor{colAbl07}{HTML}{AA4499}
\definecolor{colAbl08}{HTML}{882255}
\definecolor{colAbl09}{HTML}{661100}
\definecolor{colAbl10}{HTML}{6699CC}
\definecolor{colAbl11}{HTML}{AA4466}
\definecolor{colAbl12}{HTML}{999933}
\definecolor{colAbl13}{HTML}{1B9E77}
\definecolor{colAbl14}{HTML}{D95F02}
\definecolor{colAbl15}{HTML}{7570B3}
\definecolor{colAbl16}{HTML}{E7298A}
\definecolor{colAbl17}{HTML}{66A61E}
\definecolor{colAbl18}{HTML}{E6AB02}
\definecolor{colAbl19}{HTML}{A6761D}
\definecolor{colAbl20}{HTML}{666666}

\colorlet{colAblBlue}{colAbl01}
\colorlet{colAblOrange}{colAbl02}
\colorlet{colAblGreen}{colAbl03}
\colorlet{colAblRed}{colAbl04}
\colorlet{colAblPurple}{colAbl05}
\colorlet{colAblBrown}{colAbl06}
\colorlet{colAblPink}{colAbl07}
\colorlet{colAblGray}{colAbl08}
\colorlet{colAblOlive}{colAbl09}
\colorlet{colAblCyan}{colAbl10}

\colorlet{amdpblue}{colAMDP}
\colorlet{pisblue}{colAMDP}
\colorlet{amdpbonblue}{colAMDPBoN}
\colorlet{trudigreen}{colTruDi}
\colorlet{reppoorange}{colREPPO}
\colorlet{dimered}{colDIME}
\colorlet{spobrown}{colSPO}
\colorlet{ppogray}{colPPO}
\colorlet{dppopurple}{colDPPO}
\colorlet{fpogold}{colFPO}

\colorlet{offamdp}{colAMDP}
\colorlet{offdime}{colDIME}
\colorlet{offqsm}{colQSM}
\colorlet{offdiffql}{colDiffQL}
\colorlet{offdipo}{colDIPO}
\colorlet{offcac}{colConsistencyAC}

\colorlet{steelblue31119180}{colAblBlue}
\colorlet{darkorange25512714}{colAblOrange}
\colorlet{forestgreen4416044}{colAblGreen}
\colorlet{crimson2143940}{colAblRed}
\colorlet{mediumpurple148103189}{colAblPurple}
\colorlet{saddlebrown1408675}{colAblBrown}
\colorlet{deeppink227119194}{colAblPink}
\colorlet{gray127127127}{colAblGray}
\colorlet{olive18818934}{colAblOlive}
\colorlet{darkcyan231901207}{colAblCyan}

\colorlet{steelblue48105152}{colAMDP}
\colorlet{firebrick1925743}{colDIME}
\colorlet{goldenrod}{colFPO}
\colorlet{gray14111099}{colSPO}
\colorlet{gray112123124}{colPPO}
\colorlet{darkorchid14268173}{colDPPO}

\colorlet{colPIS}{colAMDP}
\colorlet{colTRDP}{colAMDP}
\colorlet{colReppo}{colREPPO}
\colorlet{colDime}{colDIME}

\providecommand{\legendline}{}
\renewcommand{\legendline}[2]{%
  \tikz[baseline=-0.6ex]\draw[draw=#1, line width=2.2pt] (0,0) -- (0.75cm,0);%
  \hspace{0.3em}#2%
}

%% file: figures/latex_plot_legends.tex
\newcommand{\legendMainResultsAMDPBoN}{%
  \begin{tabular}{llll}
  \legendline{colAMDP}{AMDP (Ours)} &
  \legendline{colAMDPBoN}{AMDP BoN (Ours)} &
  \legendline{colDIME}{DIME} &
  \legendline{colDPPO}{DPPO} \\
  \legendline{colFPO}{FPO} &
  \legendline{colPPO}{PPO} &
  \legendline{colREPPO}{REPPO} &
  \legendline{colSPO}{SPO}
  \end{tabular}%
}

\newcommand{\legendManiSkillResults}{%
  \begin{tabular}{lllll}
  \legendline{colAMDP}{AMDP (Ours)} &
  \legendline{colDIME}{DIME} &
  \legendline{colPPO}{PPO} &
  \legendline{colREPPO}{REPPO} &
  \legendline{colSPO}{SPO}
  \end{tabular}%
}

\newcommand{\legendHumanoidBenchResults}{%
  \begin{tabular}{llll}
  \legendline{colAMDP}{AMDP (Ours)} &
  \legendline{colDIME}{DIME} &
  \legendline{colPPO}{PPO} &
  \legendline{colREPPO}{REPPO}
  \end{tabular}%
}

\newcommand{\legendOffPolicy}{%
  \begin{tabular}{lll}
  \legendline{offamdp}{AMDP (Ours)} &
  \legendline{offdime}{DIME} &
  \legendline{offcac}{Consistency-AC} \\
  \legendline{offqsm}{QSM} &
  \legendline{offdiffql}{Diff-QL} &
  \end{tabular}%
}

\newcommand{\legendDiffusionStepsAndReverseKL}{%
  \setlength{\tabcolsep}{4pt}%
  \begin{tabular}{llll}
  \legendline{colAbl01}{4 steps} &
  \legendline{colAbl05}{16 steps, 20 batch reps} &
  \legendline{colAMDP}{\texttt{AMDP} (default)} &
  \legendline{colAbl10}{\texttt{AMDP-OU}} \\
  \legendline{colAbl03}{32 steps} &
  \legendline{colAbl04}{128 steps} &
  \legendline{colAbl07}{Reverse-KL loss} &
  \legendline{colDIME}{DIME}
  \end{tabular}%
}

\newcommand{\legendTrustRegionSizeAndSquashing}{%
  \begin{tabular}{llll}
  \legendline{colAbl11}{$\varepsilon = 0.01$} &
  \legendline{colAbl12}{$\varepsilon = 0.03$} &
  \legendline{colAbl13}{$\varepsilon = 0.3$} &
  \legendline{colAMDP}{AMDP (default)} \\
  \legendline{colAbl14}{$\varepsilon = 1.0$} &
  \legendline{colDIME}{$\varepsilon = \infty$} &
  \phantom{\legendline{colAbl08}{tanh squash}} & 
  \legendline{colAbl08}{tanh squash}
  {}
  \end{tabular}%
}

\newcommand{\legendModelArchitectures}{%
  \begin{tabular}{lll}
  \legendline{colAMDP}{default: 290k} &
  \legendline{colAbl16}{residual 2.7M} &
  \legendline{colAbl17}{residual 106M}
  \end{tabular}%
}

%% file: chapters/acknowledgements.tex
\section{Acknowledgements}
The authors thank the anonymous reviewers for their insightful feedback, and Huy Le for providing results of several baselines.
The authors acknowledge support by the state of Baden-W\"urttemberg through bwHPC, as well as the HoreKa supercomputer funded by the Ministry of Science, Research and the Arts Baden-W\"urttemberg and by the German Federal Ministry of Education and Research.
This work was supported by the European Research Council (ERC) under the European Union's Horizon Europe programme through the project SMARTI$^3$ (Grant Agreement No. 101171393), by the German Research Foundation (DFG, Deutsche Forschungsgemeinschaft) as part of the SFB-1574 -- 471687386 ``Circular Factory'', by the German Federal Ministry of Research, Technology, and Space (BMFTR) under the Robotics Institute Germany (RIG), and by the Helmholtz Association's Initiative and Networking Fund on the HAICORE@KIT and HAICORE@FZJ partition.
D.B. acknowledges support by funding from a Google PhD fellowship in Machine Learning and ML Foundations, the pilot program Core Informatics of the Helmholtz Association (HGF) and GN was supported by the Carl Zeiss Foundation under the project JuBot (Jung Bleiben mit Robotern).

%% file: tables/appendix_mujoco_playground.tex
\begin{table*}[!htbp]
    \centering
    \caption{MuJoCo Playground evaluation tasks used in our experiments.
    Policy and critic observation dimensions are reported separately.
    DM Control tasks use identical observations for both networks, while humanoid tasks provide privileged observations only to the critic.}
    \label{tab:mujoco_playground_tasks}
    \small
    \setlength{\tabcolsep}{3pt}
    \renewcommand{\arraystretch}{1.05}
    \begin{tabularx}{\textwidth}{
        @{}p{0.12\textwidth}
        p{0.29\textwidth}
        cc
        c
        >{\raggedright\arraybackslash}X@{}
    }
        \toprule
        \textbf{Category}                     & \textbf{Task}                   & \textbf{Policy Obs.} & \textbf{Critic Obs.} & \textbf{Act.} & \textbf{Description}                                    \\
        \midrule

        \multirow{20}{*}{\texttt{DM Control}} & \texttt{AcrobotSwingup}         & 6                    & 6                    & 1             & Swing up and balance the pole.                          \\
                                              & \texttt{AcrobotSwingupSparse}   & 6                    & 6                    & 1             & Sparse reward variant of \texttt{AcrobotSwingup}.       \\

                                              & \texttt{BallInCup}              & 8                    & 8                    & 2             & Put the ball in the cup.                                \\

                                              & \texttt{CartpoleBalance}        & 5                    & 5                    & 1             & Balance the cartpole.                                   \\
                                              & \texttt{CartpoleBalanceSparse}  & 5                    & 5                    & 1             & Sparse reward variant of \texttt{CartpoleBalance}.      \\
                                              & \texttt{CartpoleSwingup}        & 5                    & 5                    & 1             & Swing up the cartpole.                                  \\
                                              & \texttt{CartpoleSwingupSparse}  & 5                    & 5                    & 1             & Sparse reward variant of \texttt{CartpoleSwingup}.      \\

                                              & \texttt{CheetahRun}             & 17                   & 17                   & 6             & Train a planar cheetah to run.                          \\

                                              & \texttt{FingerSpin}             & 9                    & 9                    & 2             & Spin the body.                                          \\
                                              & \texttt{FingerTurnEasy}         & 12                   & 12                   & 2             & Turn the body to a target angle.                        \\
                                              & \texttt{FingerTurnHard}         & 12                   & 12                   & 2             & Harder target-angle variant of \texttt{FingerTurnEasy}. \\

                                              & \texttt{FishSwim}               & 24                   & 24                   & 5             & Swim with smooth reward.                                \\

                                              & \texttt{HopperStand}            & 15                   & 15                   & 4             & Balance an upright hopper.                              \\
                                              & \texttt{HopperHop}              & 15                   & 15                   & 4             & Hop forward.                                            \\

                                              & \texttt{PendulumSwingup}        & 3                    & 3                    & 1             & Swing up and balance the pendulum.                      \\

                                              & \texttt{ReacherEasy}            & 6                    & 6                    & 2             & Reach a randomized target with relaxed tolerance.       \\
                                              & \texttt{ReacherHard}            & 6                    & 6                    & 2             & Reach a randomized target with stricter tolerance.      \\

                                              & \texttt{WalkerStand}            & 24                   & 24                   & 6             & Stand upright.                                          \\
                                              & \texttt{WalkerWalk}             & 24                   & 24                   & 6             & Walk at a target horizontal velocity.                   \\
                                              & \texttt{WalkerRun}              & 24                   & 24                   & 6             & Run at a target horizontal velocity.                    \\
        \midrule

        \multirow{4}{*}{\texttt{Humanoid}}    & \texttt{G1JoystickFlatTerrain}  & 103                  & 216                  & 29            & Track joystick commands on flat terrain.                \\
                                              & \texttt{G1JoystickRoughTerrain} & 103                  & 216                  & 29            & Track joystick commands on rough terrain.               \\
                                              & \texttt{T1JoystickFlatTerrain}  & 85                   & 180                  & 23            & Track joystick commands on flat terrain.                \\
                                              & \texttt{T1JoystickRoughTerrain} & 85                   & 180                  & 23            & Track joystick commands on rough terrain.               \\
        \bottomrule
    \end{tabularx}
\end{table*}

%% file: tables/appendix_maniskill.tex
\begin{table}[!htbp]
    \centering
    \caption{ManiSkill3 evaluation tasks used in our experiments.}
    \label{tab:maniskill3_tasks}
    \small
    \setlength{\tabcolsep}{4pt}
    \renewcommand{\arraystretch}{1.05}
    \begin{tabularx}{\linewidth}{@{}%
        >{\raggedright\arraybackslash}p{0.24\linewidth}
        >{\raggedright\arraybackslash}p{0.28\linewidth}
        >{\raggedright\arraybackslash}X@{}}
        \toprule
        \textbf{Category}                       & \textbf{Task}                & \textbf{Description}                                  \\
        \midrule

        \multirow{10}{*}{\texttt{Table-top 2F}} & \texttt{PegInsertionSide-v1} & Pick up a peg and insert it into a box with a hole.   \\

                                                & \texttt{LiftPegUpright-v1}   & Move a peg lying on the table to an upright position. \\

                                                & \texttt{PickCube-v1}         & Grasp a cube and move it to a target goal position.   \\

                                                & \texttt{PlaceSphere-v1}      & Place a sphere into a shallow bin.                    \\

                                                & \texttt{PokeCube-v1}         & Poke a cube with a peg and push it to the target.     \\

                                                & \texttt{PullCube-v1}         & Pull a cube onto a target region.                     \\

                                                & \texttt{PullCubeTool-v1}     & Use an L-shaped tool to pull an out-of-reach cube.    \\

                                                & \texttt{PushT-v1}            & Push a T-shaped block into the target region.         \\

                                                & \texttt{RollBall-v1}         & Push and roll a ball to a goal region.                \\

                                                & \texttt{StackCube-v1}        & Stack one cube on top of another and release it.      \\

        \bottomrule
    \end{tabularx}
\end{table}

%% file: tables/appendix_humanoidbench.tex
\begin{table*}[!htbp]
    \centering
    \caption{HumanoidBench evaluation environments used in our experiments.}
    \label{tab:humanoidbench_tasks}
    \small
    \setlength{\tabcolsep}{3pt}
    \renewcommand{\arraystretch}{1.03}
    \begin{tabularx}{\textwidth}{@{}%
        >{\raggedright\arraybackslash}p{0.19\textwidth}
        >{\raggedright\arraybackslash}p{0.31\textwidth}
        >{\raggedright\arraybackslash}X@{}}
        \toprule
        \textbf{Category}                         & \textbf{Environment}                 & \textbf{Description}                                                             \\
        \midrule

        \multirow{14}{*}{Locomotion}              & \texttt{h1hand-stand-v0}             & Maintain a standing pose.                                                        \\
                                                  & \texttt{h1hand-walk-v0}              & Keep forward velocity close to the walking target without falling.               \\
                                                  & \texttt{h1hand-run-v0}               & Run forward at the target speed.                                                 \\
                                                  & \texttt{h1hand-reach-v0}             & Reach a randomly initialized 3D point with the left hand.                        \\
                                                  & \texttt{h1hand-maze-v0}              & Reach the goal position in a maze by taking multiple turns at intersections.     \\
                                                  & \texttt{h1hand-hurdle-v0}            & Move forward while overcoming hurdles.                                           \\
                                                  & \texttt{h1hand-crawl-v0}             & Move forward while passing through a tunnel.                                     \\
                                                  & \texttt{h1hand-sit\_simple-v0}       & Sit onto a chair situated closely behind the robot.                              \\
                                                  & \texttt{h1hand-sit\_hard-v0}         & Sit onto a chair with randomized robot and chair configurations.                 \\
                                                  & \texttt{h1hand-balance\_simple-v0}   & Balance on an unstable board with a fixed spherical pivot.                       \\
                                                  & \texttt{h1hand-balance\_hard-v0}     & Balance on an unstable board with a moving spherical pivot.                      \\
                                                  & \texttt{h1hand-stair-v0}             & Traverse a repeated sequence of upward and downward stairs.                      \\
                                                  & \texttt{h1hand-slide-v0}             & Walk over a repeated sequence of upward and downward slides.                     \\
                                                  & \texttt{h1hand-pole-v0}              & Travel forward through dense, thin poles without colliding with them.            \\

        \midrule

        \multirow{15}{=}{Whole-body manipulation} & \texttt{h1hand-cube-v0}              & Manipulate two in-hand cubes to match a randomly initialized target orientation. \\
                                                  & \texttt{h1hand-bookshelf\_simple-v0} & Relocate objects across shelves in a fixed order.                                \\
                                                  & \texttt{h1hand-bookshelf\_hard-v0}   & Relocate objects across shelves with randomized object orders and destinations.  \\
                                                  & \texttt{h1hand-window-v0}            & Use a window wiping tool while keeping its tip parallel to the window.           \\
                                                  & \texttt{h1hand-spoon-v0}             & Use a spoon to follow a circular pattern inside a pot.                           \\
                                                  & \texttt{h1hand-door-v0}              & Pull a door open using the doorknob and traverse through the doorway.            \\
                                                  & \texttt{h1hand-push-v0}              & Move a box to a randomly initialized 3D point on a table.                        \\
                                                  & \texttt{h1hand-basketball-v0}        & Catch a ball coming from a random direction and throw it into the basket.        \\
                                                  & \texttt{h1hand-truck-v0}             & Unload packages from a truck by moving them onto a platform.                     \\
                                                  & \texttt{h1hand-package-v0}           & Move a box to a randomly initialized target position.                            \\
                                                  & \texttt{h1hand-cabinet-v0}           & Open cabinet doors and manipulate objects inside the cabinet.                    \\
                                                  & \texttt{h1hand-room-v0}              & Organize scattered objects by reducing the variance of their planar locations.   \\
                                                  & \texttt{h1hand-insert\_normal-v0}    & Insert the ends of a rectangular block into two small pegs.                      \\
                                                  & \texttt{h1hand-insert\_small-v0}     & Insert the ends of a smaller rectangular block into two small pegs.              \\
                                                  & \texttt{h1hand-powerlift-v0}         & Lift a barbell of a designated mass.                                             \\

        \bottomrule
    \end{tabularx}
\end{table*}